\documentclass[11pt]{article}
\usepackage[margin=1in]{geometry}
\usepackage{amsmath,amssymb,amsthm,mathtools,bm,dsfont}
\usepackage{algorithm,algorithmicx,algpseudocode}
\usepackage{booktabs}
\usepackage{authblk}
\usepackage{multirow}
\usepackage{microtype}
\usepackage{nicefrac}
\usepackage{upgreek}
\usepackage{thm-restate}
\usepackage{etoc}
\usepackage{lmodern}
\usepackage{mathrsfs}
\usepackage{enumitem}
\usepackage{subcaption}
\usepackage{caption}
\usepackage{natbib}
\bibpunct[, ]{(}{)}{,}{a}{}{,}%
\usepackage[hidelinks]{hyperref}
\hypersetup{
  colorlinks=false
}

\usepackage{tikz, pgfplots}
\usepgfplotslibrary{fillbetween}
\pgfplotsset{compat=newest}
\usetikzlibrary{fit, positioning, shapes.misc, arrows.meta, calc, topaths, positioning}

\newcommand{\prob}[1]{\mathbb{P}( #1 )}
\newcommand{\Prob}[1]{\mathbb{P}\left( #1 \right)}
\newcommand{\mean}[1]{\mathbb{E} [ #1 ]}
\newcommand{\Mean}[1]{\mathbb{E}\left[#1 \right]}

\newcommand{\norm}[1]{ \|   #1  \| }
\newcommand{\ind}[1]{ \mathbb{I} \left\{#1 \right\} }
\newcommand{\abs}[1]{\left| #1 \right|}
\newcommand{\cPeps}[1]{\mathcal{P}^{(#1)}_{\epsilon_{#1}}}
\newcommand{\opnorm}[1]{\left\| #1 \right\|_{\text{op}}}
\newcommand*\diff{\mathop{}\!\mathrm{d}}

\DeclareMathOperator*{\argmax}{argmax}
\DeclareMathOperator*{\argmin}{argmin}
\usepackage[font=footnotesize, margin=1cm]{caption}

\usepackage{cleveref}
\crefname{section}{Section}{Sections}
\crefname{figure}{Figure}{Figures}
\crefname{theorem}{Theorem}{Theorems}
\crefname{lemma}{Lemma}{Lemmas}
\crefname{remark}{Remark}{Remarks}
\crefname{appendix}{Appendix}{Appendicies}
\crefname{proposition}{Proposition}{Propositions}
\crefname{equation}{Equation}{Equations}
\crefname{table}{Table}{Tables}
\crefname{assumption}{Assumption}{Assumptions}
\crefname{algorithm}{Algorithm}{Algorithms}
\crefname{example}{Example}{Examples}

\newtheorem{theorem}{Theorem}
\newtheorem{proposition}{Proposition}
\newtheorem{lemma}{Lemma}
\newtheorem{corollary}{Corollary}
\newtheorem{definition}{Definition}
\newtheorem{assumption}{Assumption}
\newtheorem{remark}{Remark}
\newtheorem{example}{Example}

\title{Efficient Transfer Learning via Causal Bounds}

\author[1]{Xueping Gong}
\author[2]{Wei You}
\author[3]{Jiheng Zhang}

\affil[1]{School of Management, Xiamen University, \url{xgongah@xmu.edu.cn}}
\affil[2]{Department of Industrial Engineering and Decision Analytics, The Hong Kong University of Science and Technology, \url{weiyou@ust.hk}}
\affil[2]{Department of Industrial Engineering and Decision Analytics, The Hong Kong University of Science and Technology, \url{jiheng@ust.hk}}

\begin{document}

\maketitle

\begin{abstract}
  Transfer learning seeks to accelerate sequential decision-making by leveraging offline data from related agents.
  However, data from heterogeneous sources that differ in observed features, distributions, or unobserved confounders often render causal effects non-identifiable and bias naive estimators.
  We address this by forming ambiguity sets of structural causal models defined via integral constraints on their joint densities.
  Optimizing any causal effect over these sets leads to generally non-convex programs whose solutions tightly bound the range of possible effects under heterogeneity or confounding.
  To solve these programs efficiently, we develop a hit-and-run sampler that explores the entire ambiguity set and, when paired with a local optimization oracle, produces causal bound estimates that converge almost surely to the true limits.
  We further accommodate estimation error by relaxing the ambiguity set and exploit the Lipschitz continuity of causal effects to establish precise error propagation guarantees.
  These causal bounds are then embedded into bandit algorithms via arm elimination and truncated UCB indices, yielding optimal gap-dependent and minimax regret bounds.
  To handle estimation error, we also develop a safe algorithm for incorporating noisy causal bounds.
  In the contextual-bandit setting with function approximation, our method uses causal bounds to prune both the function class and the per-context action set, achieving matching upper and lower regret bounds with only logarithmic dependence on function-class complexity.
  Our analysis precisely characterizes when and how causal side-information accelerates online learning, and experiments on synthetic benchmarks confirm substantial regret reductions in data-scarce or confounded regimes.
\end{abstract}

\section{Introduction}

Traditional sequential decision-making algorithms typically operate without prior knowledge \citep{banditBook}.
With the abundance of data collected from various sources, transfer learning has emerged as a powerful tool to accelerate learning by leveraging knowledge from a related source agent \citep{TLsurvey}.
Most transfer learning methods assume that both the source and target agents have access to the same complete information \citep{TCBforCrossDomain,TLinCB}.
Often, the data available from the source and target agents differ in features, completeness, or distributions.
This heterogeneity poses a challenge when transferring knowledge is not directly compatible.
We address this challenge by focusing on transfer learning with partially observable or biased data, using contextual bandit as an illustrating example.
In this setting, an agent makes a series of decisions (such as choosing an action based on current context) to maximize rewards,
but both the target agents have access to possibly different and incomplete sets of information.

Transfer learning has practical applications for integrating knowledge from heterogeneous data sources.
For example, in training personalized recommendation systems for e-commerce platforms, the source agent learns from detailed website user behavior (full context), while the target agent focuses on mobile app users who provide only partial data due to privacy restrictions or technical limitations.
New influencing factors may also emerge that were absent in earlier data, making datasets from different periods or platforms heterogeneous.
A similar challenge arises in autonomous driving systems.
The source agent gathers data from human drivers, including sensor readings like steering angle and environmental conditions.
However, human decision-making involves cognitive processes that cannot be captured by sensors.
The target agent aims to develop autonomous vehicles capable of making safe and efficient decisions but faces heterogeneity because the context of human decision-making is only partially observable.
These examples highlight the need for effective knowledge transfer across heterogeneous datasets.

In the presence of heterogeneous data, transfer learning becomes more challenging.
Agents have access only to partial contextual information, which complicate the task of accurately estimating the rewards of actions using source data.
Previous works \citep{boundingCE_continuous_IV,gxp,HTbyWeightedModels,PLforCB,islam2022computational,nikolaev2013balance} have attempted to address this issue by introducing additional assumptions, such as the availability of proxy variables or predefined causal relationships between observed and unobserved information.
While these approaches can be effective, the assumptions they rely on are often impractical in real-world scenarios, which limits their applicability.
In the absence of such assumptions, transferring knowledge from a source domain can sometimes degrade the performance of the target model, leading to what is known as \emph{negative transfer} \citep{rosenstein2005transfer}.

To address these challenges, it is crucial to focus on causal effects,
which capture the intrinsic properties of the environment and remain invariant across domains
\citep{bandits_causal_approaches,causalbandit,TLinCB,TLinRL,eberhardt2024discovering}.
However, in the presence of heterogeneous data, causal effects are often \textit{non-identifiable} due to unobserved confounding variables or distribution shifts,
which creates a significant barrier to the effective application of transfer learning.
Moreover, most existing works on offline transfer learning rely on the availability of instrumental variables (IVs) or proxies to infer causal relationships.
As highlighted by \citet{IAB,DeepProxy}, obtaining such variables is often challenging or infeasible in many applications.

In this paper, we adopt the structural causal model (SCM) framework \citep{causalbook}.
Our approach distills information from the source agent into causal bounds on the confounded causal effects of actions on rewards.
By leveraging these bounds, we reduce inefficient exploration in bandit algorithms, thus accelerating learning and ensuring faster convergence to optimal policies.

\subsection{Our Contributions}

\paragraph{Novel sampling framework for causal bounds.}
We focus on the joint density functions of SCM variables, defined with respect to a \emph{reference measure} to treat discrete, continuous, and mixed distributions uniformly.
To capture causal non-identifiability, we form ambiguity sets of these densities via integral constraints, thereby defining \textit{compatible SCMs} that encode diverse offline knowledge.
This allows us to define compatible SCMs that are consistent with a great variety of offline knowledge.
Optimizing any causal effect over these sets leads to generally non-convex programs whose solutions tightly bound the range of possible effects under heterogeneity or confounding.
To solve these programs efficiently, we develop a hit-and-run sampler on density space that explores the entire ambiguity set and produces causal-bound estimates that converge to the true limits in probability.
When coupled with a local optimization oracle, our method yields refined causal-bound estimates that achieve almost sure convergence.
By relaxing the ambiguity set to account for estimation error and exploiting the Lipschitz continuity of causal effects, we derive precise error propagation guarantees, laying the theoretical groundwork for transferring noisy causal bounds.
Furthermore, the Hausdorff-metric convergence extends even in infinite-dimensional spaces; we provide constructive approximation schemes in Appendix~\ref{sec: approximation of infinite function space}.

\paragraph{Transfer learning for multi-armed bandit with estimation uncertainty.}
First, we show how to embed these into multi-armed bandit algorithms to accelerate online learning,
when our causal-bound sampler recovers tight causal bounds that contain the true expected rewards.
Recognizing that bounds estimated from limited or biased data may fail to contain the true rewards in practice, we then introduce a provably safe mechanism for using noisy causal bounds.
Our analysis delivers both gap-dependent and minimax regret bounds that explicitly quantify the impact of causal estimation error,
and shows that as estimation noise vanishes, the algorithm's performance converges to that under exact bounds.

\paragraph{Contextual bandit transfer with function approximation.}
We extend our framework to contextual bandits with continuous contexts via general function approximation, setting it apart from the linear reward model in the literature \citep{PObandit}.
Specifically, we provide
(i) theoretical guarantee with logarithmic dependence on function-class complexity---a strict improvement over the square-root dependence in \citet{boundingCE_continuous_IV};
and (ii) computationally efficient methods that bound the metric entropy of causally constrained function classes through diameter reduction while characterizing pruned action sets via tractable convex programming for both unstructured and linear reward spaces.
Extensive experiments validate theoretical improvements and computational efficiency.

\subsection{Related Works}

\paragraph{Robust causal inference via optimization.}
Our work connects to a broad line of research on robust causal inference via optimization.
The do-calculus of \citet{causaloverview} uses probabilistic rules to identify causal effects, and \citet{tianBound} derived model-free bounds when point identification fails.
In discrete settings, linear program (LP) formulations were introduced by \citet{TLinMAB} and later made scalable by \citet{scalableLPbound}; however, these LP-based methods can yield loose bounds and often assume specific confounder structures.
Extensions to non-binary outcomes and continuous variables using instrumental variables have appeared in \citet{CE_nonbinary,unitSelectionCE,boundingCE_continuous_IV}.
A robust-optimization perspective was developed by \citet{bound_noisy_covariates} for partially identifying average treatment effects under noisy covariates, generalizing back-door adjustment and inverse-probability weighting to bound estimation.
To handle partially observable back-door and front-door criteria, \citet{CEbound} employed nonlinear programs.
\citet{summaryCB} provided an automated approach for causal inference in discrete settings, summarizing these optimization-based methods.
MCMC-based approximations under latent confounding were explored in \citet{MCMCbound}, albeit still relying on structural assumptions.

\paragraph{Marginal sensitivity model (MSM).}
\citet{MSM} provides an MSM framework for handling unmeasured confounding by introducing sensitivity parameters that bound the influence of hidden variables.
This allows practitioners to quantify uncertainty in causal estimates and derive valid bounds, for instance, on the average treatment effect \citep{doublysharpMSM} or the conditional average treatment effect \citep{intervalMSM,Quasi-oraclebounds}.
Building on MSM, \citet{minimaxMSM} develop a minimax-optimal policy-learning algorithm that minimizes worst-case regret over uncertainty sets for propensity scores.
Moreover, MSM supports principled sensitivity analysis:
For example, \citet{conformalMSM} propose a conformal sensitivity approach for individual treatment effects, delivering distribution-free coverage guarantees under varying levels of confounding.

\paragraph{Policy learning in offline and online settings.}
Our work relates to policy learning in both offline and online environments, including multi-armed bandits \citep{TLinMAB,TLinMABwithFiniteModels}, contextual bandits \citep{TLinCB,PL_confounded_bandit}, and reinforcement learning \citep{TLinRL,onlineRL_offline,hybridRL}.
In linear bandits, \citet{PObandit} improve online performance using partially missing offline data; we generalize their framework to arbitrary reward functions via our adaptive inverse-gap weighting.
Although recent studies \citep{POCB_linear_TS} explore partially observable contextual bandits, transfer learning in that setting remains understudied.
Transfer learning also plays a key role in operations management applications such as offline pricing policy optimization \citep{OfflineFeaturePricing,ren2024personalized,bu2020online}.

\paragraph{Partially observable Markov decision processes (POMDPs).}
POMDPs and related dynamical systems share many of the confounding challenges we address.
In the offline setting, \citet{POMDP_IV,OEPOMDP} and \citet{offlinePOMDPproxy} leverage proxy variables to identify causal effects in policy evaluation, while \citet{proximalRL} introduce bridge functions for a similar purpose under unobserved confounding.
In online reinforcement learning, \citet{POMDPlinear} and \citet{ProvablyEfficientCRL} apply the back-door criterion within the Bellman equation, and obtain provably efficient algorithms with linear function approximation even under partial observability.
More generally, \citet{CI_POMDP}, \citet{robustQ_RL}, and \citet{future_pomdp} develop techniques for causal inference in POMDPs and dynamical systems.
Outside the RL literature, \citet{TaleofTwoCities} model offline pricing under inventory control as an MDP with missing price points in the dataset.
While other works \citep{yu2020mopo,buckman2020importance,pessimismOfflineRL} have relaxed full-coverage assumptions for offline data, \citet{TaleofTwoCities} explicitly introduce partial observability into the offline MDP framework.

\subsection{Organization}

The rest of the paper is organized as follows.
In \cref{sec: prelim}, we review structural causal models, introduce compatible SCMs under prior constraints, and formalize causal non-identifiability along with the extremal-value bounds.
In \cref{sec: sampling general}, we present our causal-bound samplers for computing these bounds.
In \cref{sec: transfer}, we first show how exact causal-effect bounds can be used to prune arms and truncate UCB for accelerated regret, and then extend this to handle noisy or biased bounds while preserving performance guarantees.
In \cref{sec: transfer learning to POCB}, we extend our framework to contextual bandits with continuous contexts via function approximation, showing how causal-bound-driven prunings of the function class and action set yield strictly improved regret guarantees.
In \cref{sec: numerical}, we present our numerical results.
Details of proofs are deferred to the appendix.

\section{Preliminaries}\label{sec: prelim}
To transfer knowledge safely across domains whose data-generating processes may differ, we first formalize any available prior information from offline data summaries or expert judgment as constraints on the joint law of all variables. Within the structural causal model (SCM) framework, we then define the class of compatible SCMs to be those whose causal graph matches our assumptions and whose joint distributions satisfy exactly these imposed constraints.
Finally, by optimizing over this restricted family of models, we derive the tight bounds on any target causal effect.

\subsection{A Motivating Example: Partially Observed Structural Causal Models}
We adopt structural causal models (SCMs)~\citep{causalbook} as our foundational semantic framework. An SCM $\mathcal{M}=(\bm{V},\bm{U},\mathcal{F},\mathbb{P}_{\bm{U}})$ comprises a set of endogenous variables $\bm{V}$ and exogenous variables $\bm{U}$, a collection of structural equations $\mathcal{F}$, and a probability distribution $\mathbb{P}_{\bm{U}}$ over the exogenous variables. Each endogenous variable $V_i\in\bm{V}$ is determined by a structural equation $V_i = f_{V_i}\bigl(\mathrm{pa}(V_i)\bigr),$ where $f_{V_i}\in\mathcal{F}$ and $\mathrm{pa}(V_i)\subseteq\bm{V}\cup\bm{U}$ denotes its direct causes. The set of equations $\mathcal{F}$ induces a causal diagram $\mathcal{G}_{\mathcal{M}}$, a directed acyclic graph whose nodes correspond to variables and whose edges represent direct causal influences.

We consider partially observed offline datasets in which some covariates affecting either the action or the reward are unobserved by the decision maker. As an illustrative example, we focus on the SCM depicted in Figure~\ref{fig: causal model with unobserved U and observed W}, an abstraction of the partially observed contextual bandit (POCB). Here, $W$ denotes the observed context, $U$ the unobserved context, $A$ the action, and $Y$ the reward. We allow arbitrary causal relations between $W$ and $U$. A learning policy generates $A$ based on both $W$ and $U$, and the reward $Y$ depends on $W$, $U$, and $A$. POCB provides a rich framework for modeling decision-making scenarios where the agent has access to an observed context $W$ but lacks complete information about the unobserved context $U$.

\begin{figure}[hbtp]
  \center
  \resizebox{.34\textwidth}{!}{
    \begin{tikzpicture}[->,shorten >=1pt,auto,node distance=2.5cm,-stealth,
      thick, base node/.style={circle,draw,minimum size=48pt}, real node/.style={double,circle,draw,minimum size=50pt}]

      \node[shape=circle,draw=black] (0) at (0,0) {$A$};
      \node[shape=circle,draw=black] (1) at (5,0) {$Y$};
      \node[shape=circle,draw=black] (2) at (2.5,1.3) {$W$};
      \node[shape=rounded rectangle, inner sep=7.5pt, dashed, draw=black] (3) at (2.5,3.6) {$U$};

      \node[draw, dotted, rectangle, fit={(0) (1)}, inner sep=0.2cm, label={[label distance=-0.55cm]below:multi-armed bandit}] {};

      \node[draw, dash dot, rectangle, fit={(0) (1) (2)}, inner sep=0.75cm, label={[label distance=-0.55cm]below:paritally observable contextual bandit}] {};

      \path[]
      (0) edge[->] node {} (1)
      (2) edge[->] node {} (0)
      (2) edge[->] node {} (1)
      (3) edge[->] node {} (1)
      (3) edge[->] node {} (0)
      (2) edge[dashed,<->] node {} (3);
    \end{tikzpicture}
  }
  \caption{Causal Diagram $\mathcal{G}$ for Contextual Bandit}
  \label{fig: causal model with unobserved U and observed W}
\end{figure}

\begin{example}[Warm-start online learning]
  Imagine a healthcare system where offline data has been collected from historical patient records.
  This dataset includes partially observable information: patients' health outcomes $Y$, the treatment administered $A$ and observed contexts $W$ (e.g., age, weight, baseline health conditions).
  However, critical information such as genetic markers or specific health risk factors (denoted by $U$), is unobserved due to limitations in past testing or incomplete records.
  Additionally, some features may be intentionally masked to protect patient privacy.
\end{example}

\begin{example}[Autonomous driving]
  Consider the task of training an autonomous driving system.
  During data collection, an autonomous vehicle is operated by a human driver, capturing the driving action $A$, the reward $Y$ (e.g., driving safety), and context variables $W$ (e.g., radar signals).
  However, the human driver's judgment $U$ of the environment cannot be collected by the sensors.
\end{example}

\subsubsection{Notation.}
Throughout the paper, we use the following notation. Uppercase letters (e.g., $X$) denote random variables and lowercase letters (e.g., $x$) their realizations.
Calligraphic letters (e.g., $\mathcal{X}$) denote the domain of a variable, with $|\mathcal{X}|$ its cardinality.
For a vector of random variables $\bm{X} = (X_1, \dots, X_d)$, we write its joint cumulative distribution function as $F(\bm{x}) \triangleq \mathbb{P}(\bm{X} \le \bm{x}),$
and its joint density (or probability mass) function as $\rho(\bm{x}) \triangleq \mathbb{P}(\bm{X} = \bm{x}).$
Conditional densities are denoted by $\rho(\bm{x} \mid \bm{y}) \triangleq \mathbb{P}(\bm{X} = \bm{x} \mid \bm{Y} = \bm{y})$.
When we wish to emphasize dependence on a particular structural causal model $\mathcal{M}$, we add a subscript $\mathcal{M}$; for example, $\rho_{\mathcal{M}}(\bm{v})$ denotes the joint density of $\bm{V}$ under model $\mathcal{M}$.
For any positive integer $k$, we denote the list $\{1,2,\ldots,k\}$ by $[k]$.

\subsection{Compatible Causal Models}
We assume all agents share a common structural causal graph $\mathcal{G}$.
They may observe different subsets of the endogenous variables $\bm{V}$ and exogenous variables $\bm{U}$, or operate in distinct environments that induce different observational distributions $\rho(\bm{v}) = \mathbb{P}(\bm{V}=\bm{v})$.
These differences mean the underlying SCMs need not be identical, so the usual \emph{unconfoundedness assumption} may fail.

To derive meaningful causal bounds, we focus on the subset of SCMs that are consistent with both the graph $\mathcal{G}$ and any available prior information such as offline dataset or domain expertise.
Formally, let $\mathcal{C}$ be a collection of constraints on the joint law $\rho(\bm{v},\bm{u})$.  The SCMs satisfying $\mathcal{C}$ form our class of \emph{compatible causal models}, and we compute bounds on causal effects by optimizing over this restricted class.

\begin{definition}[Compatible structural causal models]
  Let $\mathcal{G}$ be a causal diagram over variables $\bm{V} \cup \bm{U}$, and let $\mathcal{C}$ be a collection of constraints on the joint distribution $\rho(\bm{v},\bm{u})$.
  A structural causal model $\mathcal{M}$ is called \emph{compatible} with $(\mathcal{G},\mathcal{C})$ if $\mathcal{G}_{\mathcal{M}} = \mathcal{G}$ and $\rho_{\mathcal{M}}(\bm{v},\bm{u})$ satisfies all constraints in $\mathcal{C}$. We denote the set of all such models by $\mathfrak{C} = \mathfrak{C}(\mathcal{G},\mathcal{C})\,.$
\end{definition}

\begin{example}\label{ex:constraints}
  We list some common constraints that can be imposed on the joint law $\rho_{\mathcal{M}}(\bm{v},\bm{u})$:
  \begin{enumerate}[label=(\alph*)]
    \item \emph{Observational distribution:} In many applications, the agent has an offline dataset collected under the same environment, which provides the joint distribution $\rho(\bm{v})$ of the observed variables, commonly called the \emph{observational distribution}.
      We may incorporate this by adding the constraint $\{\rho_{\mathcal{M}}(\bm{v}) = \rho(\bm{v})\}$ to our set of  constraints $\mathcal{C}$. \label{ex:observational distribution}
    \item \emph{Known marginal distribution:} In many applications, the agent have knowledge of the existence of certain confounders, but have no access to their values.
      For example, in the healthcare setting, the agent may know that a patient's genetic predisposition $U$ is a confounder of the treatment $A$ and the outcome $Y$, but the agent does not have access to $U$ due to privacy concerns.
      However, it is usually possible to obtain the marginal distribution of $U$ from external sources, such as population studies or expert domain knowledge.
      Let $\rho(u)$ denote the marginal distribution $U$, then one imposes the constraint $\{\rho_{\mathcal{M}}(u) = \rho(u)\}$ on the joint law $\rho_{\mathcal{M}}(\bm{v},u)$. \label{ex:known marginal distribution}
    \item \emph{Estimation error and distribution shift:}
      When only an estimate $\widehat{\rho}(\bm{v})$ of the observational distribution is available, one may allow estimation uncertainty or distribution shift via
      $
      \bigl|\rho_{\mathcal{M}}(\bm{v}) - \widehat{\rho}(\bm{v})\bigr| \le \varepsilon(\bm{v}),
      $
      where $\varepsilon(\bm{v})\ge0$ is a tolerance or robustness parameter.
    \item \textit{Marginal sensitivity model:} \citet{rosenbaum_model} assumes that, for each action-context pair,
      \[
        \Gamma^{-1}
        \le \frac{ \rho(a_i = 1\mid w_i, u_i )/ ( 1- \rho(a_i = 1\mid w_i, u_i ))   }{ \rho(a_j = 1\mid w_j, u_j )/ ( 1- \rho(a_j = 1\mid w_j, u_j ))  }
        \le \Gamma,
        \quad \text{for constants }
        \Gamma >1 \text{ and } \forall w_i = w_j.
      \]
      Thie sensitivity model quantifies how causal conclusions might change due to unmeasured confounding by parameterizing the maximum possible bias introduced by an unobserved covariate $U$.
  \end{enumerate}
\end{example}

\subsection{Causal Non-Identifiability}

Let $V(\mathcal{M})$ denote a target causal quantity under the structural causal model $\mathcal{M}$.
A canonical example is the effect of an intervention on an action variable $A$, formalized by Pearl's do-operator.
Writing $\mathrm{do}(A=a)$ means we modify $\mathcal{M}$ by replacing the structural equation for $A$ with the constant assignment $A=a$, which graphically corresponds to deleting all incoming edges into $A$ while leaving the rest of the causal diagram unchanged.
The resulting interventional distribution $\mathbb{P}(Y=y\mid \mathrm{do}(A=a))$ describes the law of the outcome $Y$ after setting $A=a$, and its expectation $\mathbb{E}_{\mathcal{M}} [Y \mid \mathrm{do}(A=a)]$ is called the \emph{causal effect} of $A$ on $Y$.

When $V(\mathcal{M})$ represents a causal effect, its identifiability from observed data can be determined by do-calculus \citep{causaloverview}.
For example, suppose we know the full joint distribution $F(a,y,w,u)$, and our target is $V(\mathcal{M}) = \mean{Y|\mathrm{do}(A=a),w}$.
Since $W$ and $U$ together block all back-door paths from $A$ to $Y$, we may apply the back-door adjustment to uniquely identify the causal effect from $F$:
\[
  \mean{Y|\mathrm{do}(A=a),w} \triangleq \int_{u\in\mathcal{U}} \mean{Y|a,w,u}\diff F(u|w).
\]
We refer to \cref{sec in appendix: related material} for a concise review of do-calculus rules.

While do-calculus provides a systematic method for deriving identification formulas, it cannot recover a quantity that is \textit{non-identifiable}: whenever $U$ is unobserved, there may exist multiple SCMs compatible with the same observational distribution (even under additional constraints), yielding different values of $V(\mathcal{M})$.

\begin{definition}[Non-identifiability]
  Given a causal graph $\mathcal{G}$ and a constraint set $\mathcal{C}$, a quantity $V$ is called \emph{non-identifiable} if there exist two distinct models $\mathcal{M}, \mathcal{M}' \in \mathfrak{C}$ such that
  $V(\mathcal{M}) \neq V(\mathcal{M}')$.
\end{definition}

Causal non-identifiability presents an obstacle in transfer learning.
Consider the SCM depicted in \cref{fig: causal model with unobserved U and observed W}.
The unobserved expert context $U$ confounds both the action $A$ and the outcome $Y$.
Using only the observational distribution of $(A,Y,W)$ can induce negative transfer.
The root of the problem lies in the discrepancy between the conditional expectation $\mathbb{E}[Y|a,w]$ and the causal effect $\mathbb{E}[Y|\mathrm{do}(A = a),w]$.
Naive transfer methods that estimate $\mathbb{E}[Y|a,w]$ as a substitute for the causal effect incur bias from the unknown expert policy.
The following example demonstrates how such bias can lead to suboptimal policies and degraded performance in transfer learning.

\begin{example}[Negative transfer]
  Consider a 2-arm contextual bandit where $A$, $U$, and $W$ are all binary.  At the start of each round, the hidden and observed contexts $U,W$ are drawn independently with $\prob{U=1} = 0.9 $ and $ \prob{W=1} = 0.5.$
  These contexts jointly determine the reward $Y$ according to Table \ref{tb:Reward table}.
  \begin{table}[hbtp]
    \centering
    \begin{minipage}[b]{0.57\linewidth}
      \centering
      \small
      \begin{tabular}{c|cccc}
        \toprule
        $(w,u)$ & $(0,0)$ & $(1,0)$ & $(0,1)$ & $(1,1)$ \\
        \midrule
        $\mean{Y|\mathrm{do}(A=1),w,u}$   & $0$     & $0$     & $1$     & $1$     \\
        $\mean{Y|\mathrm{do}(A=0),w,u}$   & $10$    & $10$    & $0.9$   & $0.9$   \\
        \bottomrule
      \end{tabular}
      \caption{Reward Function}
      \label{tb:Reward table}
    \end{minipage}
    \hfill
    \begin{minipage}[b]{0.42\linewidth}
      \centering
      \small
      \begin{tabular}{c|cccc}
        \toprule
        $(a,w)$     & $(0,0)$ & $(0,1)$ & $(1,0)$ & $(1,1)$ \\
        \midrule
        $\mathbb{P}(a,w)$   & $0.05$     & $0.05$     & $0.45$     & $0.45$     \\
        \bottomrule
      \end{tabular}
      \caption{Observational Distribution}
      \label{tb:Expert probability table after transfer}
    \end{minipage}
  \end{table}
  An expert agent, observing both $U$ and $W$, follows the optimal policy $\pi^*_{\text{expert}}: \mathcal{W}\times\mathcal{U} \to \mathcal{A}$.
  and thus induces the empirical distribution $\prob{A,W}$ shown in Table \ref{tb:Expert probability table after transfer}.
  Since $U=1$ with probability $0.9$, the expert chooses arm 1 in $90\%$ of rounds.
  To see the causal effect of each arm, we compute for any fixed $w\in\{0,1\}$: $\mean{Y|\mathrm{do}(A=0),w} = 1.81$ and $\mean{Y|\mathrm{do}(A=1),w} = 0.9$.
  Hence the agent's optimal policy is $\pi_{\text{agent}}^*(W)=0$ which starkly contrasts with $\pi^*_{\mathrm{expert}}$ favoring arm 1.
  The discrepancy arises because the rare context $U=0$ yields a large reward gap, while for the common context $U=1$ the arms are almost equivalent.  An agent ignorant of $U$ but aware of its skewed distribution can infer that observed rewards are biased by this hidden context and must rely on the causal effect to transfer knowledge correctly.
\end{example}

\begin{remark}
  \citet{PObandit} tackle non-identifiability issue in the linear regime by encoding partially observable, confounded data as linear constraints in the online bandit setting.  In contrast, our framework makes no parametric assumptions on the reward mechanism---the structural equation for $Y$ may be completely arbitrary.
  Likewise, \citet{boundingCE_continuous_IV} extend to general function classes using instrumental variables; however, identifying valid instruments remains challenging, and their approach does not offer guarantees on regret optimality.
\end{remark}

\subsection{Causal Bounds}

We tackle the non-identifiability issue by producing tight bounds on $V(\mathcal{M})$, defined as
\begin{equation}\label{eq: causal bounds}
  V_{\min} = V_{\min}(\mathfrak{C}) \triangleq \inf_{\mathcal{M} \in \mathfrak{C}} V(\mathcal{M}),
  \quad
  V_{\max} = V_{\max}(\mathfrak{C}) \triangleq \sup_{\mathcal{M} \in \mathfrak{C}} V(\mathcal{M}).
\end{equation}
Causal bounds are essential because they quantify the full range of plausible values for our target quantity under all models consistent with the observed data and imposed constraints.  By characterizing both the worst-case and best-case scenarios for $V(\mathcal{M})$, we obtain a measure of uncertainty stemming from unobserved confounding or distribution shifts, and ensures that any decision or policy derived from these bounds is robust to hidden variation in the structural equations.
In subsequent sections, we show how to compute these bounds efficiently under various sensitivity models and integrate them into regret analysis for transfer learning.

\section{A Sampling Method for Causal Bound Computation}
\label{sec: sampling general}

Given a nonempty family of compatible causal models $\mathfrak{C}$, our goal is to compute the causal bounds defined in \eqref{eq: causal bounds} for some causal quantity of interest $V(\mathcal{M})$.
When $V$ depends nonlinearly on the causal model, these infimum and supremum problems can be non-convex, so standard convex solvers do not apply.

In what follows, we develop an ergodic hit-and-run MCMC sampler to approximate $V_{\min}$ and $V_{\max}$ under arbitrary continuous objectives.
We first show that the compatibility constraints defining $\mathfrak{C}$ forms a convex polytope $\mathcal{P}$ in the space of joint densities.
We then present a hit-and-run sampler that generates points in $\mathcal{P}$ whose empirical minima and maxima of $V$ converge in probability to $V_{\min}$ and $V_{\max}$.
To accelerate convergence, we describe how to feed each sampled point into a deterministic local-optimization oracle, strengthening convergence to almost sure.
Finally, we explain how to relax the polytope $\mathcal{P}$ to account for estimation error or mild distribution shift, yielding distributionally robust causal bounds.

\subsection{Sampling Density Functions}
Each SCM $\mathcal{M}$ defines a probability measure $\mathbb{P}_{\mathcal{M}}$ on $\Omega = \mathcal{A} \times \mathcal{Y} \times \mathcal{W} \times \mathcal{U}$,
which we assume is absolutely continuous with respect to a reference measure $\nu$.
Hence its Radon-Nikodym derivative $\rho_{\mathcal{M}} = \mathrm{d}\mathbb{P}_{\mathcal{M}}/\mathrm{d}\nu$ exists and satisfies $\rho_{\mathcal{M}} \in L^2(\Omega,\nu)$.
We assume that the feasible density set $\mathcal{P}$ is carved out by linear constraints:
\begin{equation}\label{eq: feasible set}
  \begin{aligned}
    \textstyle
    \mathcal{P}
    = \Bigl\{\rho \in L^2(\Omega,\nu)  \Bigm| \rho\geq 0,
      \, \,
      \int_{\Omega} \alpha_i \rho \mathrm{d}\nu = \beta_i,
      \, \,
      \int_{\Omega} \alpha_j \rho \mathrm{d}\nu \le \beta_j,
      \, \,
      \forall i \in [m], j \in [m']\backslash [m]
    \Bigr\},
  \end{aligned}
\end{equation}
where each $\alpha_i\in L^2(\Omega,\nu)$ is a given constraint function and $\beta_i\in\mathbb{R}$ is its associated constant.
Let $\bm x=(a,y,w,u)\in\Omega$ denote a generic point.
Without loss of generality, we assume that $\alpha_1(\bm x) = 1$ and $\beta_1=1$, corresponding to the normalization constraint.
The choices of $\alpha$ in the linear constraints can encode a variety of causal and structural conditions; see \cref{tab:causal_constraints_examples} for examples.

\begin{table}[h]
  \renewcommand{\arraystretch}{1.2}
  \centering
  \small
  \begin{tabular}{llll}
    \toprule
    Constraint & $\alpha(\bm x)$ & $\beta$ & Interpretation \\
    \midrule
    Normalization
    & $1$
    & $1$
    & $\int_{\Omega} \rho(\bm x) \mathrm{d}\nu(\bm x) = 1$ (total probability) \\
    Propensity
    & $\mathbb{I}\{A=a_0\}$
    & $\mathbb{P}(A=a_0)$
    & $\int_{\Omega} \mathbb{I}\{A=a_0\} \rho(\bm x) \mathrm{d}\nu(\bm x) = \mathbb{P}(A=a_0)$ \\
    Context marginal
    & $\mathbb{I}\{W\in\mathcal W_0\}$
    & $\mathbb{P}(W\in\mathcal W_0)$
    & $\int_{\Omega} \mathbb{I}\{W\in\mathcal W_0\} \rho(\bm x) \mathrm{d}\nu(\bm x) = \mathbb{P}(W\in\mathcal W_0)$ \\
    Conditional mean
    & $\mathbb{I}\{A=a_k\}\,Y$
    & $\mathbb{E}[Y\mid A=a_k]$
    & $\int_{\Omega} y \mathbb{I}\{A=a_k\} \rho(\bm x) \mathrm{d}\nu(\bm x) = \mathbb{E}[Y\mid A=a_k]$ \\
    Sensitivity bands
    & $\alpha_i(\bm x)$
    & $[\beta_i^-,\beta_i^+]$
    & $\beta_i^- \le \int_{\Omega} \alpha_i(\bm x) \rho(\bm x) \mathrm{d}\nu(\bm x)  \le \beta_i^+$ (uncertainty set) \\
    \bottomrule
  \end{tabular}
  \caption{Examples of Constraints on $\rho$ via Test Functions $\alpha(\bm x)$ and Corresponding Targets $\beta$}
  \label{tab:causal_constraints_examples}
\end{table}

\begin{remark}[Reference measures]
  The reference measure $\nu$ plays two essential roles. First, by assuming $\mathbb{P}_{\mathcal{M}}\ll\nu$, it guarantees the existence of the Radon-Nikodym derivative $\rho_{\mathcal{M}}$. Second, it adapts to the nature of the variables: for discrete components, $\nu$ is the counting measure; for continuous components, the Lebesgue measure; and in mixed settings, their product $\nu = \nu_{\mathcal{A}}\otimes\nu_{\mathcal{Y}}\otimes\nu_{\mathcal{W}}\otimes\nu_{\mathcal{U}}$.
  This unified construction covers both discrete and continuous cases seamlessly.
\end{remark}

We now present a hit-and-run sampler that constructs an ergodic Markov chain whose unique stationary distribution is uniform over $\mathcal{P}$.
We extend the sampling framework of \citet{convex_sampling}, designed for convex bodies in Euclidean space, to function spaces.
Our key innovation is a closed-form computation of step sizes that automatically respect all linear constraints, thus eliminating costly line searches.
Under mild regularity and finite-dimensionality assumptions, this algorithm runs efficiently in the constrained $L^2(\Omega,\nu)$ subspace while covering a wide range of practical causal scenarios.

\begin{assumption}\label{asp: regularity}
  The constraint functions $\{\alpha_1, \alpha_2, \dots, \alpha_m\}$ are linearly independent.
  Furthermore, the feasible set $\mathcal{P}$ is nonempty and bounded in the $L^2(\Omega,\nu)$ norm.
\end{assumption}

\begin{assumption}\label{asp: finite-dimension}
  There exists an $n$-dimensional subspace of $L^2(\Omega,\nu)$ with $n<\infty$ that contains $\mathcal{P}$.
\end{assumption}

The linear constraint structure in \eqref{eq: feasible set} permits an efficient projection scheme.
Define the equality-constraint operator $\mathscr{A}:L^2(\Omega,\nu)\to\mathbb{R}^m,
\mathscr{A}(g)=\bigl[\int_{\Omega} \alpha_i(\bm{x}) g(\bm{x}) \mathrm{d}\nu(\bm{x}) \bigr]_{i=1}^m.$
By Assumption~\ref{asp: regularity}, $\mathscr{A}$ has full row rank, and its adjoint $\mathscr{A}^\top: \mathbb{R}^m \to L^2(\Omega,\nu)$ is characterized via Riesz representation:
\[
  \mathscr{A}(g)^\top \bm{\xi}
  =\int_\Omega g(\bm{x})\,\mathscr{A}^\top(\bm{\xi})(\bm{x})\,\mathrm{d}\nu(\bm{x}),
  \quad
  \mathscr{A}^\top(\bm{\xi})=\sum_{i=1}^m\xi_i\,\alpha_i(\bm{x}).
\]
This yields the orthogonal projector
$
\mathscr{P} = \mathscr{I} - \mathscr{A}^\top (\mathscr{A}\mathscr{A}^\top)^{-1} \mathscr{A}
$
onto $\ker(\mathscr{A})$, where $\mathscr{I}$ is the identity operator on $L^2(\Omega,\nu)$.
By construction, any perturbation in $\ker(\mathscr{A})$ preserves all equality constraints in \eqref{eq: feasible set}.

Our Algorithm~\ref{alg: hit-run for density} operates as follows: (i) Draw a Gaussian process $G_t \sim \mathsf{GP}(0, \mathbb{K})$ with a pre-specified positive definite kernel $\mathbb{K}$, then project $G_t$ onto $\ker(\mathscr{A})$ to obtain a feasible direction $d_t$; (ii) compute \textit{closed-form} bounds $\lambda_t^{\min},\lambda_t^{\max}$ such that $\rho_{t-1}+\lambda d_t\in\mathcal{P}$ for all $\lambda\in[\lambda_t^{\min},\lambda_t^{\max}]$; and (iii) sample $\lambda_t$ uniformly from this interval and setting $\rho_t=\rho_{t-1}+\lambda_td_t$.
The following proposition guarantees the feasibility of the samples.

\begin{proposition}
  \label{prop: valid samples in P}
  The samples $\{\rho_t\}_{t=1}^T$ generated by Algorithm~\ref{alg: hit-run for density} satisfy the constraints in $\mathcal{P}$.
\end{proposition}

\begin{algorithm}[hbtp]
  \renewcommand{\algorithmicrequire}{\textbf{Input:}}
  \renewcommand{\algorithmicensure}{\textbf{Output:}}
  \caption{Sampling Method for Compatible Causal Models}
  \label{alg: hit-run for density}
  \begin{algorithmic}[1]
    \Require Sample space $\Omega$, reference measure $\nu$, feasible set $\mathcal{P}$ defined in \eqref{eq: feasible set}, number of iterations $T$, positive definite covariance kernel $\mathbb{K}$, and initial density $\rho_0 \in \mathcal{P}$ w.r.t. $\nu$

    \State Initialize $v_{0,j} \gets \int_{\Omega} \alpha_j(\bm x) \rho_0(\bm x) \mathrm{d}\nu(\bm x) $ for $j \in [m'] \backslash [m]$
    \For{$t=1$ \textbf{to} $T$}
    \State Generate $G_t \sim \mathsf{GP}(0, \mathbb{K})$ and get projection $d_t \gets \mathscr{P}(G_t)$
    \State Compute $c_{t,j} \gets \int_{\Omega} \alpha_j(\bm x) d_t(\bm x) \mathrm{d}\nu(\bm x) $ for $j \in [m'] \backslash [m]$

    \State Compute the range of stepsize
    $$
    \lambda_t^{\min} \gets \max\biggl\{\sup_{\substack{j\in[m']\backslash [m],\,  c_{t,j} < 0}} \bigl(\beta_j - v_{t-1,j}\bigr)/c_{t,j},
    \sup_{\substack{\bm{x} \in \Omega,\,  d_t(\bm{x}) > 0}} \bigl( -\rho_{t-1}(\bm{x})/ d_t(\bm{x}) \bigr)\biggr\}
    $$
    $$
    \lambda_t^{\max} \gets \min\biggl\{ \inf_{\substack{j \in [m']\backslash [m],\,  c_{t,j} > 0}} \bigl(\beta_j - v_{t-1,j}\bigr)/c_{t,j},
    \inf_{\substack{\bm{x} \in \Omega,\,  d_t(\bm{x}) < 0}} \bigl( -\rho_{t-1}(\bm{x})/d_t(\bm{x}) \bigr)\biggr\}
    $$

    \State $\lambda_t \sim \operatorname{Uniform}[\lambda_t^{\min}, \lambda_t^{\max}]$
    \State $\rho_t \gets \rho_{t-1} + \lambda_t d_t$
    \State $v_{t,j} \gets v_{t-1,j} + \lambda_t c_{t,j}$ for $j \in [m'] \backslash [m]$
    \EndFor
    \Ensure $\{\rho_t\}_{t=0}^T$
  \end{algorithmic}
\end{algorithm}

We now show that Algorithm~\ref{alg: hit-run for density} induces an ergodic Markov chain on the convex polytope $\mathcal{P}$ whose unique stationary distribution is the uniform measure over $\mathcal{P}$.

\begin{proposition}
  \label{prop: hit-and-run}
  Under Assumptions~\ref{asp: regularity} and \ref{asp: finite-dimension}, the sequence $\{\rho_t\}_{t=0}^T$ generated by Algorithm~\ref{alg: hit-run for density} defines a reversible, ergodic Markov chain with its unique stationary distribution given by the uniform measure on $\mathcal{P}$.
\end{proposition}

\begin{remark}[Speeding up the sampling procedure]
  We discuss several strategies to accelerate the sampling procedure in Algorithm~\ref{alg: hit-run for density}.
  First, rather than project arbitrary functions, we precompute an orthonormal basis $\{\psi_k\}_{k=1}^{n-m}$ of $\ker(\mathscr{A})$.
  Then each random direction can be drawn simply as $d_t = \sum_{k=1}^{n-m} \zeta_k\,\psi_k,$ where $\zeta_k \stackrel{\text{i.i.d.}}{\sim} \mathcal{N}(0,1).$
  This leverages the finite-dimensionality guaranteed by Assumption~\ref{asp: finite-dimension} and eliminates the need to compute $\mathscr{A}^\top$ at runtime.
  Second, several steps of the sampler admit straightforward parallel execution:
  the Gaussian coefficients $\{\zeta_k\}$ for each $t$ are independent across iterations and can be drawn in parallel;
  the formulas for $\lambda_t^{\min}$ and $\lambda_t^{\max}$ depend only on the current $\rho_{t-1}$ and $d_t$, so each can be evaluated concurrently;
  and running several hit-and-run chains in parallel accelerates coverage of $\mathcal{P}$ and improves overall sample efficiency.
  Finally, we can further exploit structure in particular SCMs.
  For example, in the \textit{discrete} POCB model of Figure~\ref{fig: causal model with unobserved U and observed W}, one can reduce the effective projection dimension when estimating $V(\mathcal{M}) = \mathbb{P}(Y=y\mid\mathrm{do}(A=a),W=w)$; see Appendix~\ref{app: problem-specific acceleration}.
\end{remark}

\subsection{Causal Bounds and Its Convergence}
Algorithm~\ref{alg: hit-run for density} generates a sequence of densities $\{\rho_t\}_{t=1}^T$ whose law converges in total variation to the uniform distribution on $\mathcal{P}$.
Consequently, we can estimate the causal bounds by simple Monte Carlo: for each sample $\rho_t$ (corresponding to a model $\mathcal{M}_t$) we compute the quantity $V(\rho_t) = V(\mathcal{M}_t)$, and set
\[
  \widehat V_{\min}(T) = \min_{t \in [T]} V(\rho_t),
  \qquad
  \widehat V_{\max}(T) = \max_{t \in [T]} V(\rho_t).
\]
Denote the true bounds by $V_{\min} = \inf_{\rho\in\mathcal{P}} V(\rho)$ and $V_{\max} = \sup_{\rho\in\mathcal{P}} V(\rho)$.
To guarantee $\widehat V_{\min}(T)\to V_{\min}$ and $\widehat V_{\max}(T)\to V_{\max}$ as $T\to\infty$, we require continuity of $V$ in the $L^2(\Omega,\nu)$ topology.
\begin{assumption}\label{assu: continuous}
  There exists a compact set $\mathcal{K}\subset L^2(\Omega,\nu)$ such that  $\mathcal{P} \subset \mathcal{K}$, and $V : \mathcal{K} \rightarrow \mathbb{R}$ is continuous with respect to the $L^2(\Omega,\nu)$ norm.
\end{assumption}
Indeed, many causal effects of interest satisfies \cref{assu: continuous}.
For example, in Propositions~\ref{prop: continuity of E[Y|do(a)]} and~\ref{prop: continuity of E[Y|do(a),w]}, we show that the causal effects $\mathbb{E}_{\rho}[Y\mid \mathrm{do}(A=a)]$ and $\mathbb{E}_{\rho}[Y\mid \mathrm{do}(A=a),W=w]$ are in fact Lipschitz continuous in $\|\cdot\|_{L^2(\Omega,\nu)}$ under standard assumptions on the SCM~$\mathcal{M}$.

Combining the compactness of $\mathcal{P}$ (Assumptions~\ref{asp: regularity} and~\ref{asp: finite-dimension}) with the continuity of $V$ (Assumption~\ref{assu: continuous}) yields consistency of our Monte Carlo estimates:
\begin{theorem}\label{thm: convergence in probability in Markov}
  Under Assumptions~\ref{asp: regularity}, \ref{asp: finite-dimension}, and \ref{assu: continuous}, let $\{\rho_t\}_{t=1}^T$ be the output of Algorithm~\ref{alg: hit-run for density}.  Then
  \[
    \widehat V_{\min}(T) \xrightarrow{p} V_{\min},
    \quad
    \widehat V_{\max}(T) \xrightarrow{p} V_{\max}.
  \]
\end{theorem}

\subsection{Accelerated Convergence via a Deterministic Optimization Oracle}\label{sec: optimization oracle}
In high-dimensional settings, uniform sampling from the feasible polytope $\mathcal{P}$ may suffer from the curse of dimensionality, causing the empirical bounds $\widehat V_{\min}(T)$ and $\widehat V_{\max}(T)$ to converge slowly to the true extremes.
To accelerate convergence, we assume access to deterministic local-optimization oracles satisfying:
\begin{assumption}[Local-optimization oracle]\label{asp: OPT}
  There exists a radius $\delta>0$ such that, for each local optimum $\rho_{\mathrm{loc}}\in\mathcal{P}$, the oracle $\texttt{OPT}$ returns $\rho_{\mathrm{loc}}$ whenever its input $\rho_0$ lies in the ball $\mathcal{B}(\rho_{\mathrm{loc}},\delta)\cap\mathcal{P}$.
\end{assumption}

\begin{remark}
  Although densities $\rho$ formally live in an infinite-dimensional function space, practical implementations employ finite-dimensional parameterizations (e.g., discretization, basis expansions, or neural network weights).
  Under such a parametrization, $\mathcal{P}$ reduces to a subset of finite-dimensional Euclidean space, and $\texttt{OPT}$ corresponds to standard local optimizers (e.g., L-BFGS, gradient descent, or Newton methods).
  These routines exhibit well-known local convergence guarantees: when initialized within an attraction basin $\mathcal{B}(\rho_{\mathrm{loc}},\delta)$, they converge to $\rho_{\mathrm{loc}}$ at linear or superlinear rates.
  \cref{asp: OPT} thus captures the empirical behavior of practical solvers and breaks the curse of dimensionality by leveraging local geometry, accelerating bound estimation by orders of magnitude compared to naive sampling.
\end{remark}

Let $\texttt{OPT}_{\min}$ and $\texttt{OPT}_{\max}$ be local minimization and maximization oracles satisfying \cref{asp: OPT}.
While neither oracle alone can escape local optima of the generally non-convex objective $V(\rho)$, their combination with hit-and-run sampling yields dramatic acceleration:
the sampler ergodically explores $\mathcal P$ to generate diverse starting points, and each oracle then rapidly refines its input to the nearest local extremum.
This hybrid strategy is detailed in Algorithm~\ref{alg:sampling with optimization oracle}.

\begin{algorithm}[hbtp]
  \renewcommand{\algorithmicrequire}{\textbf{Input:}}
  \renewcommand{\algorithmicensure}{\textbf{Output:}}
  \caption{Accelerated Sampling of Causal Bounds with Local Optimization}
  \label{alg:sampling with optimization oracle}
  \begin{algorithmic}[1]
    \Require Initial density $\rho_0 \in \mathcal{P}$, number of iterations $T$, local oracles $\texttt{OPT}_{\min}, \texttt{OPT}_{\max}$
    \State Initialize the hit-and-run sampler (\cref{alg: hit-run for density}) at $\rho_0$
    \For{$t = 1,\dots,T$}
    \State Sample $\rho_t$ via one step of \cref{alg: hit-run for density}
    \State $\rho_{\min,t} \gets \texttt{OPT}_{\min}(\rho_t)$ and $\rho_{\max,t} \gets \texttt{OPT}_{\max}(\rho_t)$
    \EndFor
    \Ensure 
    $\widehat V_{\min}^{\texttt{OPT}}(T) = \min_{1 \le t \le T} V(\rho_{\min,t})$ and
    $\widehat V_{\max}^{\texttt{OPT}}(T) = \max_{1 \le t \le T} V(\rho_{\max,t})$
  \end{algorithmic}
\end{algorithm}

\begin{theorem}\label{thm: a.s. convergence in Markov}
  Under Assumptions~\ref{asp: regularity}, \ref{asp: finite-dimension}, \ref{assu: continuous}, and \ref{asp: OPT}, the outputs of Algorithm~\ref{alg:sampling with optimization oracle} satisfy almost-sure convergence:
  \[
    \widehat V_{\min}^{\texttt{OPT}}(T)\xrightarrow{\mathrm{a.s.}} V_{\min},
    \qquad
    \widehat V_{\max}^{\texttt{OPT}}(T)\xrightarrow{\mathrm{a.s.}} V_{\max}.
  \]
\end{theorem}

We demonstrate the effectiveness of combining our hit-and-run sampler with a local-optimization oracle on a synthetic POCB example; see \cref{subsubsec in appendix: numerical setups for sampling general} for details.  First, we run Algorithm~\ref{alg: hit-run for density} to generate $10^4$ feasible causal-model densities $\{\rho_t\}$.  Each $\rho_t$ then initializes both a minimization oracle $\texttt{OPT}_{\min}(\rho_t)$ and a maximization oracle $\texttt{OPT}_{\max}(\rho_t)$.  Figure~\ref{fig: sampling general} presents three histograms: the left and right panels show the oracle outputs, and the middle panel shows the raw evaluations $V(\rho_t)$ prior to optimization.  Despite the highly non-convex landscape and multitude of local optima, the sampler produces sufficiently diverse starting points so that the oracles reliably approach the near-global bounds.
Moreover, whereas the raw sampler outputs (middle) converge slowly toward the extremes (as guaranteed by Theorem~\ref{thm: convergence of empirical bounds}), the optimization steps concentrate samples quickly at the true minimum and maximum, yielding high density at the boundaries.

\begin{figure}[hbtp]
  \centering
  \includegraphics[width=\textwidth]{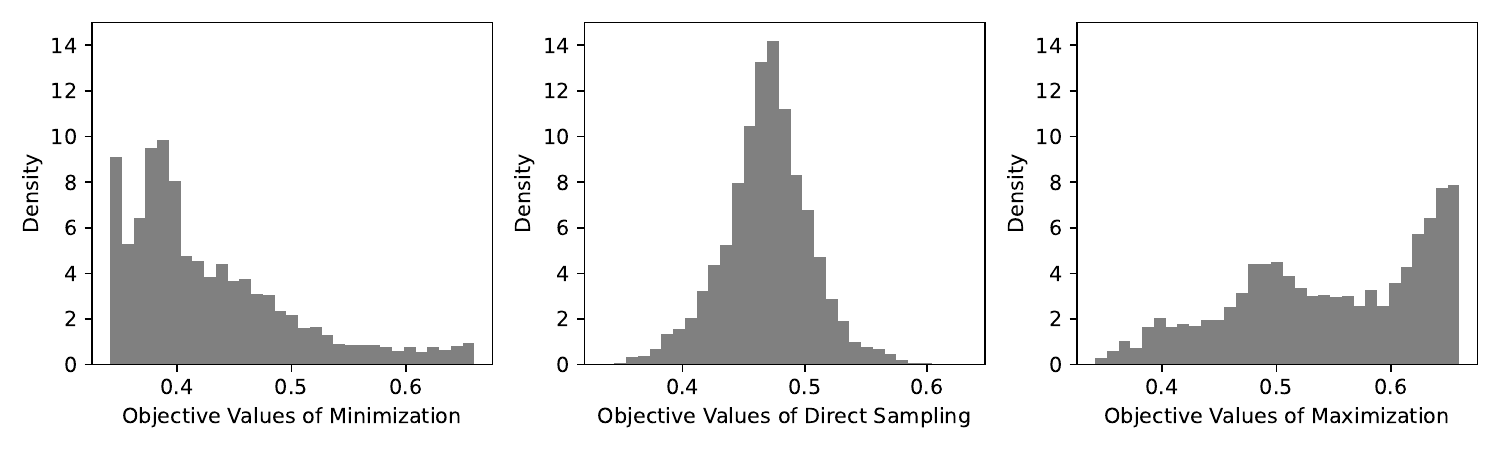}
  \caption{Histograms of Sampled Causal-Effect Values}
  \label{fig: sampling general}
\end{figure}

\subsection{Incorporating Parameter Uncertainty}\label{sec: parameter uncertainty}
To account for estimation error, we replace the original feasible set $\mathcal{P}$ from \eqref{eq: feasible set} by the relaxed set
\begin{align*}
  \mathcal{P}^{(N)}_{\epsilon_N}
  =
  \Bigl\{
    \rho\in L^2(\Omega,\nu) \Bigm|
    \rho\geq 0, \, \,& \textstyle\int_{\Omega} \rho \mathrm{d} \nu = 1,
    \, \, \bigl| \int_{\Omega} \alpha_i \rho \mathrm{d} \nu - \widehat\beta_i^{(N)}\bigr|\le\epsilon_N, \\
    & \textstyle\int_{\Omega} \alpha_j \rho \mathrm{d} \nu \le \widehat\beta_j^{(N)} + \epsilon_N,
    \, \,
    \forall i \in [m], j \in [m']\backslash [m]
  \Bigr\},
\end{align*}
where $\widehat{\bm{\beta}}^{(N)}$ are estimated target quantities, and the residual $\epsilon_N$ quantifies sampling error or possible distribution shift.
Whenever $\mathcal{P}^{(N)}_{\epsilon_N}$ contains the true model with high probability (e.g.\ via concentration inequalities), optimizing over it yields conservative, distributionally robust causal bounds that still cover the truth with the same confidence. If the residuals $\epsilon_N \to 0$ as $N \to \infty$, $\mathcal{P}^{(N)}_{\epsilon_N}$ collapses to $\mathcal{P}$, recovering exact bounds.

\begin{remark}
  In practice, our relaxed feasible set $\mathcal{P}^{(N)}_{\epsilon_N}$ offers robustness under several common sources of uncertainty, while guaranteeing convergence to the true bounds as uncertainty vanishes:
  \begin{itemize}
    \item \textit{Estimation error.}
      When moment estimates $\widehat{\bm{\beta}}^{(N)}$ are noisy (e.g., small clinical cohorts) a tolerance $\epsilon_N = O(1/\sqrt{N})$ absorbs sampling variability.  As the sample size $N$ grows, $\epsilon_N\to0$ and the optimal bounds converge to the true values.

    \item \textit{Distribution shift.}
      In settings such as recommendation systems facing user-preference drift, one may set a fixed $\epsilon_N>0$ to account for persistent shifts.  If the environment stabilizes (e.g., after market saturation) and the shift magnitude decays, letting $\epsilon_N\to0$ ensures convergence to the stationary optimum.

    \item \textit{Privacy protection.}
      Under differential privacy, database queries inject noise scaling like $O(1/\sqrt{N})$ \citep{hanLDP}.  Our framework absorbs this noise into $\epsilon_N$, yet still guarantees that as $N\to\infty$ (and $\epsilon_N\to0$), the computed causal bounds recover the true effects.
  \end{itemize}
\end{remark}

Before turning to convergence analysis, we impose the following regularity on our feasible sets.
Recall that $\mathcal{K}$ is the compact superset from Assumption \ref{assu: continuous}.
\begin{assumption}
  \label{asp: finite-measure and uniform compactness}
  For every $N$, the relaxed feasible set $\mathcal{P}^{(N)}_{\epsilon_N} \subset \mathcal{K}$. Furthermore, the measure $\nu$ is finite on $\Omega$, i.e., $\nu(\Omega)<\infty$.
\end{assumption}

We will show that, as $\epsilon_N\to0$, the sets $\mathcal{P}^{(N)}_{\epsilon_N}$ converge to $\mathcal{P}$ in the Hausdorff metric.
Recall that for any two subsets $\mathcal{P}_1,\mathcal{P}_2\subset L^2(\Omega,\nu)$, their Hausdorff distance is
\[
  d_H(\mathcal{P}_{1},\mathcal{P}_{2})
  = \max\biggl\{
    \sup_{a\in \mathcal{P}_{1}}\inf_{b\in \mathcal{P}_{2}}\norm{a-b}_{L^2(\Omega,\nu)},\;
    \sup_{b\in \mathcal{P}_{2}}\inf_{a\in \mathcal{P}_{1}}\norm{a-b}_{L^2(\Omega,\nu)}
  \biggr\}.
\]

\begin{proposition}
  \label{prop: hausdorff convergence for extended feasible set}
  Under Assumptions \ref{asp: regularity} and \ref{asp: finite-measure and uniform compactness}, let $\widehat{\bm{\beta}}^{(N)}$ be an estimate satisfying $\norm{\bm{\beta} - \widehat{\bm{\beta}}^{(N)} }_{\infty} \leq \epsilon_N$ and $\lim_{N\to\infty}\epsilon_N=0$.
  Then the relaxed feasible sets converge in Hausdorff distance \[\lim_{N\to\infty} d_H(\mathcal{P}^{(N)}_{\epsilon_N},\mathcal{P}) = 0.\]
\end{proposition}

We now establish the Hausdorff convergence rate of the feasible region $\cPeps{N}$ to the polytope $\mathcal{P}$,
which is crucial for controling the uncertainty levels.

\begin{proposition}\label{prop: hausdorff-rate}
  Assume \cref{asp: regularity} and \cref{asp: finite-measure and uniform compactness}, and suppose that there exists a Slater point $\rho^*\in\mathcal{P}$ satisfying $\rho^*(\bm x)\stackrel{\nu\text{-a.e.}}{\ge}\delta>0$ and $\int_{\Omega} \alpha_j(\bm x) \rho^*(\bm x) \mathrm{d}\nu(\bm x) \le \beta_j-\gamma_j$ for all $j \in [m']\backslash [m]$ for some constants $\delta,\gamma_j>0$.  If $\|\widehat{\bm{\beta}}^{(N)}-\bm{\beta}\|_{\infty}\le\epsilon_N \to 0$, then there is a constant $L_H = L_H(\delta,\gamma_j,a_j,\mathcal K,m,\nu(\Omega))>0$ such that
  \[
    d_H\bigl(\mathcal{P}^{(N)}_{\epsilon_N}, \mathcal{P}\bigr)
    \le
    L_H \epsilon_N.
  \]
\end{proposition}

Define the empirical bounds
\[
  V_{\min}^{(N)}  = \min_{\rho\in\mathcal{P}^{(N)}_{\epsilon_N}}V(\rho),
  \quad
  V_{\max}^{(N)}  = \max_{\rho\in\mathcal{P}^{(N)}_{\epsilon_N}}V(\rho).
\]
We now show that if $\mathcal{P}^{(N)}_{\epsilon_N}\to\mathcal{P}$ in Hausdorff distance, then $V_{\min}^{(N)}\to V_{\min}$ and $V_{\max}^{(N)}\to V_{\max}$.

\begin{theorem}\label{thm: convergence of empirical bounds}
  Under \cref{asp: regularity},\ref{assu: continuous} and \ref{asp: finite-measure and uniform compactness}, we have
  $\lim_{N\to\infty}V_{\min}^{(N)}= V_{\min}$
  and
  $\lim_{N\to\infty}V_{\max}^{(N)}= V_{\max}$.
\end{theorem}

We can now combine the sampling procedure in Algorithm~\ref{alg:sampling with optimization oracle} with the uncertainty-robust feasible set $\mathcal{P}^{(N)}_{\epsilon_N}$ to obtain distributionally robust causal bounds.
If the target quantity $V(\rho)$ is further Lipschitz continuous in the $L^2(\Omega,\nu)$ norm (see examples in Proposition~\ref{prop: continuity of E[Y|do(a)]} and Proposition~\ref{prop: continuity of E[Y|do(a),w]}),
then the difference of empirical bounds with the true bounds can be further controlled by the uncertainty level.

\begin{theorem}\label{thm: difference of empirical bounds and true bounds}
  Suppose that
  $
  |V(\rho_1) - V(\rho_2)| \leq L_V \|\rho_1 - \rho_2\|_{L^2(\Omega,\nu)}, \forall \rho_1,\rho_2\in\mathcal{K}
  $
  and
  $d_H\left( \cPeps{N}, \mathcal{P} \right) \leq L_H \epsilon_N$,
  then
  $ |V_{\min}^{(N)} - V_{\min}| \leq L_VL_H \epsilon_N $
  and
  $ |V_{\max}^{(N)} - V_{\max}| \leq L_VL_H \epsilon_N $.
\end{theorem}

\begin{remark}[Infinite-dimensional function spaces]
  The convergence guarantees of Theorem \ref{thm: difference of empirical bounds and true bounds} hold more generally for infinite-dimensional function spaces.
  In particular, one can apply our hit-and-run sampler to any finite-dimensional approximation of an otherwise infinite-dimensional function space.
  This allows us to solve optimization problems on infinite-dimensional function spaces.
  In Section \ref{sec: approximation of infinite function space}, we illustrate this by constructing two common infinite-dimensional spaces and showing that their finite-dimensional approximation converge in the Hausdorff metric.
  Consequently, even when we optimize over these finite subspaces, Theorem \ref{thm: difference of empirical bounds and true bounds} ensures that the approximated solutions $V_{\min}^{(N)}$ and $V_{\max}^{(N)}$ (which can be obtained using our sampling method) converge to the true bounds $V_{\min}$ and $V_{\max}$ as the approximation dimension grows.
\end{remark}

\section{Transfer Learning for Multi-Armed Bandits}\label{sec: transfer}

In this section, we consider two contextual-bandit agents sharing the same SCM $\mathcal{M}$ in \cref{fig: causal model with unobserved U and observed W}.
A fully-observable contextual bandit (FOCB) source agent sees both $W$ and a sensitive attribute $U$, chooses action $A$, and receives reward $Y\in [0,1]$.
A multi-armed bandit (MAB) target agent observes none of $W$ and $U$ while choosing action $A$ and receiving reward $Y$; see the dotted box in \cref{fig: causal model with unobserved U and observed W} for an illustration.
The MAB target agent aims to learn the causal effect of $A$ on $Y$, i.e., $\mu_a \triangleq \mean{Y|\mathrm{do}(A=a)}$, for each arm $a\in\mathcal{A}$.
Let $\mu^*=\max_{a\in \mathcal{A}} \mu_a $ and $a^*$ denote the optimal expected reward and the optimal arm, respectively.
The goal is to minimize the cumulative regret under all offline knowledge represented by $\mathcal{H}$:
\begin{equation}
  \label{eq: MAB regret}
  \textstyle
  \mathrm{Reg}(T) =\mathbb{E} \Bigl[T\mu^* - \sum_{t=1}^T \mu_{a_t} \,\Bigm|\, \mathcal{H}\Bigr],
\end{equation}
where the expectation is taken with respect to the randomness of the algorithm.

Due to privacy concerns, the FOCB only passes $(A,Y,W)$ to the target MAB agent, which must learn without ever observing $U$.
We leverage offline data from a FOCB source agent to accelerate learning in a MAB target agent.
For any target causal effects under any set of compatible causal models, we assume that the offline knowledge has been encoded in a set of valid causal bounds, e.g., obtained by \Cref{alg:sampling with optimization oracle}.
In Section \ref{sec: transfer MAB valid bounds}, we show how to incorporate these bounds to safely prune suboptimal arms and truncate UCB indices, yielding provably faster regret rates than vanilla UCB.
In Section \ref{sec: transfer noisy}, we further extend our framework to explicitly incorporate estimation noise or bias in the causal bounds, which commonly arises when offline data is limited or when distribution shift exists between source and target.

\subsection{Transfer Learning via Valid Causal Bounds}\label{sec: transfer MAB valid bounds}
To illustrate, we take as running example the case where our only prior knowledge consists of the marginal distributions $\rho(a,y,w)$ of the observed variables and the marginal distribution $\rho(u)$ of the unobserved variables.
In our framework, these marginals are imposed as linear constraints on the compatible-model class (cf. Examples \ref{ex:constraints}\ref{ex:observational distribution} and \ref{ex:constraints}\ref{ex:known marginal distribution}). If $W$ or $U$ is continuous, we enforce each marginal constraint only at a finite set of context values, so that the overall number of constraints remains finite.
Focusing on a finite set of constraints arises naturally in real-world offline data settings, where you only ever observe a limited number of context values.
Moreover, limiting the model to this finite grid of contexts helps prevent overfitting, sacrificing only a bit of flexibility in exchange for substantially greater stability.

\subsubsection{Causal Bound Calculation.}\label{sec: causal bounds MAB}
Using do-calculus, the interventional mean reward is
\begin{align}
  V_a(\rho) \triangleq\mathbb{E}_{\rho}\bigl[Y\mid \mathrm{do}(A=a)\bigr]
  & = \int_{\mathcal W}\int_{\mathcal U} \mathbb{E}\bigl[Y\mid A=a,w,u\bigr] \rho(w,u) \diff u \diff w \notag \\
  & = \int_{\mathcal W}\int_{\mathcal U}\int_{\mathcal Y} y \rho(y \mid a,w,u) \diff y \; \rho(w,u) \diff u \diff w,
\end{align}
where the second line follows from the usual back-door adjustment.
Consequently, the causal-effect bounds for arm $a$ reduce to solving $l(a) = \min_{\rho \in\mathcal{P}} V_a(\rho)$ and $ h(a)= \max_{\rho \in\mathcal{P}} V_a(\rho),$
where $\mathcal{P}$ is the convex polytope of all joint densities $\rho(a,y,w,u)$ satisfying the linear constraints on the marginals of $(A,Y,W)$ and $U$ (evaluated at a finite grid of $w,u$ values), as in \eqref{eq: feasible set}.
These two generally non-convex programs instantiate the generic causal-bounds problem \eqref{eq: causal bounds} in the MAB setting.
By running Algorithm~\ref{alg:sampling with optimization oracle}, we obtain consistent estimators $\widehat{l}(a)$ and $\widehat{h}(a)$ for the true bounds $l(a)$ and $h(a)$ of each arm.

\subsubsection{Arm Elimination and Upper Confidence Bounds Truncations.}
\label{subsec: transfer learning in MAB}

To isolate the impact of causal bounds on learning speed, we assume that for each arm $a\in\mathcal{A}$ we have valid bounds
\[
  \mu_a = \mean{Y |\mathrm{do}(A=a)} \in \bigl[\,\widehat{l}(a), \widehat{h}(a)\bigr].
\]
This reflects the case where abundant offline data are available under the same causal model, so that causal bounds can be estimated accurately without bias.
For notational simplicity in this subsection, we write $\widehat{l}(a) \equiv l(a)$ and $\widehat{h}(a)\equiv h(a)$, respectively.
We turn next, in Section \ref{sec: transfer noisy}, to the more challenging case of noisy or misspecified bounds to expose their impact.

We present the proposed algorithm in \cref{alg: TL-MAB}, which prunes active arms and truncates the standard UCB algorithm using the valid causal bounds.
First, any arm $a$ for which $h(a)<\max_{i\in\mathcal{A}} l(i)$ is pruned.
Indeed, there exists $a'\in\mathcal{A}$ such that $\mu_a < h(a) < l(a') < \mu_{a'}$ for any such arm, implying that it is guaranteed to be suboptimal.
We denote the remaining active arms as
\begin{equation}\label{eq:prune_mab_action}
  \mathcal{A}^* \triangleq \Bigl\{a \in \mathcal{A}: h(a) \ge \max_{i\in\mathcal{A}} l(i)\Bigr\}.
\end{equation}
Next, the upper confidence bound $U_a(t)$ for each remaining arm is truncated into $\widehat{U}_a(t) = \min\{U_a(t), h(a)\}$. Since $\mu_a \leq h(a)$, this truncation preserves optimism while incorporating offline evidence.
To further reduce the confidence width, we define the maximum variance of the reward distribution for each action $a\in\mathcal{A}$ as
\begin{equation}
  \label{eq: maximum variance}
  {\sigma}_a^2 = \max\Bigl\{\mu(1-\mu):  \mu \in \bigl[{l}(a),{h}(a)\bigr]\Bigr\}.
\end{equation}
The algorithm then selects the arm with the highest truncated UCB.

\begin{algorithm}[hbtp]
  \renewcommand{\algorithmicrequire}{\textbf{Input:}}
  \renewcommand{\algorithmicensure}{\textbf{Output:}}
  \caption{Transfer Learning for Multi-Armed Bandit with Valid Causal Bounds}
  \label{alg: TL-MAB}
  \begin{algorithmic}[1]
    \Require time horizon $T$, causal bounds $[l(a),h(a)]$ for each arm $a\in\mathcal{A}$, confidence parameter $\delta$
    \State Remove any arm $a$ for which $h(a)<\max_{i\in\mathcal{A}} l(i)$ to obtain the active arm set $\mathcal{A}^*$
    \State Compute the maximum variance for each arm according to \eqref{eq: maximum variance}
    \State Initialize the empirical mean $\widehat\mu_{a}(1) = 0$ and the number of pulls $n_{a}(1)  = 0$
    \For{$t = 1,\cdots,T$}
    \For{each arm $a \in \mathcal{A}^*$}
    \State Compute the upper confidence bound $U_a(t) = \min\left\{1, \widehat{\mu}_a(t) + \sqrt{2\sigma_a^2 \log (2t/\delta) / n_a(t)} \right\}$
    \State Compute the truncated UCB $\widehat{U}_a(t) = \min\{ U_a(t), h(a) \}$
    \EndFor
    \State Pull arm $a_t = \argmax_{a \in \mathcal{A}^*} \widehat{U}_a(t)$ and observe a reward $y_t$
    \State Update $\widehat{\mu}_{a_t}(t+1) = \bigl(\widehat{\mu}_{a_t}(t) \cdot n_{a_t}(t) + y_t\bigr) / (n_{a_t}(t) + 1)$ and $n_{a_t}(t+1) = n_{a_t}(t) + 1$
    \State For each $a\neq a_t$, update $\widehat{\mu}_{a}(t+1) = \widehat{\mu}_{a}(t)$ and $n_{a}(t+1) = n_{a}(t)$
    \EndFor
  \end{algorithmic}
\end{algorithm}

\subsubsection{Regret Analysis.}
We first upper bound the expected number of pulls for each sub-optimal arm.

\begin{theorem}
  \label{thm: number of pulling in MAB}
  For a MAB problem with a finite action set $|\mathcal{A}| < \infty$ and rewards bounded within $[0,1]$,
  the number of draws $\mathbb{E}[n_a(T)]$ in \cref{alg: TL-MAB} with $\delta=\frac{1}{T}$ for any sub-optimal arm $a \neq a^*$ is upper bounded as follows:
  $$
  \mathbb{E}[n_a(T)] \leq
  \begin{cases}
    0, & \text{if } a \notin \mathcal{A}^*, \text{ i.e., }h(a)<\max_{i\in\mathcal{A}} l(i) \le \mu^*, \\
    |\mathcal{A}|, & \text{if } a \in \mathcal{A^*} \text{ and } h(a) < \mu^*,  \\
    8\sigma_a^2\log (T) / \Delta_a^2, & \text{if } a \in \mathcal{A^*} \text{ and } h(a) \geq \mu^*.
  \end{cases}
  $$
  where $\Delta_a = \mu^*- \mu_a$ is the sub-optimality gap for the arm $a$.
\end{theorem}
\Cref{thm: number of pulling in MAB} formally demonstrates that causal bounds accelerate learning by classifying arms into three categories: first, definitively suboptimal arms with $\{a \mid h(a)<\max_{i\in\mathcal{A}}l(i)\}$ are eliminated without any pulls; then, arms with a large-gap but informative causal bounds, i.e., those $a\in\mathcal{A}^*$ with $h(a)<\mu^*$, which the theorem shows are pulled only $\mathcal{O}(1)$ times; finally, potentially optimal arms with inconclusive causal bounds, i.e., those $a\in\mathcal{A}^*$ with $h(a)\geq \mu^*$, attract the main exploration budget.
We collect these potentially optimal arms into $\widetilde{\mathcal{A}^*} \triangleq \left\{ a \in\mathcal{A} \mid h(a) \ge \mu^* \right\}\subset \mathcal{A}.$
As a result, our algorithm achieves strictly lower regret than vanilla UCB: tight bounds eliminate hopeless arms early, while truncated UCB curbs over-optimism even when bounds are less informative.  The following corollary of Theorem~\ref{thm: number of pulling in MAB} makes this precise.
Denote the family of instances with rewards bounded within $[0,1]$ and consistent with causal bounds $\mu_a \in [l(a), h(a)]$ as
$$
\mathfrak{M} = \Bigl\{ \text{MAB instances with } Y\in[0,1], \; l(a) \leq \mu_{a} \leq h(a), \forall a\in\mathcal{A} \Bigr\}.
$$

\begin{theorem}\label{thm:instance_LB_UB_MAB_with_causal_bounds}
  For a fixed instance in $\mathfrak{M}$ and horizon $T$, let $a^*=\arg\max_a\mu_a$ denote the best arm and define $\Delta_a=\mu_{a^*}-\mu_a>0$.
  Then the regret of \cref{alg: TL-MAB} is upper bounded by
  \[\textstyle
    \mean{ \mathrm{Reg}(T)}  \leq  \mathcal{O}\Bigl(\sum_{a\in \widetilde{\mathcal{A}}^*:\Delta_a>0 } \log T / \Delta_a \Bigr), \quad \text{where } \widetilde{\mathcal{A}}^* = \left\{ a \in\mathcal{A} \mid h(a) \ge \mu^* \right\}.
  \]
  Furthermore, for any consistent algorithm $\pi$, the expected regret satisfies:
  \[\textstyle
    \mathbb{E}[\mathrm{Reg}(T)] \geq \Omega\Bigl(\sum_{a\in \widetilde{\mathcal{A}}^*:\Delta_a>0 } \log T / \Delta_a \Bigr).
  \]
\end{theorem}

By \cref{thm:minimax_UB_MAB_with_causal_bounds}, \cref{alg: TL-MAB} is a consistent algorithm.
Consequently, \cref{thm:instance_LB_UB_MAB_with_causal_bounds} demonstrates that the regret bounds of \cref{alg: TL-MAB} are tight up to constant factors for any given instance.
\begin{theorem}\label{thm:minimax_UB_MAB_with_causal_bounds}
  For any horizon $T$, the worst-case regret of \cref{alg: TL-MAB} satisfies:
  \[
    \sup_{\mathfrak{M}} \mathbb{E}[\mathrm{Reg}(T)] \leq \mathcal{O}\left( \min\Bigl\{ \max_{a\in\mathcal{A}^*} \sigma_a \sqrt{|\mathcal{A}^*| T \log T},  wT \Bigr\} \right),
  \]
  where $w = \max_{a\in\mathcal{A}^*} (h(a) -  l(a)) > 0$ is the largest confidence width.
\end{theorem}

We now derive a matching lower bound on the minimax regret under this transfer learning setting.

\begin{theorem}\label{thm:minimax_LB_MAB_with_causal_bounds}
  Suppose that $h(a)-l(a)=w$ for all arms $a\in\mathcal{A}^*$.
  Furthermore, there exist constants $\mu_0\in[0,1]$ and $\kappa \in (0,0.5)$ such that
  $ [\mu_0,\mu_0+\kappa w] \subset [l(a),h(a)]$ for all arms $a\in\mathcal{A}^*$.
  Then for any algorithm $\pi$,
  \[
    \sup_{\mathfrak{M}} \mathbb{E}[\mathrm{Reg}(T)] \geq \Omega\left( \min\bigl\{ \sqrt{|\mathcal{A}^*| T},  wT \bigr\} \right).
  \]
\end{theorem}

\subsection{Transfer Learning via Noisy Causal Bounds}\label{sec: transfer noisy}

In Section \ref{subsec: transfer learning in MAB}, we assumed access to valid causal bounds that hold for every arm, which permits straightforward arm elimination and UCB truncation.
In practice, however, causal-bound estimates can be noisy when offline data is limited, or biased due to distribution shift between source and target.
In this section, we show how to incorporate such noisy bounds into online learning while retaining performance guarantees.
We begin with the following assumption on the quality of the noisy bounds and their associated confidence level.

\begin{assumption}[Noisy causal bounds]
  \label{asp: causal bounds with noisy estimates}
  There exists a nonnegative function $\epsilon_a(\delta)$ such that the estimated causal bounds $\widehat l(a)$ and $\widehat h(a)$ satisfies
  $$
  \mathbb{P}\Bigl(
    \forall\,a\in\mathcal{A}:\,
    \bigl|\widehat h(a)-h(a)\bigr|\le\epsilon_a(\delta)
    \;\text{ and }\;
    \bigl|\widehat l(a)-l(a)\bigr|\le\epsilon_a(\delta)
  \Bigr) \ge 1-\delta.
  $$
\end{assumption}

Since $\widehat{l}(a)$ and $\widehat{h}(a)$ may be noisy or biased, we cannot eliminate suboptimal arms with absolute confidence.
How, then, can we safely leverage these uncertain bounds in an online learning algorithm, and how does their uncertainty impact regret? To address this, we propose a bandit algorithm that integrates the estimated causal intervals $[\widehat{l}(a), \widehat{h}(a)]$ into UCB algorithms, ensuring robust performance even under noisy bound estimates.

\begin{example}
  We present concrete examples where \Cref{asp: causal bounds with noisy estimates} is satisfied:
  \begin{enumerate}[label=(\alph*)]
    \item \textit{Lipschitz-based error propagation from parameter estimates.}
      By Theorem~\ref{thm: difference of empirical bounds and true bounds}, the causal bounds $l(a)$ and $h(a)$ are Lipschitz continuous functions of the underlying parameter vector $\bm\beta$.  If our offline sample yields an estimate $\widehat{\bm\beta}^{(N)}$ satisfying $\prob{ \norm{\bm\beta - \widehat{\bm\beta}^{(N)}  }_{\infty} < \epsilon_N } \geq 1-\delta$.
      Then by Lipschitz continuity, we can set $\epsilon_a(\delta) = L_H L_V \epsilon_N$ to satisfy Assumption~\ref{asp: causal bounds with noisy estimates}.

    \item \textit{Distribution shift with bounded differences.}
      \citet{biased_offline} assume that $|\mu^{\text{on}}_a - \mu^{\text{off}}_a|\leq \Xi(a)$,
      where $\mu^{\text{on}}_a ,\mu^{\text{off}}_a$ are mean rewards for the arm $a$ in the online and offline environments, respectively,
      and $\Xi(a)$ quantifies the distribution shift.
      Let $\bar{y}_a$ be the offline empirical mean for arm $a$.  Define $\widehat h(a)= \bar y_a+\Xi(a)$ and $\widehat l(a)= \bar y_a-\Xi(a)$.
      Given $T^S_a$ offline samples for arm $a$, standard sub-Gaussian concentration yields:
      $\bigl|\bar{y}_a- \mu^{\text{off}}_a \bigr|\le \mathcal{O}\bigl(\sqrt{\log(1/\delta)/T^S_a}\bigr)$
      with probability $1-\delta$.  Hence setting
      $\epsilon_a(\delta)=\mathcal{O}\bigl(\sqrt{\log(|\mathcal{A}|/\delta)/T^S_a}\bigr)$
      ensures Assumption~\ref{asp: causal bounds with noisy estimates}.

    \item \textit{Nonparametric bounds via observational constraints.} Even without any parametric assumptions, we always have for each arm-reward pair $a,y$: $\rho(a,y) \leq \rho(y \mid \mathrm{do}(A = a)) \leq 1- \sum_{y'\neq y} \rho(a,y')$.
      The quantity $\epsilon_a(\delta)$ can be derived from the empirical estimation error of the distribution $\rho(a,y)$.
      For example, if $Y$ is binary, let $\widehat{\rho}(a,y)$ be the empirical estimators. We may define $\widehat{l}(a) = \widehat{\rho}(a,1), \widehat{h}(a) = 1- \widehat{\rho}(a,0)$
      and set $\epsilon_a(\delta) = \max \{|\widehat{\rho}(a,1) - {\rho}(a,1)|, |\widehat{\rho}(a,0) - {\rho}(a,0)| \}$.
      Moreover, these bounds extend immediately to conditional versions, since $\rho(a,y\mid w) \leq \rho(Y = y\mid\mathrm{do}(a),w) \leq 1- \sum_{y'\neq y}  \rho(a,y'\mid w)$.
  \end{enumerate}
\end{example}

\subsubsection{Transfer Learning for MAB with Noisy Causal Bounds.}
In this section, we present a transfer learning algorithm for MAB that exploits noisy causal bounds estimated from offline data to improve online performance even under covariate shift.
Based on the estimated causal bounds, we define the set of candidate actions $\widehat{\mathcal{A}}$ using the estimated causal bounds and their confidence intervals:
\begin{equation}\label{eq: active action set}
  \widehat{\mathcal{A}} =\mathcal{A} - \left\{ a\in\mathcal{A} \; \left| \; \widehat{h}(a) + \epsilon_a(\delta) < \max_{i\in\mathcal{A}} \Bigl[\widehat{l}(i) -  \epsilon_i(\delta)\, \Bigr] \right.  \right\}.
\end{equation}
\cref{asp: causal bounds with noisy estimates} guarantees that, with probability at least $1-\delta$, every action $a\notin\widehat{\mathcal{A}}$ is suboptimal, and the optimal arm remains in $\widehat{\mathcal{A}}$.
For brevity, let $\widehat{\mathcal{E}}$ denote this ``good'' event of retaining the best arm.
Our algorithm will only consider actions in $\widehat{\mathcal{A}}$, which are those that are not eliminated with high probability.
As a side benefit, these bounds also tighten our estimate of each arm's maximum reward variance defined as
\[
  \widehat{\sigma}_a^2 = \max\Bigl\{\mu(1-\mu):  \mu \in \bigl[\, \widehat{l}(a) - \epsilon_a(\delta) ,\widehat{h}(a) + \epsilon_a(\delta) \bigr]\Bigr\}.
\]

For each active arm in $\widehat{\mathcal{A}}$, we compute two upper confidence bounds.  First, the conventional UCB:
\begin{equation}\label{eq:conventional UCB}
  U_a(t) \triangleq \widehat{\mu}_{a}(t) + \sqrt{ 2\widehat{\sigma}_a^2 \log\bigl(2t/\delta\bigr) / n_{a}(t)}
\end{equation}
where $\widehat{\mu}_{a}(t)$ is the empirical mean reward of action $a$ after $n_{a}(t)$ pulls during online learning.
Second, the warm-start UCB incorporates noisy causal bounds:
\begin{equation}\label{eq:warm-start UCB}
  U^{\epsilon}_a(t) \triangleq \widehat{\mu}^{\epsilon}_a(t) + \sqrt{\frac{ 2\widehat{\sigma}_a^2\log\bigl(2t/\delta\bigr)  + 1  }{ n_{a}(t)+ \epsilon^{-2}_{a}(\delta) }}, \quad \text{where} \quad  \widehat{\mu}^{\epsilon}_a(t) \triangleq \frac{n_{a}(t) \cdot \widehat{\mu}_{a}(t) + \epsilon^{-2}_a(\delta) \cdot \widehat{h}(a)}{n_{a}(t) + \epsilon^{-2}_a(\delta)}.
\end{equation}
Here, the warm-start UCB combines the action's empirical mean during online learning with its estimated causal upper bound $\widehat{h}(a)$ from the offline data, weighted by the number of online pulls $n_{a}(t)$ and estimation accuracy $ \epsilon^{-2}_a(\delta)$.
The next lemma establishes that both $ U_a(t) $ and $ U^{\epsilon}_a(t) $ are valid confidence bounds.
\begin{lemma}
  \label{lem: warm-start UCB event}
  Define the event
  \begin{align*}
    \mathcal{E} = & \Bigg\{
      \bigl|\widehat{\mu}_{a}(t)-\mu_a\bigr|
      \le
      \sqrt{\frac{ 2\widehat{\sigma}_a^2 \log\bigl(2t/\delta\bigr)}{ n_{a}(t)}}, \quad
      |\widehat{\mu}^{\epsilon}_a(t) - \mu_a   |
      \le
      \sqrt{\frac{ 2\widehat{\sigma}_a^2\log\bigl(2t/\delta\bigr)  + 1  }{ n_{a}(t)+ \epsilon^{-2}_{a}(\delta) }}+ \frac{  \epsilon^{-2}_a(\delta) (h(a)-\mu_a) }{n_{a}(t) + \epsilon^{-2}_a(\delta) }
    \Bigg\}.
  \end{align*}
  Then $\mathbb{P}(\mathcal{E}) \geq 1 - 2|\widehat{\mathcal{A}}|\delta.$
\end{lemma}

At each time $t$, we pick the action $ a_t $ that maximizes the \textit{minimum} of the computed upper confidence bounds $ U_a(t) $ and $ U^{\epsilon}_a(t) $.
By taking the smaller of the two confidence bounds, the algorithm incorporates the causal information conservatively.
If $ U^{\epsilon}_a(t) $ is much larger than $ U_a(t) $, the selection will favor relying on the online data alone.
\cref{alg: MAB with noisy causal bounds} summarizes this procedure.

\begin{algorithm}
  \caption{Transfer Learning for Multi-Armed Bandit with Noisy Causal Bounds}
  \label{alg: MAB with noisy causal bounds}
  \begin{algorithmic}[1]
    \State \textbf{Input:} Confidence parameter $\delta$, noisy bounds $\bigl\{(\widehat{l}(a), \widehat{h}(a)), \epsilon_a(\delta)\bigr\}_{a\in\mathcal{A}}$
    \State Initialize  $n_a(0) = 0$,  $\widehat\mu_a(0) = 0$ for each $a\in\mathcal{A} $
    \State Compute active set $\widehat{\mathcal{A}}$ via \eqref{eq: active action set}
    \For{$t = 1, \cdots$}
    \State Compute upper confidence bounds $U_{a}(t)$ and $U^{\epsilon}_{a}(t)$ by \eqref{eq:conventional UCB} and \eqref{eq:warm-start UCB} for each $a\in\widehat{\mathcal{A}}$
    \State Select
    $
    a_t \in \argmax_{a\in\widehat{\mathcal{A}}} \left\{ U_{a}(t) \wedge U^{\epsilon}_{a}(t)  \right\}.
    $
    \State Update number of pulls $i$ as $n_{a}(t+1) = n_{a}(t) + \ind{a=a_t}$
    \State Observe reward $y_t$ and update empirical means $\widehat{\mu}_{a}(t+1) = (\widehat{\mu}_{a}(t)n_{a}(t) + y_t\ind{a=a_t}) / (n_{a}(t+1)) $
    \EndFor
  \end{algorithmic}
\end{algorithm}
We remark that, when the estimation error $\epsilon_a(\delta)$ is small (i.e., $ \epsilon_a(\delta) \to 0$ for all $a\in\mathcal{A}$),
the algorithm reduces to \cref{alg: TL-MAB} with valid causal bounds.
The new \cref{alg: MAB with noisy causal bounds} explicitly handles the intricate interplay between noise in the causal estimates, and balance of online-offline learning.

\subsubsection{Regret Analysis.}
We now analyze the regret of \cref{alg: MAB with noisy causal bounds} using the estimated causal bounds. In particular, we highlight how eliminating suboptimal arms and employing the warm-start UCB improve the regret guarantees.
For each arm $a\in\mathcal{A}$, we define
\begin{equation}\label{eq:H_a}
  H_a =  \epsilon^{-2}_a(\delta)  \cdot \bigl(\mu^* - h(a)\bigr)_+^2,
\end{equation}
where $x_+ = \max\{x,0\}$ denote the positive part.
We show that the number of pulls assigned to arm $a$ scales in the order of $\mathcal{O}\bigl((\log T - H_a)/\Delta_a\bigr)$, where $\Delta_a = \mu^* - \mu_a$ is the optimality gap.
This leads to the following upper bound on the regret of \cref{alg: MAB with noisy causal bounds}.
\begin{theorem}
  \label{thm: regret upper bound of MAB with noisy causal bounds}
  The regret of \cref{alg: MAB with noisy causal bounds} satisfies
  \[\textstyle
    \mathbb{E}[\mathrm{Reg}(T)] =  \mathcal{O}\Bigl(\sum_{a \in \widehat{\mathcal{A}}: \Delta_a > 0} \bigl( \widehat{\sigma}_a^2 \log( T ) - H_a\bigr)_+ / \Delta_a\Bigr).
  \]
\end{theorem}

\begin{remark}
  The term $H_a$ quantifies far below the true optimal reward $\mu^*$ the causal upper bound $h(a)$ lies.
  When $\mu^* > h(a)$, the squared gap $(\mu^* - h(a))_+^2$ shrinks the amount of necessary exploration for arm $a$.
  Hence larger $H_a$ cuts exploration for arm $a$.
  When the estimator $\epsilon_a(\delta)$ becomes small, $H_a$ increases as long as $h(a) < \mu^*$.
  Therefore, accurate causal bounds let the algorithm sharply reduce exploration as \cref{alg: TL-MAB} does.
  A special case arises when the error satisfy $\epsilon_a(\delta) = \mathcal{O}\bigl( (\mu^{*}-h(a)) / \sqrt{\widehat{\sigma}_a^2 \log T} \bigr).$
  In this case, the term $H_a$ dominates $\log\bigl(  T\bigr)$ for any arm with $\mu^* > h(a)$, leading to zero regret for that arm.
  This corresponds to the scenario where the estimate $\widehat{h}(a)$ is nearly exact estimate of the causal upper bound, allowing for near perfect elimination of all suboptimal arms with $h(a) < \mu^*$.
  This aligns with the theory established in \cref{thm:instance_LB_UB_MAB_with_causal_bounds}.
  By contrast, when $h(a)$ is very close to $\mu^*$, $H_a$ stays small and we must continue sampling $a$.
\end{remark}

\section{Transfer Learning for Partially Observable Contextual Bandits}
\label{sec: transfer learning to POCB}

In the previous section, we introduced transfer learning for MAB by leveraging exact and noisy causal bounds derived from offline data.
We now extend this framework to the partially observable contextual bandit (POCB) setting, where the agent observes a context at each round and selects an action based on both the context and historical information.

The causal mechanism for a POCB is illustrated in the dash-dotted box in \cref{fig: causal model with unobserved U and observed W}.
At each round $t = 1,2,\cdots,T$, the agent observe a context $w_t$ and performs an action $\mathrm{do}(A = a_t)$ based on the context and historical information.
For each arm $a\in\mathcal{A}$, the expected reward of intervention $a$ given context $w$ is $\mu_{a,w} \triangleq \mean{Y |\mathrm{do}(A=a),w}$.
Let $\mu^*_w$ denote the optimal expected reward with respect to context $w$.
The objective of the POCB agent is to minimize the cumulative regret $\mathrm{Reg}(T) =  \sum_{t=1}^T \Mean{ \mu^*_{w_t} -  \mu_{a_t,w_t} \mid w_t, \mathcal{H} },$
where the expectation is taken with respect to the randomness of the algorithm and $\mathcal{H}$ is the filtration capturing the known marginal distributions $\rho(a,y,w)$ and $\rho(u)$.

Using do-calculus, the conditional causal effect for a continuous contexts can be written as
\begin{align*}
  \mu(a, w) \triangleq \mathbb{E}\bigl[Y \mid \mathrm{do}(A=a),W=w\bigr]
  & = \int_{\mathcal U} \mathbb{E}\bigl[Y \mid A=a,W=w,U=u\bigr] \rho(u \mid w)\diff u \\
  & = \int_{\mathcal U} \int_{\mathcal Y} y \rho\bigl(y \mid a,w,u\bigr) \diff y \, \rho(u \mid w) \diff u.
\end{align*}
Similar to \cref{sec: causal bounds MAB}, we can derive valid causal bounds $l(a,w)$ and $h(a,w)$ for each $(a,w) \in \mathcal{A} \times \mathcal{W}$ such that $l(a,w) \leq \mu_{a,w} \leq h(a,w)$, serving as the prior knowledge for the POCB agent.

\subsection{Continuous Context Space via Function Approximation}
Without further structure on the reward function (e.g. linearity), finding the globally optimal mapping from every possible context to an action is generally intractable.
To cope with an continuous or richly structured context space, we adopt a function approximation framework.

Throughout this section, we assume the agent has access to a class of reward functions $\mathcal{F}\subset \mathcal{A}\times \mathcal{W} \to [0,1]$ that characterizes the mean reward distribution for a given context-action pair.
We make the standard realizability assumption \citep{fasterCB,CBwithRegressionOracles}:
\begin{assumption}
  \label{asp: realizability}
  There exists $f^*\in\mathcal{F}$ such that $f^*(a, w) = \mu_{a,w}$, for all $(a,w) \in \mathcal{A}\times\mathcal{W}$.
\end{assumption}
For any $f \in \mathcal{F}$, define the induced greedy policy $\pi_f (w) = \argmax_{a \in \mathcal{A}} f(a, w)$ and let $\Pi = \{\pi_f \mid f \in \mathcal{F}\}$ be the policy class.
Under Assumption~\ref{asp: realizability}, the cumulative regret is
\[
  \mathrm{Reg}(T) = \sum_{t=1}^T  \Mean{f^*(\pi_{f^*} (w_t), w_t) - f^*(a_t, w_t) \mid w_t, \mathcal{H}}.
\]

Given a set of valid causal bounds $[l(a,w),h(a,w)]$ for all $(a,w)$, we prune the function class $\mathcal{F}$, focusing on those that are consistent with the bounds.
Concretely, define the pruned function class
\[
  \mathcal{F}^* = \{  f\in\mathcal{F} \mid  l(a,w) \leq f(a,w) \leq h(a,w), \quad \forall (a,w) \}.
\]
Moreover, for each context $w$, we need only consider actions that could be optimal under some feasible $f \in \mathcal{F}^*$, i.e., a pruned action set
\begin{equation}\label{eq:prune_contextual_bandit_action}
  \mathcal{A}^*(w) = \left\{  a \in\mathcal{A} \,\Big|\, \exists f \in \mathcal{F}^* \text{ with } a = \argmax_{i \in \mathcal{A}} f(i,w) \right\}.
\end{equation}

By pruning function space and pruning action set, we eliminate both implausible reward models and actions that cannot be optimal, thereby focusing exploration on the most promising candidates.
We consequently propose \cref{alg: TL-function approximation} to integrate causal bounds and function approximation in transfer learning for contextual bandits,
inspired by the inverse gap weighting (IGW) technique \citep{CBwithOracle, CBwithPredictableRewards, instanceCB_RL, fasterCB}.

\begin{algorithm}[ht]
  \renewcommand{\algorithmicrequire}{\textbf{Input:}}
  \renewcommand{\algorithmicensure}{\textbf{Output:}}
  \caption{Transfer Learning for Contextual Bandit with Function Approximation}
  \label{alg: TL-function approximation}
  \begin{algorithmic}[1]
    \Require time horizon $T$, function space $\mathcal{F}$, confidence parameter $\delta$, tuning parameters $\eta$, and causal bounds $[l(a,w),h(a,w)]$
    \State Eliminate function space $\mathcal{F}$ and obtain $\mathcal{F}^*$ via causal bound
    \State Set epoch schedule $\{\tau_m=2^m,\forall m \in \mathbb{N} \}$
    \For{ epoch $m = 1,2,\cdots,\lceil \log_2 T \rceil$ }
    \State Compute the least square estimation
    $
    \widehat{f}_m = \argmin_{f\in\mathcal{F}^*} \sum_{t=1}^{\tau_{m-1}}( f(a_t,w_t) -y_t )^2
    $
    \For{ round $t = \tau_{m-1}+1,\cdots, \tau_{m}$ }
    \State Observe the context $w_t$
    \State Compute the best action candidate set $\mathcal{A}^*(w_t) $ by \eqref{eq:prune_contextual_bandit_action}
    \State Compute $\gamma_t =  \sqrt{  \frac{\eta |\mathcal{A}^*(w_t)| \tau_{m-1}  }{   \log (   2\delta^{-1}|\mathcal{F}^*|\log T   )  }     }    $ (for the first epoch, $\gamma_1 = 1$)
    \State  Compute $\widehat{f}_m(a,w_t)$ for each action $a\in\mathcal{A}^*(w_t)$, $\widehat{a}_t =\max_{a\in\mathcal{A}} \widehat{f}_m(a,w_t)$, and
    $$
    p_t(a) =
    \begin{cases}
      0 ,  & \text{ for all } a\in \mathcal{A} -\mathcal{A}^*(w_t),  \\
      \Bigl(|\mathcal{A}^*(w_t)|+\gamma_t \bigl(\widehat{f}_m (\widehat{a}_t ,w_t)- \widehat{f}_m ( a,w_t ) \bigr) \Bigr)^{-1},  & \text{ for all } a\in \mathcal{A}^*(w_t)-\{ \widehat{a}_t \} \\
      1- \sum_{a\neq \widehat{a}_t }  p_t(a), & \text{ for } a = \widehat{a}_t ,\\
    \end{cases}
    $$
    \State Sample $a_t \sim p_t(\cdot)$, take action $a_t$, and observe a reward $y_t$
    \EndFor
    \EndFor
  \end{algorithmic}
\end{algorithm}

\subsection{Regret Analysis of Transfer Learning with Function Approximation}

As noted by \citet{instanceCB_RL}, gap-dependent regret bounds are generally not feasible for contextual bandits, so our focus remains on minimax regret.
The following theorem establishes the regret upper bound.

\begin{theorem}
  \label{thm: upper bound for TL with function approximation}
  Consider a contextual bandit problem with $|\mathcal{A}|<\infty$ and $|\mathcal{F}|<\infty$ under \cref{asp: realizability}.
  With probability at least $1-\delta$, the expected regret of \cref{alg: TL-function approximation} is upper bounded by
  $$
  \mean{\mathrm{Reg}(T)} \le \mathcal{O} \left(\sqrt{ \mathbb{E}_W[ \mathcal{A}^*(W) ] T  \log (\delta^{-1}|\mathcal{F}^*|\log T  ) } \right).
  $$
\end{theorem}

\begin{remark}[Comparison with the literature]
  While drawing inspiration from the IGW method of \citet{instanceCB_RL}, our algorithm departs in three fundamental respects.
  First, we incorporate causal bounds into the exploration process, effectively restricting the hypothesis class to $\mathcal{F}^*$ and reducing the average action-set size to $\mathbb{E}_W[|\mathcal{A}^*(W)|]$.
  This refinement lowers the worst-case regret dependence from $|\mathcal{F}|$ and $|\mathcal{A}|$ to $|\mathcal{F}^*|$ and $\mathbb{E}_W[|\mathcal{A}^*(W)|]$.
  Second, we replace the epoch-based parameter $\gamma_m$ with a round-dependent learning rate $\gamma_t$ that adapts to the causal constraints of each context, enabling dynamic action pruning.
  Third, by leveraging causal bounds instead of relying solely on data-driven confidence sets, we tighten the regret guarantee to
  $\mathcal{O}\bigl(\sqrt{T\,\log(\delta^{-1}\log T)}\bigr)$,
  removing the $\log T$ factor in their bound of $\mathcal{O}(\sqrt{T\,\log(\delta^{-1}T^2)}\,\log T)$.

  Previous efforts on transfer learning in general contextual bandits leverage instrumental variables to obtain causal bound \citep{boundingCE_continuous_IV}.
  While powerful, this approach leads to regret that scales as $\mathcal{O}(\sqrt{|\Pi|})$, rather than the more desirable $\mathcal{O}(\sqrt{\log|\Pi|})$, and hinges on the often-intractable task of finding valid instruments. Moreover, by treating each basis policy in $\Pi$ as an independent ``arm,'' it ignores the fact that similar policies share considerable overlap---pulling one policy yields information about its neighbors.
  This lack of structure forces the $\sqrt{|\Pi|}$ dependence, whereas by exploiting causal bounds to prune both the function class and action sets, our method achieves the improved $\mathcal{O}(\sqrt{\log|\Pi|})$ rate.
\end{remark}

\begin{remark}[Infinite function classes]
  We note that \cref{alg: TL-function approximation} and \cref{thm: upper bound for TL with function approximation} naturally extend to infinite function classes $\mathcal{F}$.
  In such cases, the dependence on $\log |\mathcal{F}^*|$ in the algorithm's parameters and regret bounds is replaced by standard learning-theoretic complexity measures, such as metric entropy.
  Suppose $\mathcal{F}$ is equipped with a maximum norm $\norm{\cdot}_{\infty}$.
  Let $(\mathcal{F},\norm{\cdot}_{\infty})$ be a normed space. A set $\{f_1,\cdots,f_N\}$ is an $\epsilon$-covering of $\mathcal{F}$ if, for every
  $f \in \mathcal{F}$, there exists an $i$ such that $\norm{f-f_i}_{\infty}\leq \epsilon$.
  The covering number $N(\mathcal{F},\norm{\cdot}_{\infty},\epsilon )$ is defined as the minimal cardinality $N$ over all $\epsilon$-coverings of $\mathcal{F}$.
  Consider an $\epsilon$-covering $\mathcal{F}^*_{\epsilon}$ of $\mathcal{F}^*$ so that for any reward function $f^*$,
  there exists a function $f_{\epsilon}^*\in \mathcal{F}^*_{\epsilon}$ such that
  $
  \norm{ f_{\epsilon}^* - f^*  }_{\infty} \leq \epsilon.
  $
  Since $|\mathcal{F}^*_{\epsilon}| $ is finite, we can replace $\mathcal{F}^*$ with $\mathcal{F}^*_{\epsilon} $
  without altering any algorithmic procedures.
  Hence, the regret can be bounded by
  $
  \mathrm{Reg}(T) \leq 8 \sqrt{ \mathbb{E}_W[ \mathcal{A}(W) ] T  \log (2\delta^{-1}|\mathcal{F}^*_{\epsilon}|\log T  ) } + \epsilon T.
  $
  We then set $\epsilon=1/T$.
\end{remark}

We now demonstrate that the upper bound in \cref{thm: upper bound for TL with function approximation} matches the minimax lower bound for transfer learning.
Define the set of contextual bandit instances compliant with the causal bounds as
$$
\mathfrak{M} = \Bigl\{ \text{contextual bandit instances with } l(a,w) \leq f^*(a,w) \leq h(a,w),  \forall (a,w)\in\mathcal{A} \times\mathcal{W} \Bigr\}.
$$

\begin{theorem}
  \label{thm: lower bound for TL with function approximation}
  Consider a contextual bandit problem with $|\mathcal{A}|<\infty$ and $|\mathcal{F}|<\infty$ under \cref{asp: realizability}.
  Then for any algorithm $\mathsf{A}$ with access to the function space $\mathcal{F}$, we have
  $$
  \sup_{  \mathfrak{M}  }  \limsup_{T\to\infty}  \frac{  \mathrm{Reg}(T) }{ \sqrt{T}  } \geq  \sqrt{\mathbb{E}_W [ | {\mathcal{A}^*}(W)|  ] \log |\mathcal{F}^*| }.
  $$
\end{theorem}

\subsection{Computing \texorpdfstring{$\log |\mathcal{F}^*|$}{Metric Entropy} and \texorpdfstring{$\mathcal{A}^*(w)$}{Pruned Action Set}.}\label{sec: calculating F and A}

The two quantifies $\gamma_t$ and $\widehat{f}_m $ in \cref{alg: TL-function approximation} rely on the function class $\mathcal{F}^*$.
The estimated function $\widehat{f}_m$ can be obtained by solving a least-squares problem over the convex hull $\operatorname{conv}(\mathcal{F}^*)=\operatorname{conv}(\mathcal{F})\cap \bigl\{ f: l(a,w)\le f(a,w)\le h(a,w), \forall a,w\bigr\}$.
For $\gamma_t$, we need to compute $\log |\mathcal{F}^*|$ and $\mathcal{A}^*(w)$.
A straightforward searching approach of calculating $\mathcal{F}^*$ and $\mathcal{A}^*(w)$ has a time complexity of $\mathcal{O}(|\mathcal{F}|)$, inefficient for large or infinite $|\mathcal{F}|$.

\subsubsection{Computing \texorpdfstring{$\log |\mathcal{F}^*|$}{Metric Entropy}.}

Denote $N(\mathcal{F},\norm{\cdot}_{\infty},\epsilon )$ as the covering number of the function space $\mathcal{F}$ with respect to the maximum norm $\norm{\cdot}_{\infty}$.
Note that $N(\mathcal{F}^*,\norm{\cdot}_{\infty},\epsilon ) \leq N(\mathcal{F},\norm{\cdot}_{\infty},\epsilon ) $ since $\mathcal{F}^*\subset \mathcal{F}$.
The covering number explicitly shows how causal bounds help improve the algorithm's performance by reducing the search space.
For linear function spaces $\mathcal{F}=\{ \bm{\theta}^\top \bm{\phi}(a,w) \mid  \norm{\bm{\theta}}_{\infty}\leq 1  \}$ with known features $\bm{\phi}(a,w)\in \mathbb{R}^d$,
imposing causal constraints $l(i,w)\le \bm{\theta}^\top\bm{\phi}(i,w)\le h(i,w)$ reduce the searching space in $\bm{\theta}$-space.
Obtaining a precise covering number for such function spaces is generally intractable.
However, the volume analysis in \cite{High_dim_probability} shows that the covering number of the function space $\mathcal{F}^*$ is bounded by
$
(3\operatorname{diam}(\mathcal{F}) / \epsilon)^d.
$
For $\operatorname{diam}(\mathcal{F}^*)$, we may use the maximum of the bounds $h(a,w)-l(a,w)$ over all $(a,w)\in\mathcal{A}\times\mathcal{W}$.
A more refined method is to solve the following optimization problem
\begin{equation*}
  \max_{  \norm{\bm{\theta}}_{\infty} \leq 1 }  |\bm{\theta}^\top \bm{\phi}(a,w)| , \quad s.t. \ \bm{\theta}^\top \bm{\phi}(a,w)\leq h(a,w) - l(a,w), \  \forall (a,w) \in \mathcal{A}\times\mathcal{W}
\end{equation*}
to obtain the upper bound of the diameter of $\mathcal{F}^*$.

\subsubsection{Computing \texorpdfstring{$\mathcal{A}^*(w)$}{Pruned Action Set}.}
We now derive a tractable characterization of $\mathcal{A}^*(w)$.
\begin{proposition}
  \label{prop: condition to compute Aw}
  Suppose that $\mathcal{F}$ is a compact set in the topology of pointwise convergence.
  Then $a\in \mathcal{A}^*(w)$ if and only if $\max_{f \in \mathcal{F}^*} \left( f(a, w) - \max_{i \neq a} f(i, w) \right) \geq 0.$
\end{proposition}
\cref{prop: condition to compute Aw} allows us to test each candidate $a$ by solving an optimization problem.
Importantly, the objective function is concave in $f$, and the feasible set $\mathcal{F}^* \subseteq \mathcal{F}$ is convex whenever $\mathcal{F}$ is.
Hence each test reduces to a convex program. Equivalently, one can form the Lagrangian dual
\[
  \inf_{\substack{\lambda_i \geq 0 \\ \sum_{i \neq a} \lambda_i = 1}} \ \sup_{f \in \mathcal{F}^*} \biggl[ f(a, w) - \sum_{i \neq a} \lambda_i f(i, w) \biggr].
\]
By weak duality, the dual objective bounds the primal from above, so checking whether this dual optimum is non-negative yields a conservative superset of $\mathcal{A}^*(w)$.
Thus, checking the non-negativity of the dual objective yields a conservative superset of $\mathcal{A}^*(w)$.
When strong duality holds, e.g. under Slater's condition, the dual and primal optima coincide, and we recover $\mathcal{A}^*(w)$.
We give two examples of function spaces $\mathcal{F}$ for which this dual problem can be solved efficiently.
\begin{itemize}
  \item \textbf{Unrestricted reward class:}  $\mathcal{F}$ is all real-valued functions $f:\mathcal{A}\times\mathcal{W}\to\mathbb{R}$.
    Then $\mathcal{F}^* = \bigl\{ f: l(i,w)\le f(i,w)\le h(i,w), \forall i,w\bigr\}$     is a convex box in function-space, and Slater's condition holds.  Because the box constraints decouple across arms, the inner maximization becomes
    \[
      \max_{f\in\mathcal{F}^*}\biggl[f(a,w)- \sum_{i\neq a}\lambda_i f(i,w)\biggr]
      = h(a,w) - \sum_{i\neq a}\lambda_i l(i,w),
    \]
    and the dual problem simplifies to
    \[
      h(a,w) - \max_{\substack{\lambda_i \geq 0 \\ \sum_{i \neq a} \lambda_i = 1}} \sum_{i \neq a} \lambda_i l(i,w) = h(a,w) - \max_{i \neq a} l(i,w),
    \]
    where the equality holds because the maximum over the simplex is attained at a vertex.
    Thus $a\in\mathcal{A}^*(w)$ exactly when $h(a,w)\ge\max_{i}l(i,w)$, i.e., a contextual version of the pruning rule \eqref{eq:prune_mab_action}.

  \item \textbf{Linear predictors:} $\mathcal{F}=\{ \bm{\theta}^\top \bm{\phi}(a,w) \mid  \norm{\bm{\theta}}_{\infty}\leq 1  \}$ for known features $\bm{\phi}(a,w)$.  Imposing causal constraints $l(i,w)\le \bm{\theta}^\top\bm{\phi}(i,w)\le h(i,w)$ yields a convex polytope in $\bm{\theta}$-space.
    Then action $a\in\mathcal{A}^*(w)$, if the optimal value $s^* \geq 0$ for the following linear program
    \[\max_{\bm{\theta},s} \; s \quad \text{s.t.} \quad
      \bm{\theta}^\top \phi(a,w) \ge \bm{\theta}^{\top} \phi(i,w)^\top  + s,\;\forall i\neq a, \quad l(i,w) \le \bm{\theta}^{\top}\phi(i,w) \le h(i,w),\;\forall i.
    \]
\end{itemize}

\section{Numerical Experiments}
\label{sec: numerical}
\subsection{Causal Bounds}
\subsubsection{Tighter Causal Bounds.}
We compare our causal-bound sampler (\cref{alg:sampling with optimization oracle}) against the nonlinear-programming approach of \citet{CEbound}, where all variables are \textit{binary}.
Using randomly generated true marginals $\rho(a, y, w)$ and $\rho(u)$ (\cref{tab: observational distribution} in \cref{sec in appendix: Benchmark Sampling Method}),
we assume these distributions are known exactly, as in \citet{CEbound}.
The key distinction lies in how the feasible region is defined. \citet{CEbound} enforce only the Fr\'{e}chet-Hoeffding bounds:
$
\max\left\{0, \rho(a_i,y_j,w_k) + \rho(u_l) - 1\right\} \leq p_{ijkl} \leq \min\left\{ \rho(a_i,y_j,w_k), \rho(u_l) \right\}
$
for fixed $(i,j)$ and each tuple $(k,l)$ (where $\rho(a_i,y_j,w_k)$, $\rho(u_l)$, and $p_{ijkl}$ are defined in \cref{subsec in appendix: discrete sample space}).
Critically, they omit the global marginal-consistency constraints ensuring $(p_{ijkl})$ forms a valid joint distribution.
By contrast, our formulation enforces all marginalization constraints simultaneously;
thus every point in our feasible set corresponds to a proper SCM.
Consequently, \citet{CEbound}'s feasible set may include invalid ``models" violating joint consistency,
leading to strictly looser causal bounds (see \cref{app: benchmark causal bound computation}).
To solve the two distinct nonlinear optimization problems, we combine \cref{alg:sampling with optimization oracle} with SciPy's \texttt{optimize.minimize} solver,
seeding it with multiple starting points from Algorithm \ref{alg: hit-run for density} with initialization $p_{ijkl} = \rho(a_i,y_j,w_k)\rho(u_l)$; see \cref{subsec in appendix: discrete sample space} for details.

Our approach produces noticeably tighter intervals than the nonlinear program of \citet{CEbound}.
In particular, whereas \citet{CEbound} obtain $\mathbb{E}[Y\mid\mathrm{do}(A=0)]\in[0.283,0.505]$ and $
\mathbb{E}[Y\mid\mathrm{do}(A=1)]\in[0.240,0.807],$
we shrinks these to $\mathbb{E}[Y\mid\mathrm{do}(A=0)]\in[0.352,0.471]$ and $\mathbb{E}[Y\mid\mathrm{do}(A=1)]\in[0.265,0.768],$
demonstrating that our joint-distribution constraints yield strictly tighter causal bounds.

\subsubsection{Estimation Error.}
To assess the impact of estimation error, we inject a uniform perturbation $\epsilon \in [0,0.1]$ into all estimated probabilities and then solve the resulting causal-bound programs using SciPy's \texttt{optimize.minimize}.
Starting from $\epsilon = 0.1$ and gradually reducing it toward zero, we record the estimated lower and upper bounds at each step. As Figure \ref{fig: estimation for causal bounds} shows, both estimates converge to the true causal bounds as $\epsilon$ approaches zero, confirming \cref{thm: convergence of empirical bounds}.

\begin{figure}[hbtp]
  \centering
  \begin{tikzpicture}
    \begin{axis}[
        height=0.25\textwidth,
        width=0.5\textwidth,
        xlabel=Error $\epsilon$ (log scale),
        ylabel=Causal Bounds,
        xmin=0.005, xmax=0.2,
        ymin=0.0, ymax=0.9,
        xmode=log,
        log basis x=10,
        xtick={0.01,0.02,0.05,0.1,0.2},
        xticklabels={0.01,0.02,0.05,0.1,0.2},
        ytick={0,0.2,0.4,0.6,0.8},
        extra x ticks={0.005,0.03,0.04,0.06,0.08,0.12},
        extra x tick labels={},
        extra x tick style={grid style={gray!10, very thin}},
        axis lines=left,
        line width=0.7pt,
        tick style={line width=0.7pt},
        tick label style={font=\small},
        label style={font=\small},
        legend style={
          at={(1.03,0.66)},
          anchor=north east,
          cells={anchor=west},
          font=\footnotesize,
          fill=white,
          fill opacity=0.7,
          text opacity=1,
          draw=none
        },
        title style={font=\small, yshift=-1ex}
      ]

      \addplot [
        color=red!85!black,
        mark=*,
        mark size=2pt,
        line width=1.4pt,
        densely dashed
      ] table [
        col sep=comma,
        x=epsilon,
      y=ha,] {data/CEUpperBounds.csv};

      \addplot [
        color=blue!85!black,
        mark=square*,
        mark size=2pt,
        line width=1.4pt,
        solid
      ] table [x index=1, y index=2, col sep=comma] {data/CEUpperBounds.csv};

      \legend{
        Upper causal bound,
        Lower causal bound
      }

    \end{axis}
  \end{tikzpicture}

  \caption{Estimation for Causal Bounds $\mathbb{E}[Y|\mathrm{do}(A=0)]$ with Estimation Errors of Different Levels}

  \label{fig: estimation for causal bounds}
\end{figure}

\subsubsection{Computation Complexity.}

We compared the efficiency of Algorithm~\ref{alg: hit-run for density} and a naive Algorithm~\ref{alg: MC-causal model} (Appendix~\ref{sec in appendix: Benchmark Sampling Method}) by generating $10^3$ samples with each.
Algorithm \ref{alg: MC-causal model} constructs feasible points by cycling through each coordinate, and progressively adding sampled coordinates to feasibility constraints and solving linear programs to ensure validity.
This approach, while simple to implement, suffers from computational burdens in high dimensions.
We set $n \triangleq n_{\mathcal{A}}=n_{\mathcal{Y}}=n_{\mathcal{W}}=n_{\mathcal{U}}$ (see \cref{assu: discrete}), so that the problem dimension grows as $\mathcal{O}(n^4)$.
Numbers of valid samples generated per second are summarized in \cref{tb: samples per second}.
\begin{table}[ht]
  \small
  \centering
  \begin{tabular}{c|ccccccccc}
    \toprule
    $n$ & 2    & 3    & 4    & 5    & 6       & 7       & 8       & 9       & 10      \\ \midrule
    \cref{alg: MC-causal model}          & 17.5   & 1.2    & 0.1    & $\approx 0$ & $\approx 0$ & $\approx 0$ & $\approx 0$ & $\approx 0$ & $\approx 0$ \\
    \cref{alg: hit-run for density} & 18271.77 & 9330.17 & 3368.91 & 1576.36 & 694.61 & 326.70 & 172.70 & 91.08 & 48.79  \\  \bottomrule
  \end{tabular}
  \caption{Comparison of Samples per Second for \cref{alg: MC-causal model} and \cref{alg: hit-run for density}}
  \label{tb: samples per second}
\end{table}

\subsection{Transfer Learning in Multi-Armed Bandits}

We evaluate \cref{alg: TL-MAB} and \cref{alg: MAB with noisy causal bounds} on a synthetic $6$-arm stochastic bandit.
In all experiments we fix the confidence level at $\delta = 0.1$ and corrupt each reward with zero-mean Gaussian noise of standard deviation $0.1$.
Table \ref{tab:arm configuration in MAB} lists the six arms used in our experiments, along with their true means and causal bounds.
For simplicity, we use identical values for the exact and noisy bounds such that $\mu_a \in [l(a), h(a)]$ for every arm.
The noisy causal bound comes with $\epsilon_a(\delta) = 0.1$ estimation error.

\begin{table}[ht]
  \small
  \centering
  \begin{tabular}{ccccl}
    \toprule
    Arm $a$ & $\mu_a$ & $[l(a),h(a)]$ & $[\widehat{l}(a),\widehat{h}(a)]$ & Note  \\ \midrule
    0 & 0.3 & $[0.25,0.50]$ & $[0.25,0.50]$ & Worst arm \\
    1 & 0.4 & $[0.35,0.60]$ & $[0.35,0.60]$ & Clearly suboptimal $a \notin \mathcal{A}^*$\\
    2 & 0.5 & $[0.45,0.70]$ & $[0.45,0.70]$ & Clearly suboptimal, $a \notin \mathcal{A}^*$\\
    3 & 0.7 & $[0.55,0.78]$ & $[0.55,0.78]$ & Ambiguous, $h(a)<\mu^*$ \\
    4 & 0.7 & $[0.65,0.85]$ & $[0.65,0.85]$ & Ambiguous, $h(a)>\mu^*$ \\
    5 & 0.8 & $[0.75,0.90]$ & $[0.75,0.90]$ & Optimal arm \\
    \bottomrule
  \end{tabular}
  \caption{True Mean Rewards, Causal Bounds, and Designations for Arm Elimination}
  \label{tab:arm configuration in MAB}
\end{table}

These parameters are chosen to highlight how causal bounds shape each algorithm's exploration.
Arms 3 and 4, both with true mean 0.7, are crafted to isolate the impact of the causal upper bound: Arm 3 has $h_3 < \mu_5$ so its upper bound lies below the optimum, and the offline data is sufficient to eliminate it entirely; on the contrary, Arm 4 has $h_4 > \mu_5$, forcing the algorithm to explore it online.
We run both algorithms for $T=10^4$ rounds, repeating each configuration over 50 independent trials.

The results in Tables \ref{tab: Summary statistics of suboptimal arm pulls and final regret} clearly demonstrate the power of causal bounds in pruning suboptimal arms.
In particular, \cref{alg: TL-MAB} never pulls Arms 0-3, exactly as predicted: Arms 0-2 have upper bounds below the pruning threshold, and Arm 3's bound falls just under the optimal mean.
In contrast, Plain UCB and \cref{alg: MAB with noisy causal bounds} with the noisy bounds continue to sample these same arms, with \cref{alg: MAB with noisy causal bounds} substantially reducing unnecessary exploration.
Arm 4 with causal upper bounds exceeding $\mu^*$ cannot be pruned beforehand, and indeed all algorithms sample Arm 4 at similar rates.
This confirms that causal bounds help most when they definitively rule out an arm; when the bounds are inconclusive, online exploration is essential.

\begin{table}[hbtp]
  \centering
  \small
  \setlength{\tabcolsep}{4pt}
  \begin{tabular}{@{}l *{5}{c} c@{}}
    \toprule
    &
    \multicolumn{5}{c}{Suboptimal Arm Pulls (Mean$\pm$SD)} &
    Final Regret \\
    \cmidrule(lr){2-6}
    & Arm 0 & Arm 1 & Arm 2 & Arm 3 & Arm 4 & (Mean$\pm$SD) \\
    \midrule
    Plain UCB & 21.84 $\pm$ 1.92 & 32.50 $\pm$ 2.76 & 54.70 $\pm$ 5.08 & 368.48 $\pm$ 29.54 & 372.64 $\pm$ 29.52 & 114.44 $\pm$ 4.54 \\
    \cref{alg: MAB with noisy causal bounds} & 0.00 $\pm$ 0.00 & 8.56 $\pm$ 0.95 & 54.36 $\pm$ 4.09 & 366.98 $\pm$ 28.05 & 365.80 $\pm$ 30.24 & 93.01 $\pm$ 4.87 \\
    \cref{alg: TL-MAB} & 0.00 $\pm$ 0.00 & 0.00 $\pm$ 0.00 & 0.00 $\pm$ 0.00 & 0.00 $\pm$ 0.00 & 301.10 $\pm$ 27.62 & 30.11 $\pm$ 2.76 \\
    \midrule
    $H_a$ & 9.0  & 4.0 & 1.0 & 0.04 & 0.0 &  \\
    \bottomrule
  \end{tabular}
  \caption{Summary Statistics of Suboptimal Arm Pulls and Final Regret}
  \label{tab: Summary statistics of suboptimal arm pulls and final regret}
\end{table}

Interestingly, when the causal bounds for Arm 3 have a relatively large $\epsilon_3(\delta) = 0.1$, \cref{alg: MAB with noisy causal bounds} selects it almost as often as plain UCB,
indicating that a negligible $H_3 = 0.04$ is not enough to influence exploration.
To assess the impact of estimation error, we vary $\epsilon_3(\delta)$ from $0.03$ down to $0.015$ (so $H_3$ varies from $0.44$ to $1.78$),
while keeping all other bounds with $0.1$ errors.
We run the algorithm for $T = 10^4$ rounds.

\begin{figure}[htbp]
  \centering
  \begin{subfigure}{0.48\textwidth}
    \centering
    \begin{tikzpicture}
      \definecolor{barcolor}{RGB}{31,119,180}
      \definecolor{linecolor}{RGB}{214,39,40}
      \definecolor{errorcolor}{RGB}{100,100,100}

      \begin{axis}[
          width=0.92\textwidth,
          height=0.6\textwidth,
          axis y line*=left,
          xlabel={Estimation Error ($\epsilon_3(\delta)$)},
          ylabel={Selection Count},
          ymin=0, ymax=410,
          ytick={0,100,200,300,400},
          yticklabel style={font=\footnotesize},
          ylabel style={font=\small},
          xmin=0.013, xmax=0.032,
          xtick={0.015,0.018,0.02,0.022,0.025,0.03},
          xticklabel style={font=\footnotesize},
          xlabel style={font=\small},
          legend style={at={(0.5,1.25)}, anchor=north, font=\scriptsize, legend columns=-1},
          bar width=0.0016,
          error bars/y dir=both,
          error bars/y explicit,
          error bars/error bar style={color=errorcolor, thick}
        ]
        \addplot+[
          ybar,
          fill=barcolor!40,
          draw=barcolor,
        ]
        coordinates {
          (0.015, 0.0)    +- (0,0.0)
          (0.018, 20.7)   +- (0,24.611)
          (0.02, 130.16)  +- (0,21.692)
          (0.022, 213.98) +- (0,25.781)
          (0.025, 299.58) +- (0,27.370)
          (0.03, 358.62)  +- (0,25.509)
        };
        \addlegendentry{Selection Count}

      \end{axis}

      \begin{axis}[
          width=0.92\textwidth,
          height=0.6\textwidth,
          axis y line*=right,
          axis x line=none,
          ytick={0,25,50,75,100,125},
          yticklabel style={font=\footnotesize},
          ymin=0, ymax=125,
          xmin=0.013, xmax=0.032,
          ylabel style={font=\small},
          legend style={at={(0.5,1.25)}, anchor=north, font=\small, legend columns=-1},
        ]
        \addplot+[
          color=linecolor,
          mark=*,
          mark options={fill=white, scale=1.2},
          line width=1.2pt,
          error bars/.cd,
          y dir=both,
          y explicit,
          error bar style={color=errorcolor, thick}
        ]
        coordinates {
          (0.015, 57.248) +- (0,3.481)
          (0.018, 59.708) +- (0,4.713)
          (0.02, 70.132)  +- (0,4.075)
          (0.022, 78.686) +- (0,5.294)
          (0.025, 87.146) +- (0,4.653)
          (0.03, 92.872)  +- (0,3.428)
        };

      \end{axis}
    \end{tikzpicture}
    \caption{Impact on Arm 3}
    \label{fig:impact_arm3}
  \end{subfigure}
  \hfill
  \begin{subfigure}{0.48\textwidth}
    \centering
    \begin{tikzpicture}
      \definecolor{barcolor}{RGB}{31,119,180}
      \definecolor{linecolor}{RGB}{214,39,40}
      \definecolor{errorcolor}{RGB}{100,100,100}

      \begin{axis}[
          width=0.92\textwidth,
          height=0.6\textwidth,
          axis y line*=left,
          xlabel={Estimation Error ($\epsilon_4(\delta)$)},
          ymin=0, ymax=410,
          yticklabel style={font=\footnotesize},
          ylabel style={font=\small},
          xmin=0.013, xmax=0.032,
          xtick={0.015,0.018,0.02,0.022,0.025,0.03},
          xticklabel style={font=\footnotesize},
          xlabel style={font=\small},
          legend style={at={(0.5,1.25)}, anchor=north, font=\small, legend columns=-1},
          bar width=0.0016,
          error bars/y dir=both,
          error bars/y explicit,
          error bars/error bar style={color=errorcolor, thick}
        ]
        \addplot+[
          ybar,
          fill=barcolor!40,
          draw=barcolor,
        ]
        coordinates {
          (0.015, 374.66) +- (0,30.335)
          (0.018, 368.7)  +- (0,29.131)
          (0.02, 366.72)  +- (0,32.515)
          (0.022, 368.38) +- (0,33.180)
          (0.025, 371.6)  +- (0,26.967)
          (0.03, 373.46)  +- (0,31.629)
        };
      \end{axis}

      \begin{axis}[
          width=0.92\textwidth,
          height=0.6\textwidth,
          axis y line*=right,
          axis x line=none,
          ylabel={Cumulative Regret},
          yticklabel style={font=\footnotesize},
          xmin=0.013, xmax=0.032,
          ymin=0, ymax=125,
          ytick={0,25,50,75,100,125},
          ylabel style={font=\small},
          legend style={at={(0.5,1.25)}, anchor=north, font=\scriptsize, legend columns=-1},
        ]
        \addplot+[
          color=linecolor,
          mark=*,
          mark options={fill=white, scale=1.2},
          line width=1.2pt,
          error bars/.cd,
          y dir=both,
          y explicit,
          error bar style={color=errorcolor, thick}
        ]
        coordinates {
          (0.015, 93.846) +- (0,4.356)
          (0.018, 93.984) +- (0,4.780)
          (0.02, 93.346)  +- (0,4.335)
          (0.022, 92.244) +- (0,4.385)
          (0.025, 94.218) +- (0,4.618)
          (0.03, 94.824)  +- (0,4.726)
        };
        \addlegendentry{Cumulative Regret}
      \end{axis}
    \end{tikzpicture}
    \caption{Impact on Arm 4}
    \label{fig:impact_arm4}
  \end{subfigure}

  \caption{Impact of Estimation Error ($\epsilon$) on Transfer Learning Algorithm Performance\\
    \textit{Notes.} Bars represent the average selection count (left axis), while lines show
  the cumulative regret (right axis). Error bars indicate $\pm1$ standard deviation.}
  \label{fig: Impact of exploration rate}
\end{figure}

The results are plotted in \cref{fig:impact_arm3}.
For the same study of Arm 4, see \cref{fig:impact_arm4}.
Though Arms 3 and 4 share the same true mean, their causal bounds drive very different outcomes.
For Arm 3 with $h_3 < \mu^*$, increasing its estimation accuracy sharply reduces how often it's chosen, which in turn drives down the final regret.
By contrast, Arm 4's bound $h_4 > \mu^*$ offers no discriminatory power, so neither its selection frequency nor the accumulated regret changes as $\epsilon_4(\delta)$ decreases.
In other words, when causal bounds cannot rule out an arm in nature, its estimation accuracy does not help.
These empirical patterns validates Theorem \ref{thm: regret upper bound of MAB with noisy causal bounds},
which predicts that only the term $(\mu^* - h(a))_+$ governs the value of estimation accuracy.

\subsection{Transfer Learning in Contextual Bandits}

We evaluate \cref{alg: TL-function approximation} on a synthetic linear contextual bandit with five arms $\mathcal{A} = \{a_1, \cdots, a_5\}$ and eleven discrete contexts $\mathcal{W} = \{w_1, \dots, w_{11}\}$.
The experiment proceeds for $T = 10^4$.
At each round $t$, the context $w_t$ is drawn uniformly from $\mathcal{W}$, and feature vectors $\bm{\phi}(a,w_t) \in \mathbb{R}^2$ for each arm $a$ are taken from \cref{tb:feature_vectors}.
\cref{alg: TL-function approximation} is implemented with confidence parameter $\delta = 0.1$ and tuning parameter $\eta = 1.0$.
Prior knowledge are encoded in causal bounds $l(a,w)$ and $h(a,w)$ in \Cref{tb:causal bounds in numerical experiment}, which we use to compute the exact candidate set $\mathcal{A}^*(w)$ by solving the LPs described in Section \ref{sec: calculating F and A}.
To highlight the impact of the size of the action set, we also consider two lightweight supersets $\mathcal{A}$ and $\mathcal{A}_0(w) \triangleq \{a \mid h(a,w) \geq \max_i l(i,w)\}$ of the optimal action sets, satisfying $\mathcal{A}^*(w) \subset \mathcal{A}_0(w) \subset \mathcal{A}$; see Table \ref{tb: Comparison of Candidate Sets}.

Note that if we disable both function-class pruning ($\mathcal{F}^* = \mathcal{F}$) and action-set elimination ($\mathcal{A}^*(w) = \mathcal{A}$), our algorithm reduces to FALCON \citep{fasterCB}, an instantiation of inverse-gap weighting without leveraging any causal knowledge.
To assess the benefit of each component, we compare variants that selectively use causal bounds on $\mathcal{F}$ and/or $\mathcal{A}$.
Since $\mathcal{F}$ is infinite, we replace $|\mathcal{F}|$ and $|\mathcal{F}^*|$ by their covering-number proxies $(3T \operatorname{diam}(\mathcal{F}))^d$ and $(3T \operatorname{diam}(\mathcal{F}^*))^d$.
Since both $\mathcal{F}$ and $\mathcal{F}^*$ are simplex, we can compute their diameters as $\operatorname{diam}(\mathcal{F}) = 2\sqrt{2}$ and $\operatorname{diam}(\mathcal{F}^*) = \sqrt{2}/20$.

We sample the 2-dimensional ground truth parameter $\bm{\theta^*}$ randomly from the feasible region compatible with the causal bounds.
Then, for each chosen arm-context pair $(a,w)$, the reward is drawn as $Y = \bm{\phi}(a,w)^\top \bm{\theta}^* + \mathcal{N}(0, 0.1)$.
We repeat this procedure for 50 independent trials.
The full suite of results appears in \Cref{tb:regret_stats}.

\begin{table}[hbtp]
  \centering
  \small
  \begin{tabular}{lcrrrrr}
    \toprule
    Algorithm & Action set & Mean & Std & Median & Min & Max \\
    \midrule
    FALCON    & $\mathcal{A}$ & 1240.69 & 29.33 & 1236.02 & 1190.62 & 1316.09 \\
    \cref{alg: TL-function approximation}  &  $\mathcal{A}$  & 1100.09 & 22.13 & 1102.42 & 1046.82 & 1157.81 \\
    \cref{alg: TL-function approximation} &  $\mathcal{A}_0(w)$ & 515.04 & 14.32 & 511.95 & 487.77 & 551.68 \\
    \cref{alg: TL-function approximation}  & $\mathcal{A}^*(w)$ & 111.47 & 3.45 & 111.46 & 102.78 & 118.96 \\
    \bottomrule
  \end{tabular}
  \caption{Summary Statistics of Cumulative Regrets for Different Parameter Configurations}
  \label{tb:regret_stats}
\end{table}
We observe that the performance boost mirrors the shrinkage in candidate-action sets: on average $|\mathcal{A}^*(w)| \approx 2.0$ and $|\mathcal{A}_0(w)| \approx 2.9$, while $|\mathcal{A}| = 5$.
Notably, smaller action sets also yield much smaller standard deviations of cumulative regret, showing robustness of a parsimonious model derived from causal-bound elimination.
When comparing FALCON with \cref{alg: TL-function approximation} with action set $\mathcal{A}$,
we observed that pruning the function class alone yields marginal regret improvements of $11\%$, while most of the gains come from eliminating suboptimal actions up front.

\section{Conclusions}

In this paper, we tackle a practical transfer learning scenario in which unobserved confounders, or distribution shifts render causal effects non-identifiable.
Rather than debiasing point estimates, we derive tight causal-effect bounds by solving non-convex programs over the space of joint density functions constrained by prior offline knowledge.

To compute these bounds efficiently, we develop a hit-and-run sampler that asymptotically draws uniform samples from the polytope of compatible structural causal models.
Each sample is then refined via a local optimization oracle, yielding causal-bound estimates that converge almost surely.
By relaxing our constraints to account for estimation error---and leveraging the Lipschitz continuity of causal effects---we prove that our relaxed feasible sets converge in the Hausdorff metric and thus their extrema converge to the true bounds.

Embedding these bounds into online learning yields immediate gains: in multi-armed bandits, suboptimal arms are safely eliminated; in contextual bandits with continuous contexts, we prune both policy classes and action sets, reducing regret dependence from $\sqrt{|\Pi|}$ to $\sqrt{\log|\Pi|}$ without assuming linear rewards.
Under limited or noisy data, our sensitivity model extension preserves guarantees and ensures a smooth transition back to the exact-bounds regime as estimation improves.

There are several future research directions worth exploring.
Beyond linear moment-type constraints, a natural next step is to allow our sampler to enforce nonlinear ambiguity sets, such as those arising from quadratic or kernel-based constraints.
Concretely, one could replace each linear functional with a nonlinear requirement such as $\int \alpha(\rho(\bm{x})) \diff \nu(\bm{x}) = \beta$, where $\alpha$ is a nonlinear function of the joint density $\rho$,  and then extend hit-and-run to explore these curved regions.
On the bandit side, adapting our IGW strategy to continuous action spaces will require new tools for controlling exploration complexity, such as metric entropy or Rademacher complexity of the action-value function class.
One promising avenue is to combine IGW with function approximation over action space to achieve regret guarantees in large or uncountable action settings.

\bibliography{ref}
\bibliographystyle{plainnat}

\clearpage

\appendix
\part*{Appendix}
\addcontentsline{toc}{part}{Appendix}
\etocsetnexttocdepth{subsubsection}

\localtableofcontents

\clearpage

\section{Proofs for Section \ref{sec: sampling general}}

\subsection{Proof of Proposition \ref{prop: valid samples in P}}

\begin{proof}
  We prove this by mathematical induction on $t$.
  The base case is guaranteed by algorithm input specification, $\rho_0 \in \mathcal{P}$.
  Assuming $\rho_{t-1} \in \mathcal{P}$, we need to show $\rho_t = \rho_{t-1} + \lambda_t d_t \in \mathcal{P}$.

  \begin{enumerate}
    \item \textbf{Non-negativity:}
      To ensure $\rho_t(\bm{x}) \geq 0$ for all $\bm{x} \in \Omega$, we analyze the pointwise behavior of the update $\rho_t(\bm{x}) = \rho_{t-1}(\bm{x}) + \lambda_t d_t(\bm{x})$.
      We consider three cases based on the sign of the direction function $d_t$ at each point.
      \begin{itemize}
        \item Case 1 ($d_t(\bm{x}) < 0$):
          From the definition of $\lambda_t^{\max}$ and its component for $d_t(\bm{x}) < 0$:
          $$
          \lambda_t \leq \lambda_t^{\max} \leq -\dfrac{\rho_{t-1}(\bm{x})}{d_t(\bm{x})}
          $$
          Multiplying by $d_t(\bm{x}) < 0$ (reversing inequality):
          $$
          \lambda_t d_t(\bm{x}) \geq \rho_{t-1}(\bm{x})
          $$
          Thus $\rho_t(\bm{x}) = \rho_{t-1}(\bm{x}) + \lambda_t d_t(\bm{x}) \geq 0$.

        \item Case 2 ($d_t(\bm{x}) > 0$):
          From the definition of $\lambda_t^{\min}$ and its component for $d_t(\bm{x}) > 0$:
          $$
          \lambda_t \geq \lambda_t^{\min} \geq -\dfrac{\rho_{t-1}(\bm{x})}{d_t(\bm{x})}
          $$
          Multiplying by $d_t(\bm{x}) > 0$:
          $$
          \lambda_t d_t(\bm{x}) \geq -\rho_{t-1}(\bm{x})
          $$
          Thus $\rho_t(\bm{x}) = \rho_{t-1}(\bm{x}) + \lambda_t d_t(\bm{x}) \geq 0$.

        \item Case 3 ($d_t(\bm{x}) = 0$):
          $\rho_t(\bm{x}) = \rho_{t-1}(\bm{x}) \geq 0$.
      \end{itemize}
      Thus $\rho_t(\bm{x}) \geq 0$ for all $\bm{x} \in \Omega$.

    \item \textbf{Equality constraints:}
      Since $d_t = \mathscr{P}(G_t)$ and $\mathscr{P} = \mathscr{I} - \mathscr{A}^\top (\mathscr{A} \mathscr{A}^\top)^{-1} \mathscr{A}$,
      we have $\mathscr{A}(d_t) = 0$. Therefore:
      $$
      \mathscr{A}(\rho_t) = \mathscr{A}(\rho_{t-1} + \lambda_t d_t) = \mathscr{A}(\rho_{t-1}) + \lambda_t \mathscr{A}(d_t) = \bm{\beta} + \lambda_t \cdot \mathbf{0} = \bm{\beta},
      $$
      where $\bm{\beta} = (\beta_1,\dots,\beta_m)^\top$.

    \item \textbf{Inequality constraints:}
      For each $j \in [m'] \backslash [m]$, we need to show $\int \alpha_i(\bm{x}) \rho_t(\bm{x}) \mathrm{d}\nu(\bm{x}) \leq \beta_i$. Let $v_{t-1,i} = \int \alpha_i(\bm{x}) \rho_{t-1}(\bm{x}) \mathrm{d}\nu(\bm{x}) \leq \beta_i$ (by induction hypothesis) and
      $c_{t,i} = \int \alpha_i(\bm{x}) d_t(\bm{x}) \mathrm{d}\nu(\bm{x})$. Then:
      $$
      \int \alpha_i(\bm{x}) \rho_t(\bm{x}) \mathrm{d}\nu(\bm{x}) = v_{t-1,i} + \lambda_t c_{t,i}
      $$
      We consider three cases based on $c_{t,i}$:
      \begin{itemize}
        \item Case 1 ($c_{t,i} > 0$):
          From the definition of $\lambda_t^{\max}$ and its inequality component:
          $$
          \lambda_t \leq \lambda_t^{\max} \leq \dfrac{\beta_i - v_{t-1,i}}{c_{t,i}}
          $$
          Multiplying by $c_{t,i} > 0$:
          $$
          \lambda_t c_{t,i} \leq \beta_i - v_{t-1,i}
          $$
          Thus $v_{t-1,i} + \lambda_t c_{t,i} \leq \beta_i$.

        \item Case 2 ($c_{t,i} < 0$):
          From the definition of $\lambda_t^{\min}$ and its inequality component:
          $$
          \lambda_t \geq \lambda_t^{\min} \geq \dfrac{\beta_i - v_{t-1,i}}{c_{t,i}}
          $$
          Multiplying by $c_{t,i} < 0$ (reversing inequality):
          $$
          \lambda_t c_{t,i} \leq \beta_i - v_{t-1,i}
          $$
          Thus $v_{t-1,i} + \lambda_t c_{t,i} \leq \beta_i$.

        \item Case 3 ($c_{t,i} = 0$):
          $\int \alpha_i(\bm{x}) \rho_t(\bm{x}) \mathrm{d}\nu(\bm{x}) = v_{t-1,i} \leq \beta_i$ (by induction hypothesis).
      \end{itemize}
  \end{enumerate}

  By induction, $\rho_t $ satisfies all constraints in $ \mathcal{P}$ for all $t \ge 0$.
\end{proof}

\subsection{Proof of Proposition \ref{prop: hit-and-run}}

\begin{proposition}
  \label{prop: compactness of feasible set}
  Under \cref{asp: regularity} and \cref{asp: finite-dimension}, the set $\mathcal{P}$ defined in \eqref{eq: feasible set} is compact in $L^2(\Omega,\nu)$.
\end{proposition}

\begin{proof}
  Since $\mathcal{P}$ is finite-dimensional,
  there exists a finite-dimensional subspace $V \subset L^2(\Omega,\nu)$ such that $\mathcal{P} \subseteq V$.
  By assumption, $\mathcal{P}$ is bounded in the $L^2(\Omega,\nu)$ norm.
  By Heine-Borel theorem, it suffices to show that $\mathcal{P}$ is closed in $L^2(\Omega,\nu)$.

  To prove that $\mathcal{P}$ is closed, consider a sequence $\{\rho_n\} \subset \mathcal{P}$ converging to some $\rho$ in $L^2(\Omega,\nu)$, i.e., $\|\rho_n - \rho\|_{L^2(\Omega,\nu)} \to 0$ as $n \to \infty$. We must show that $\rho \in \mathcal{P}$, meaning $\rho$ satisfies all constraints defining $\mathcal{P}$.

  Define the linear functionals $\phi_i(\sigma) = \int \alpha_i \sigma  \mathrm{d}\nu$ for $i = 1, \dots, m'$. Since $\alpha_i \in L^2(\Omega,\nu)$, each $\phi_i$ is continuous on $L^2(\Omega,\nu)$ by the Cauchy-Schwarz inequality:
  $$
  |\phi_i(\sigma)| \leq \|\alpha_i\|_{L^2(\Omega,\nu)} \|\sigma\|_{L^2(\Omega,\nu)} < \infty.
  $$
  Continuity implies that if $\rho_n \to \rho$ in $L^2(\Omega,\nu)$, then $\phi_i(\rho_n) \to \phi_i(\rho)$ for each $i$.

  \begin{itemize}
    \item \textbf{Equality constraints ($i = 1, \dots, m$):} For each $i$, $\phi_i(\rho_n) = \beta_i$ for all $n$. Since $\phi_i(\rho_n) \to \phi_i(\rho)$, we have $\phi_i(\rho) = \beta_i$.
    \item \textbf{Inequality constraints ($i = m+1, \dots, m'$):} For each $i$, $\phi_i(\rho_n) \leq \beta_i$ for all $n$. Since $\phi_i(\rho_n) \to \phi_i(\rho)$, we have $\phi_i(\rho) \leq \beta_i$.
    \item \textbf{Non-negativity constraint:} Each $\rho_n \geq 0$ almost everywhere. Since $\rho_n \to \rho$ in $L^2(\Omega,\nu)$, there exists a subsequence $\{\rho_{n_k}\}$ that converges to $\rho$ almost everywhere (by the fact that $L^2$ convergence implies a.e. convergence of a subsequence). Since $\rho_{n_k} \geq 0$ a.e. for all $k$, it follows that $\rho \geq 0$ a.e.
  \end{itemize}

  Thus, $\rho$ satisfies all constraints, so $\rho \in \mathcal{P}$. Therefore, $\mathcal{P}$ is closed in $L^2(\Omega,\nu)$.

  Since $\mathcal{P} \subseteq V$ and $\mathcal{P}$ is closed in $L^2(\Omega,\nu)$, it is also closed in $V$ under the subspace topology (because $V$ is closed in $L^2(\Omega,\nu)$ as a finite-dimensional subspace). As $\mathcal{P}$ is bounded and closed in the finite-dimensional space $V$, it is compact in $V$ by the Heine-Borel theorem. Since the topology on $V$ is induced by the $L^2(\Omega,\nu)$ norm, $\mathcal{P}$ is compact in $L^2(\Omega,\nu)$.
\end{proof}

Given the compactness of $\mathcal{P}$, we can now prove that the hit-and-run sampler generates a Markov chain with uniform stationary distribution over $\mathcal{P}$.

\begin{proof}
  From \cref{prop: valid samples in P}, $\rho_t $ satisfies all constraints in $ \mathcal{P}$.
  Given $\rho_{t-1}$, the next state $\rho_{t}$ is obtained by sampling a direction ${d}_t$ and a step length $\lambda_t$,
  both independent of the past. Hence $\rho_t = \rho_{t-1} + \lambda_t {d}_t$
  depends only on $\rho_{t-1}$, so $\{\rho_t\}_{t=1}^T$ is a Markov chain on a finite-dimensional and compact state space. Consequently, a stationary distribution always exists.
  Furthermore, by construction, the Markov chain is irreducible and Harris recurrent.

  To identify its stationary distribution $\pi$, we verify detailed balance:
  $$
  \pi(\rho) \mathbb{P}(\rho\to\rho')
  = \pi(\rho') \mathbb{P}(\rho'\to\rho), \quad \forall \rho,\rho'\in\mathcal{P}.
  $$
  Since the symmetric directions are drawn with equal probability densities in a zero-mean Gaussian process, and step lengths uniformly over the feasible interval, the transition probabilities are symmetric, i.e.,
  $\mathbb{P}(\rho\to\rho') = \mathbb{P}(\rho'\to \rho)$.
  Hence detailed balance holds exactly when $\pi(\rho)=\pi(\rho')$ for all pairs, implying that the unique stationary distribution is uniform on $\mathcal{P}$.
\end{proof}

\subsection{Proof of Continuity for Causal Effects}
\begin{proposition}
  \label{prop: continuity of E[Y|do(a)]}
  Consider the map $V(\rho): \rho \mapsto \mathbb{E}_{\rho}[Y\mid \mathrm{do}(a)]$.
  Suppose that the sample space $\Omega$ is compact and $\nu(\Omega)<+\infty$.
  Moreover, $\rho$ has a uniform lower and upper bound $\kappa_1, \kappa_2 > 0$ (i.e., $\kappa_1 \leq \rho\leq \kappa_2$).
  Then, for each fixed $a$, there exists a constant $L_V > 0$ (depending on $a$, the compact support, the uniform bound $\kappa_1,\kappa_2$,
  and the reference measure $\nu$) such that for any two densities $\rho_1, \rho_2$ satisfying the conditions,
  $$
  |V(\rho_1) - V(\rho_2)| \leq L_V \|\rho_1 - \rho_2\|_{L^2(\Omega,\nu)}.
  $$
\end{proposition}

\begin{proof}
  Given the value $a$, define the auxiliary space $\mathcal{Z} = \mathcal{W} \times \mathcal{U}$ and the variable $z = (w, u)$. For the densities $\rho_i$ ($i = 1, 2$), define their marginal densities and conditional expectation functions:
  \begin{align*}
    \rho_{Z,i}(z) &:= \int_{\mathcal{A} \times \mathcal{Y}} \rho_i(a', y', z) \mathrm{d}\nu(a', y'), \\
    f_i(z) &:= \mathbb{E}_{\rho_i}[Y \mid a, z] = \frac{\int_{\mathcal{Y}} y \rho_i(a, y, z) \mathrm{d}\nu(y)}{\int_{\mathcal{Y}} \rho_i(a, y, z) \mathrm{d}\nu(y)}.
  \end{align*}
  According to the theorem conditions, $\Omega$ is compact and $|Y| \leq M$, so $|f_i(z)| \leq M$. The causal effect difference can be decomposed as:
  $$
  V(\rho_1) - V(\rho_2) = \underbrace{\int_{\mathcal{Z}} f_1(z) (\rho_{Z,1}(z) - \rho_{Z,2}(z)) \mathrm{d}\nu(z)}_{T_1} + \underbrace{\int_{\mathcal{Z}} (f_1(z) - f_2(z)) \rho_{Z,2}(z) \mathrm{d}\nu(z)}_{T_2}.
  $$
  We will control $|T_1|$ and $|T_2|$ separately.

  First, by $|f_1(z)| \leq M$ and the Cauchy-Schwarz inequality:
  $$
  |T_1| \leq M \int_{\mathcal{Z}} |\rho_{Z,1} - \rho_{Z,2}| \mathrm{d}\nu(z) \leq M \sqrt{\nu(\mathcal{Z})} \|\rho_{Z,1} - \rho_{Z,2}\|_{L^2(\nu_z)}.
  $$
  Further, for each $z$, by Cauchy-Schwarz and Fubini's theorem:
  \begin{align*}
    |\rho_{Z,1}(z) - \rho_{Z,2}(z)|
    &\leq \int_{\mathcal{A} \times \mathcal{Y}} |\rho_1 - \rho_2| \mathrm{d}\nu(a', y') \\
    &\leq \sqrt{\nu(\mathcal{A} \times \mathcal{Y})} \left( \int_{\mathcal{A} \times \mathcal{Y}} |\rho_1 - \rho_2|^2 \mathrm{d}\nu(a', y') \right)^{1/2}.
  \end{align*}
  Taking the $L^2(\nu_z)$ norm on both sides gives:
  $$
  \|\rho_{Z,1} - \rho_{Z,2}\|_{L^2(\nu_z)}^2 \leq \nu(\mathcal{A} \times \mathcal{Y}) \|\rho_1 - \rho_2\|_{L^2(\Omega, \nu)}^2,
  $$
  that is, $\|\rho_{Z,1} - \rho_{Z,2}\|_{L^2(\nu_z)} \leq \sqrt{\nu(\mathcal{A} \times \mathcal{Y})} \|\rho_1 - \rho_2\|_{L^2(\Omega, \nu)}$. Substituting into the previous formula:
  $$
  |T_1| \leq M \sqrt{\nu(\mathcal{Z}) \cdot \nu(\mathcal{A} \times \mathcal{Y})} \|\rho_1 - \rho_2\|_{L^2(\Omega, \nu)}.
  $$

  To bound the term,
  the key is to estimate $|f_1(z) - f_2(z)|$. Let:
  $$
  g_i(z) := \int_{\mathcal{Y}} \rho_i(a, y, z) \mathrm{d}\nu(y), \quad h_i(z) := \int_{\mathcal{Y}} y \rho_i(a, y, z) \mathrm{d}\nu(y).
  $$
  From the density bounds $\kappa_1 \leq \rho_i \leq \kappa_2$, we have $g_i(z) \geq \kappa_1 \nu(\mathcal{Y})$ and $|h_i(z)| \leq M g_i(z)$. Calculating:
  $$
  |f_1 - f_2| = \left| \frac{h_1}{g_1} - \frac{h_2}{g_2} \right| \leq \frac{|h_1| |g_2 - g_1|}{g_1 g_2} + \frac{|h_1 - h_2|}{g_2} \leq M \frac{|g_1 - g_2|}{g_2} + \frac{|h_1 - h_2|}{g_2}.
  $$
  Using $g_2 \geq \kappa_1 \nu(\mathcal{Y})$ and $|h_1 - h_2| \leq M \int_{\mathcal{Y}} |\rho_1 - \rho_2| \mathrm{d}\nu(y)$:
  \begin{align*}
    |f_1(z) - f_2(z)|
    &\leq \frac{M}{\kappa_1 \nu(\mathcal{Y})} \int_{\mathcal{Y}} |\rho_1 - \rho_2| \mathrm{d}\nu(y) + \frac{M}{\kappa_1 \nu(\mathcal{Y})} \int_{\mathcal{Y}} |\rho_1 - \rho_2| \mathrm{d}\nu(y) \\
    &= \frac{2M}{\kappa_1 \nu(\mathcal{Y})} \int_{\mathcal{Y}} |\rho_1(a, y, z) - \rho_2(a, y, z)| \mathrm{d}\nu(y).
  \end{align*}
  Substituting into $T_2$ and applying Cauchy-Schwarz:
  \begin{align*}
    |T_2|
    &\leq \frac{2M}{\kappa_1 \nu(\mathcal{Y})} \int_{\mathcal{Z}} \rho_{Z,2}(z) \left( \int_{\mathcal{Y}} |\rho_1 - \rho_2| \mathrm{d}\nu(y) \right) \mathrm{d}\nu(z) \\
    &\leq \frac{2M}{\kappa_1 \nu(\mathcal{Y})} \left( \int_{\mathcal{Z} \times \mathcal{Y}} \rho_{Z,2}(z)^2 \mathrm{d}\nu(y)\mathrm{d}\nu(z) \right)^{1/2} \left( \int_{\mathcal{Z} \times \mathcal{Y}} |\rho_1 - \rho_2|^2 \mathrm{d}\nu(y)\mathrm{d}\nu(z) \right)^{1/2}.
  \end{align*}
  From $\rho_{Z,2}(z) \leq \kappa_2 \nu(\mathcal{A} \times \mathcal{Y})$ and the definition of $\|\rho_1 - \rho_2\|_{L^2(\Omega,\nu)}$:
  \begin{align*}
    \int_{\mathcal{Z} \times \mathcal{Y}} \rho_{Z,2}(z)^2 \mathrm{d}\nu(y)\mathrm{d}\nu(z)
    &\leq \nu(\mathcal{Y}) \left( \kappa_2 \nu(\mathcal{A} \times \mathcal{Y}) \right)^2 \nu(\mathcal{Z}), \\
    \int_{\mathcal{Z} \times \mathcal{Y}} |\rho_1 - \rho_2|^2 \mathrm{d}\nu(y)\mathrm{d}\nu(z)
    &\leq \|\rho_1 - \rho_2\|_{L^2(\Omega, \nu)}^2.
  \end{align*}
  Let $B = \kappa_2 \nu(\mathcal{A} \times \mathcal{Y}) \sqrt{\nu(\mathcal{Z})}$, then:
  $$
  |T_2| \leq \frac{2M}{\kappa_1 \nu(\mathcal{Y})} \cdot \sqrt{\nu(\mathcal{Y})} B \|\rho_1 - \rho_2\|_{L^2(\Omega, \nu)} = \frac{2M B}{\kappa_1 \sqrt{\nu(\mathcal{Y})}} \|\rho_1 - \rho_2\|_{L^2(\Omega, \nu)}.
  $$
  Substituting $B$ and simplifying:
  $$
  |T_2| \leq \frac{2M \kappa_2 \nu(\mathcal{A} \times \mathcal{Y}) \sqrt{\nu(\mathcal{Z}) \nu(\mathcal{Y})} }{\kappa_1} \|\rho_1 - \rho_2\|_{L^2(\Omega, \nu)}.
  $$

  Summarizing the results of the two terms:
  $$
  |V(\rho_1) - V(\rho_2)| \leq |T_1| + |T_2| \leq C_1 \|\rho_1 - \rho_2\|_{L^2(\Omega,\nu)} + C_2 \|\rho_1 - \rho_2\|_{L^2(\Omega,\nu)},
  $$
  where:
  $$
  C_1 = M \sqrt{\nu(\mathcal{Z}) \nu(\mathcal{A} \times \mathcal{Y})}, \quad C_2 = \frac{2M \kappa_2 \nu(\mathcal{A} \times \mathcal{Y}) \sqrt{\nu(\mathcal{Z}) \nu(\mathcal{Y})} }{\kappa_1}.
  $$
  Taking $L_V = C_1 + C_2$, that is:
  $$
  L_V = M \sqrt{\nu(\mathcal{Z}) \nu(\mathcal{A} \times \mathcal{Y})} + \frac{2M \kappa_2 \sqrt{\nu(\mathcal{Z}) \nu(\mathcal{Y})} \nu(\mathcal{A} \times \mathcal{Y}) }{\kappa_1}.
  $$
  Since $\Omega$ is compact, $\nu(\mathcal{Z})$, $\nu(\mathcal{A} \times \mathcal{Y})$, $\nu(\mathcal{Y})$ are finite, and $\kappa_1, \kappa_2 > 0$, so $L_V > 0$ is a constant depending on $a$, the compact support, $\kappa_1, \kappa_2$, and $\nu$. Thus, we obtain:
  $$
  |V(\rho_1) - V(\rho_2)| \leq L_V \|\rho_1 - \rho_2\|_{L^2(\Omega, \nu)}.
  $$
\end{proof}

\begin{corollary}
  Suppose $A$ is binary.
  Under the same conditions as \cref{prop: continuity of E[Y|do(a)]},
  the average treatment effect $V(\rho) = \mathbb{E}_{\rho}[Y\mid \mathrm{do}(A=1)] - \mathbb{E}_{\rho}[Y\mid \mathrm{do}(A=0)]$ is continuous in the $L^2(\Omega,\nu)$ norm.
\end{corollary}

\begin{proposition}
  \label{prop: continuity of E[Y|do(a),w]}
  Consider the map $V(\rho): \rho \mapsto \mathbb{E}_{\rho}[Y\mid \mathrm{do}(a), w]$.
  Suppose that the sample space $\Omega$ is compact and $\nu(\Omega)<+\infty$.
  Moreover, $\rho$ has a uniform lower and upper bound $\kappa_1, \kappa_2 > 0$ (i.e., $\kappa_1 \leq \rho\leq \kappa_2$).
  Then, for each fixed $a$, there exists a constant $L_V > 0$ (depending on $a$, the compact support, the uniform bound $\kappa_1,\kappa_2$,
  and the reference measure $\nu$) such that for any two densities $\rho_1, \rho_2$ satisfying the conditions,
  $$
  |V(\rho_1) - V(\rho_2)| \leq L_V \|\rho_1 - \rho_2\|_{L^2(\Omega,\nu)}.
  $$
\end{proposition}

\begin{proof}
  Fix the intervention level $a$ and covariate $w$. The space $\mathcal{U}$ corresponds to unobserved variables. For densities $\rho_i$ ($i = 1, 2$), define:
  \begin{align*}
    \rho_{U|W,i}(u|w) &:= \frac{\int_{\mathcal{A} \times \mathcal{Y}} \rho_i(a', y', w, u) \, \mathrm{d}\nu(a', y')}{\int_{\mathcal{U}} \left( \int_{\mathcal{A} \times \mathcal{Y}} \rho_i(a', y', w, u') \, \mathrm{d}\nu(a', y') \right) \mathrm{d}\nu(u')}, \\
    f_i(u,w) &:= \mathbb{E}_{\rho_i}[Y \mid a, w, u] = \frac{\int_{\mathcal{Y}} y \rho_i(a, y, w, u) \, \mathrm{d}\nu(y)}{\int_{\mathcal{Y}} \rho_i(a, y, w, u) \, \mathrm{d}\nu(y)}.
  \end{align*}
  Then $V(\rho_i) = \int_{\mathcal{U}} f_i(u,w) \rho_{U|W,i}(u|w) \, \mathrm{d}\nu(u)$. Let:
  \begin{align*}
    D_i(w) &:= \int_{\mathcal{U}} \int_{\mathcal{A} \times \mathcal{Y}} \rho_i(a', y', w, u) \, \mathrm{d}\nu(a', y') \mathrm{d}\nu(u), \\
    N_i(w) &:= \int_{\mathcal{U}} f_i(u,w) \left( \int_{\mathcal{A} \times \mathcal{Y}} \rho_i(a', y', w, u) \, \mathrm{d}\nu(a', y') \right) \mathrm{d}\nu(u),
  \end{align*}
  so $V(\rho_i) = N_i(w)/D_i(w)$. By $\kappa_1 \leq \rho_i \leq \kappa_2$ and $|Y| \leq M$ (with $\Omega$ compact):
  $$
  D_i(w) \geq \kappa_1 \nu(\mathcal{A} \times \mathcal{Y} \times \mathcal{U}) =: c_1 > 0, \quad |N_i(w)| \leq M \kappa_2 \nu(\mathcal{A} \times \mathcal{Y} \times \mathcal{U}) =: c_2.
  $$

  We then decompose the difference and control the denominator.
  $$
  V(\rho_1) - V(\rho_2) = \frac{N_1 - N_2}{D_1} + \frac{N_2}{D_1 D_2} (D_2 - D_1).
  $$
  Using $D_i \geq c_1$ and $|N_2| \leq c_2$:
  $$
  |V(\rho_1) - V(\rho_2)| \leq \frac{|N_1 - N_2|}{c_1} + \frac{c_2}{c_1^2} |D_1 - D_2|.
  $$
  Control $|D_1 - D_2|$:
  \begin{align*}
    |D_1(w) - D_2(w)|
    &\leq \int_{\mathcal{U}} \int_{\mathcal{A} \times \mathcal{Y}} |\rho_1 - \rho_2| \, \mathrm{d}\nu(a', y') \mathrm{d}\nu(u) \\
    &\leq \sqrt{\nu(\mathcal{A} \times \mathcal{Y} \times \mathcal{U})} \cdot \|\rho_1 - \rho_2\|_{L^2(\Omega,\nu)}.
  \end{align*}

  Next, control the numerator difference $|N_1 - N_2|$.
  Define \[\rho_{A,Y,U,i}(w,u) := \int_{\mathcal{A} \times \mathcal{Y}} \rho_i(a', y', w, u) \, \mathrm{d}\nu(a', y'),\] 
  then:
  $$
  N_i(w) = \int_{\mathcal{U}} f_i(u,w) \rho_{A,Y,U,i}(w,u) \, \mathrm{d}\nu(u).
  $$
  Decompose:
  \begin{align*}
    |N_1 - N_2|
    &\leq \int_{\mathcal{U}} |f_1 - f_2| \rho_{A,Y,U,1} \, \mathrm{d}\nu(u) + M \int_{\mathcal{U}} |\rho_{A,Y,U,1} - \rho_{A,Y,U,2}| \, \mathrm{d}\nu(u).
  \end{align*}

  Let $g_i(u,w) := \int_{\mathcal{Y}} \rho_i(a, y, w, u) \, \mathrm{d}\nu(y)$, then:
  \begin{align*}
    |f_1 - f_2|
    &\leq \frac{2M}{\kappa_1 \nu(\mathcal{Y})} \int_{\mathcal{Y}} |\rho_1(a,y,w,u) - \rho_2(a,y,w,u)| \, \mathrm{d}\nu(y).
  \end{align*}
  Substitute and apply Cauchy-Schwarz:
  \begin{align*}
    &\int_{\mathcal{U}} |f_1 - f_2| \rho_{A,Y,U,1} \, \mathrm{d}\nu(u) \\
    &\leq \frac{2M}{\kappa_1 \nu(\mathcal{Y})} \sqrt{ \nu(\mathcal{Y}) \left( \kappa_2 \nu(\mathcal{A} \times \mathcal{Y}) \right)^2 \nu(\mathcal{U}) } \cdot \| \rho_1 - \rho_2 \|_{L^2(\Omega,\nu)} \\
    &= \frac{2M \kappa_2 \nu(\mathcal{A} \times \mathcal{Y}) \sqrt{\nu(\mathcal{U}) \nu(\mathcal{Y})} }{\kappa_1} \|\rho_1 - \rho_2\|_{L^2(\Omega,\nu)}.
  \end{align*}

  Since
  \begin{align*}
    M \int_{\mathcal{U}} |\rho_{A,Y,U,1} - \rho_{A,Y,U,2}| \, \mathrm{d}\nu(u)
    &\leq M \sqrt{\nu(\mathcal{U}) \nu(\mathcal{A} \times \mathcal{Y})} \|\rho_1 - \rho_2\|_{L^2(\Omega,\nu)}
  \end{align*}
  and
  $$
  |N_1 - N_2| \leq M \left( \frac{2 \kappa_2 \nu(\mathcal{A} \times \mathcal{Y}) \sqrt{\nu(\mathcal{U}) \nu(\mathcal{Y})} }{\kappa_1} + \sqrt{\nu(\mathcal{U}) \nu(\mathcal{A} \times \mathcal{Y})} \right) \|\rho_1 - \rho_2\|_{L^2(\Omega,\nu)},
  $$
  we have
  \begin{align*}
    & |V(\rho_1) - V(\rho_2)| \\
    & \quad \leq \left[ \frac{M}{c_1} \left( \frac{2 \kappa_2 \nu_{\mathcal{A}\mathcal{Y}} \sqrt{\nu({\mathcal{U}}) \nu({\mathcal{Y}})} }{\kappa_1} + \sqrt{\nu({\mathcal{U}})} \nu_{\mathcal{A}\mathcal{Y}} \right) + \frac{c_2}{c_1^2} \sqrt{\nu(\mathcal{A} \times \mathcal{Y} \times \mathcal{U})} \right] \|\rho_1 - \rho_2\|_{L^2(\Omega,\nu)}.
  \end{align*}
  Substitute $c_1 $ and $c_2 $ to get:
  $$
  L_V = \frac{M}{\kappa_1} \left( \frac{2 \kappa_2 \sqrt{\nu({\mathcal{U}}) \nu({\mathcal{Y}})} }{\kappa_1 \sqrt{\nu(\mathcal{A} \times \mathcal{Y} \times \mathcal{U})}} + \frac{1}{\sqrt{\nu({\mathcal{U}})}} \right) + \frac{M \kappa_2 \sqrt{\nu(\mathcal{A} \times \mathcal{Y} \times \mathcal{U})}}{\kappa_1^2 \nu(\mathcal{A} \times \mathcal{Y} \times \mathcal{U})}.
  $$
  Since $\Omega$ is compact, $\nu(\mathcal{A} \times \mathcal{Y} \times \mathcal{U}) < \infty$ and $L_V > 0$, the theorem is proved:
  $$
  |V(\rho_1) - V(\rho_2)| \leq L_V \|\rho_1 - \rho_2\|_{L^2(\Omega,\nu)}.
  $$
\end{proof}

\begin{corollary}
  Suppose $A$ is binary.
  Under the same conditions as Proposition~\ref{prop: continuity of E[Y|do(a),w]},
  the conditional average treatment effect $V(\rho) = \mathbb{E}_{\rho}[Y\mid \mathrm{do}(A=1),w] - \mathbb{E}_{\rho}[Y\mid \mathrm{do}(A=0),w]$ is continuous in the $L^2(\Omega,\nu)$ norm.
\end{corollary}

\subsection{Proof of Theorem \ref{thm: convergence in probability in Markov}}

\begin{proof}
  Under \cref{asp: regularity} and \cref{assu: continuous}, $V(\rho)$ is continuous and the feasible set $\mathcal P$ is compact, so there exists $\rho_{\min}\in\mathcal P$ such that $V(\rho_{\min}) = V_{\min}$. Let $\{\rho_t\}_{t\ge1}$ be the Markov chain generated by \cref{alg: hit-run for density}.
  Denote $\mathcal B(\rho, \delta) =\{\rho': \norm{\rho-\rho'}_{L^2(\Omega,\nu)}\le\delta\}$.
  By continuity of $V$, for any $\epsilon>0$ there exists $\delta>0$ such that
  $$
  V(\rho) < V_{\min} + \epsilon
  \quad\text{for all}\quad
  \rho \in \mathcal B(\rho_{\min}, \delta).
  $$

  Denote by $\mu_k(\cdot\mid\rho_0)$ the law of the chain after $k$ transitions starting from $\rho_0$.
  From Theorem 1.1 of \cite{HR_mixing_time}, $\mu_k(\cdot\mid  \rho_0)$ converges in total variation to the uniform distribution $\pi$ on $\mathcal P$.
  Hence, there exists a mixing time $\tau$ independent of initial state, such that, for all $k>\tau$,
  $$
  \bigl\|\mu_k(\cdot\mid\rho_0) - \pi\bigr\|_{TV}
  \le \tfrac12 \pi\bigl(\mathcal{B}(\rho_{\min}, \delta)\bigr).
  $$
  It follows that for any $t>\tau$,
  \begin{align*}
    \Prob{V(\rho_t) \geq  V_{\min} + \epsilon }
    & \leq  \mu_t\left(  \rho_t \notin  \mathcal{B}(\rho_{\min},\delta) \mid \rho_0  \right) \\
    & \leq  1- \pi(   \mathcal{B}(\rho_{\min},\delta)   ) + \frac{1}{2}\pi(   \mathcal{B}(\rho_{\min},\delta)   ) \\
    & = 1- \frac{1}{2}\pi(   \mathcal{B}(\rho_{\min},\delta)   ) .
  \end{align*}
  Partition the first $T$ steps into $k = \bigl\lfloor T/\tau\bigr\rfloor$ blocks of length $\tau$, and define
  $$
  A_j = \bigcap_{t=(j-1)\tau+1}^{j\tau} \{V(\rho_t)\ge V_{\min}+\epsilon\}.
  $$
  Then for any $j \ge 1$, we have
  $$
  \mathbb{P}\bigl(A_j \mid \rho_{(j-1)\tau}\bigr)
  \le \Prob{V(\rho_{j\tau}) \geq  V_{\min} + \epsilon }
  \le 1 - \tfrac12 \pi\bigl(\mathcal B(\rho_{\min}, \delta)\bigr),
  $$
  and hence by the Markov property,
  \begin{align*}
    \mathbb{P}\Bigl(\min_{1\le t\le T}V(\rho_t) < V_{\min}+\epsilon \mid  \rho_0\Bigr)
    & = 1 - \mathbb{P}\Bigl(\bigcap_{j=1}^k A_j\Bigr) \\
    & = 1 - \prod_{j=1}^{k} \mathbb{P}\Bigl( A_j   \mid A_{j-1},\cdots,A_1 , \rho_0 \Bigr) \\
    & = 1 - \prod_{j=1}^{k} \mathbb{P}\Bigl( A_j   \mid \rho_{(j-1)\tau} \Bigr) \\
    & \ge 1 - \Bigl(1 - \tfrac12 \pi\bigl(\mathcal B(\rho_{\min}, \delta)\bigr)\Bigr)^k \\
    & \to 1
    \quad\text{as }T\to\infty.
  \end{align*}
  Since $V_{\min}$ is feasible, this shows $\min_{1\le t\le T}V(\rho_t)\to V_{\min}$ in probability.
  An analogous argument applies to the maximum, yielding $\max_{1\le t\le T}V(\rho_t)\to V_{\max}$ in probability.
\end{proof}

\subsection{Proof of Theorem \ref{thm: a.s. convergence in Markov}}

\begin{proof}
  We prove the claim for $\widehat V_{\min}$; the argument for $\widehat V_{\max}$ is identical.
  By continuity of $V$ on compact $\mathcal P$, there exists $\rho_*\in\mathcal P$ with $V(\rho_*)=V_{\min}$.  This $\rho_*$ is also a local minimizer.
  Assumption \ref{asp: OPT} guarantees a radius $\delta>0$ such that
  $$
  \norm{\rho-\rho_*}_{L^2(\Omega,\nu)}\le\delta
  \quad\Longrightarrow\quad
  \texttt{OPT}_{\min}(\rho)   =\rho_*.
  $$
  Denote $B=\{\rho: \norm{\rho-\rho_*}_{L^2(\Omega,\nu)}\le\delta\}$.  Since $B\subset\mathcal P$ has positive volume, the uniform stationary distribution $\pi$ of the hit-and-run chain satisfies $\pi(B)>0$.
  Hence, the ergodic property of hit-and-run implies that
  $\rho_t\in B$ infinitely often with probability one.
  Whenever $\rho_t\in B$, we have $\rho_{min,t}=\texttt{OPT}_{\min}(\rho_t)=\rho_*$ and so
  $
  V(\rho_{min,t})=V(\rho_*)=V_{\min}.
  $
  By the Borel-Cantelli lemma , almost surely there is some finite $T_0$ such that for all $T\ge T_0$,
  $\widehat V_{\min}(T)=V_{\min}$.  Hence $\widehat V_{\min}(T) \xrightarrow{\text{a.s.}} V_{\min}$.
\end{proof}

\subsection{Proof of Proposition \ref{prop: hausdorff convergence for extended feasible set}}

\begin{proof}
  We work in the space of probability densities with the $L^2(\Omega,\nu)$ norm $\norm{\rho_1 - \rho_2}_{L^2(\Omega,\nu)} = \left( \int_\Omega |\rho_1 - \rho_2|^2 \mathrm{d}\nu \right)^{1/2}$.
  To prove Hausdorff convergence, it suffices to show:
  $$
  \limsup_{N\to\infty}\mathcal{P}^{(N)}_{\epsilon_N} \subseteq \mathcal{P}
  \quad\text{and}\quad
  \mathcal{P} \subseteq \liminf_{N\to\infty}\mathcal{P}^{(N)}_{\epsilon_N}.
  $$

  \paragraph{Upper semicontinuity.}
  Let $\rho_N \in \mathcal{P}^{(N)}_{\epsilon_N}$ and $\rho_N \to \rho$ in $L^2(\Omega,\nu)$.

  \textit{Step 1: Show $\rho$ is a probability density.}
  Since $\nu(\Omega) < \infty$, $L^2$ convergence implies $L^1$ convergence:
  $$
  \norm{\rho_N - \rho}_{L^1(\Omega,\nu)} \leq \sqrt{\nu(\Omega)} \cdot \norm{\rho_N - \rho}_{L^2(\Omega,\nu)} \to 0.
  $$
  Thus $\rho \geq 0$ $\nu$-a.e. (by a.e. convergence of a subsequence), and
  $$
  \left| \int_\Omega \rho(\bm{x})  \mathrm{d}\nu(\bm{x}) - 1 \right| = \left| \int_\Omega (\rho(\bm{x}) - \rho_N(\bm{x})) \mathrm{d}\nu(\bm{x}) \right| \leq \int_\Omega |\rho(\bm{x}) - \rho_N(\bm{x})| \mathrm{d}\nu(\bm{x}) = \norm{\rho - \rho_N}_{L^1(\Omega,\nu)} \to 0,
  $$
  so $\int_\Omega \rho(\bm{x})  \mathrm{d}\nu(\bm{x}) = 1$.

  \textit{Step 2: Equality constraints ($i \in [m]$).}
  By the Cauchy-Schwarz inequality and $\alpha_i \in L^2$:
  \begin{align*}
    \left| \int_\Omega \alpha_i(\bm{x}) \rho(\bm{x})  \mathrm{d}\nu(\bm{x}) - \beta_i \right|
    &\leq  \left| \int_\Omega \alpha_i(\bm{x}) (\rho(\bm{x}) - \rho_N(\bm{x}))  \mathrm{d}\nu(\bm{x}) \right| \\
    & \qquad + \left| \int_\Omega \alpha_i(\bm{x}) \rho_N(\bm{x})  \mathrm{d}\nu(\bm{x}) - \widehat{\beta}^{(N)}_i \right|
    + \left| \widehat{\beta}^{(N)}_i - \beta_i \right| \\
    &\leq  \norm{\alpha_i}_{L^2(\Omega,\nu)} \norm{\rho - \rho_N}_{L^2(\Omega,\nu)}  + \epsilon_N + \epsilon_N \to 0.
  \end{align*}

  \textit{Step 3: Inequality constraints ($j \in [m'] \backslash [m]$).}
  Similarly:
  \begin{align*}
    \int_\Omega \alpha_j(\bm{x}) \rho(\bm{x})  \mathrm{d}\nu(\bm{x})
    &= \int_\Omega \alpha_j(\bm{x}) \rho_N(\bm{x})  \mathrm{d}\nu(\bm{x}) + \int_\Omega \alpha_j(\bm{x}) (\rho(\bm{x}) - \rho_N(\bm{x}))  \mathrm{d}\nu(\bm{x}) \\
    &\leq (\widehat{\beta}^{(N)}_j + \epsilon_N) +  \left| \int_\Omega \alpha_j(\bm{x}) (\rho(\bm{x}) - \rho_N(\bm{x}))  \mathrm{d}\nu(\bm{x}) \right| \\
    &\leq \widehat{\beta}^{(N)}_j + \epsilon_N +  \norm{\alpha_j}_{L^2(\Omega,\nu)} \norm{\rho - \rho_N}_{L^2(\Omega,\nu)} \\
    &\leq \beta_j + 2\epsilon_N +  \norm{\alpha_j}_{L^2(\Omega,\nu)} \norm{\rho - \rho_N}_{L^2(\Omega,\nu)} \to \beta_j.
  \end{align*}
  Thus $\rho \in \mathcal{P}$, proving $\limsup_{N\to\infty}\mathcal{P}^{(N)}_{\epsilon_N} \subseteq \mathcal{P}$.

  \paragraph{Lower semicontinuity.}
  Fix any $\rho \in \mathcal{P}$. Set $\rho_N = \rho$ for all $N$. Then:
  \begin{align*}
    & \norm{\rho_N - \rho}_{L^2(\Omega,\nu)} = 0, \\
    & \left| \int \alpha_i(\bm{x}) \rho_N(\bm{x}) \mathrm{d}\nu(\bm{x}) - \widehat{\beta}^{(N)}_i \right| \leq \epsilon_N, \quad i \in [m], \\
    & \int \alpha_j(\bm{x}) \rho_N(\bm{x}) \mathrm{d}\nu(\bm{x}) \leq \widehat{\beta}^{(N)}_j + \epsilon_N,\quad  j \in [m']\backslash [m].
  \end{align*}
  So $\rho_N \in \mathcal{P}^{(N)}_{\epsilon_N}$, and thus $\rho \in \liminf_{N\to\infty}\mathcal{P}^{(N)}_{\epsilon_N}$.

  Combining both inclusions gives the result.
\end{proof}

\subsection{Proof of Proposition~\ref{prop: hausdorff-rate}}
Define constants:
\begin{align*}
  &A_{\max} = \max_{i \in [m']} \|\alpha_i\|_{L^2(\Omega, \nu)}, \\
  &\gamma = \min_{j \in [m'] \backslash [m]} \gamma_j, \\
  &G_{ij} = \int_{\Omega} \alpha_i \alpha_j \mathrm{d}\nu \quad \text{for} \quad \forall i,j \in [m], \\
  &K_0 = \sqrt{\opnorm{G^{-1}}}, \\
  &C_1 = \frac{2(A_{\max} K_0 \sqrt{m} + 1)}{\gamma}, \\
  &C_2 = \frac{2\sqrt{\nu(\Omega)} K_0 \sqrt{m}}{\delta}, \\
  &M = \max_{\rho \in \mathcal{K}} \norm{\rho}_{L^2(\Omega, \nu)}, \\
  &L_H = 2K_0 \sqrt{m} + 2M(\max(C_1, C_2) + 1).
\end{align*}

\begin{proof}
  We prove the Hausdorff distance bound by establishing two components.

  \paragraph*{Part 1: Upper bound ($\sup_{\sigma_N \in \cPeps{N}} \operatorname{dist}(\sigma_N, \mathcal{P}) \leq K \epsilon_N$).}
  Fix $\sigma_N \in \cPeps{N}$. By definition, $\sigma_N$ satisfies:
  \begin{align}
    &\int_{\Omega} \sigma_N(\bm{x}) \mathrm{d}\nu(\bm{x}) = 1 \label{eq:norm-const} \\
    &\abs{\int_{\Omega} \alpha_i(\bm{x}) \sigma_N(\bm{x}) \mathrm{d}\nu(\bm{x}) - \widehat{\beta}^{(N)}_i} \leq \epsilon_N, \quad i = 2,\dots,m \label{eq:eq-const} \\
    &\int_{\Omega} \alpha_j(\bm{x}) \sigma_N(\bm{x}) \mathrm{d}\nu(\bm{x}) \leq \widehat{\beta}^{(N)}_j + \epsilon_N, \quad j \in [m'] \backslash [m] \label{eq:ineq-const}
  \end{align}
  Using $\abs{\beta_i - \widehat{\beta}^{(N)}_i} \leq \epsilon_N$, the constraint violation $\Delta(\sigma_N)$ is bounded as:
  \begin{align*}
    \Delta(\sigma_N) &= \max \left( \max_{2 \leq i \leq m} \abs{\int \alpha_i(\bm{x}) \sigma_N(\bm{x}) \mathrm{d}\nu(\bm{x}) - \beta_i}, \max_{j \in [m'] \backslash [m]} \max\left(0, \int \alpha_j(\bm{x}) \sigma_N(\bm{x}) \mathrm{d}\nu(\bm{x}) - \beta_j\right) \right) \\
    &\leq \max \left( \max_{2 \leq i \leq m} (\epsilon_N + \epsilon_N), \max_{j \in [m'] \backslash [m]} (\epsilon_N + \epsilon_N) \right) \\
    &= 2\epsilon_N.
  \end{align*}

  We now construct $\rho \in \mathcal{P}$ such that $\|\sigma_N - \rho\|_{L^2(\Omega,\nu)} \leq K \epsilon_N$.

  \subparagraph*{Step 1: Correct equality constraints (excluding normalization).}
  Define the affine subspace for the non-normalization equality constraints:
  $$
  \mathcal{S} = \left\{ \eta \in L^2(\Omega, \nu) : \int_{\Omega} \eta(\bm{x}) \mathrm{d}\nu(\bm{x}) = 1, \int_{\Omega} \alpha_i(\bm{x}) \eta(\bm{x}) \mathrm{d}\nu(\bm{x}) = \beta_i, \quad i=2,\dots,m \right\}.
  $$
  Let $V=\mathrm{span}\{\alpha_1,\alpha_2,\cdots,\alpha_m\}$.
  Since the Gram matrix $G$ for $\{\alpha_1, \alpha_2, \dots, \alpha_m\}$ is invertible by (ii),
  $\mathscr{A}_{V}$ (the restriction of $\mathscr{A}$ in $V$) is one-to-one.
  Define $\bm{d} \in \mathbb{R}^{m}$:
  $$
  d_i = \beta_{i} - \int_{\Omega} \alpha_{i}(\bm{x}) \sigma_N(\bm{x}) \mathrm{d}\nu(\bm{x}), \quad i \in [m].
  $$
  Note $\|\bm{d}\|_\infty \leq \Delta(\sigma_N) \leq 2\epsilon_N$.
  From open mapping theorem, the map $\mathscr{A}_{V}$ has a continuous inverse on its image, and thus the preimage $\mathscr{A}_{V}^{-1}(\bm{d})$ is well-defined.
  Moreover, $\mathscr{A}_{V}^{-1}$ is linear so its operator norm is bounded: $\norm{\mathscr{A}_{V}^{-1}}\leq K_0$.
  Let $w = \mathscr{A}_{V}^{-1}(\bm{d})$ be the $L^2$-norm solution to:
  $$
  \int_{\Omega} \alpha_i(\bm{x}) w(\bm{x}) \mathrm{d}\nu(\bm{x}) = d_i, \quad i \in [m].
  $$
  By properties of Gram matrices, $\|w\|_{L^2(\Omega,\nu)} \leq K_0 \|\bm{d}\|_2 \leq K_0 \sqrt{m} \cdot 2\epsilon_N$. Define:
  $$
  \sigma_{\mathcal{S}} = \sigma_N + w.
  $$
  This satisfies $\sigma_{\mathcal{S}} \in \mathcal{S}$ (including normalization, as $\int w(\bm{x}) \mathrm{d}\nu(\bm{x}) = 0$ by linear independence) and:
  \begin{equation}\label{eq:sigma_distance}
    \|\sigma_N - \sigma_{\mathcal{S}}\|_{L^2(\Omega,\nu)} = \|w\|_{L^2(\Omega,\nu)} \leq 2K_0 \sqrt{m} \epsilon_N.
  \end{equation}

  \subparagraph*{Step 2: Convex combination with Slater point.}
  Define $\rho_\lambda = (1-\lambda) \sigma_{\mathcal{S}} + \lambda \rho^*$ for $\lambda \in [0,1]$. Since $\sigma_{\mathcal{S}}, \rho^* \in \mathcal{S}$, we have $\rho_\lambda \in \mathcal{S}$ for all $\lambda$. We choose $\lambda$ to control constraints and ensure $\lambda \to 0$ as $\epsilon_N \to 0$.

  \textit{Control inequality constraints:}
  For $j \in [m'] \backslash [m]$:
  $$
  \int_{\Omega} \alpha_j(\bm{x}) \rho_\lambda(\bm{x}) \mathrm{d}\nu(\bm{x}) = (1-\lambda) \int_{\Omega} \alpha_j(\bm{x}) \sigma_{\mathcal{S}}(\bm{x}) \mathrm{d}\nu(\bm{x}) + \lambda \int_{\Omega} \alpha_j(\bm{x}) \rho^*(\bm{x}) \mathrm{d}\nu(\bm{x}).
  $$
  Using the triangle inequality and Cauchy-Schwarz:
  \begin{align*}
    \abs{\int_{\Omega} \alpha_j(\bm{x}) \sigma_{\mathcal{S}}(\bm{x}) \mathrm{d}\nu(\bm{x}) - \beta_j} &\leq \abs{\int_{\Omega} \alpha_j(\bm{x}) \sigma_{\mathcal{S}}(\bm{x}) \mathrm{d}\nu(\bm{x}) - \int_{\Omega} \alpha_j(\bm{x}) \sigma_N(\bm{x}) \mathrm{d}\nu(\bm{x})} \\
    & \qquad + \abs{\int_{\Omega} \alpha_j(\bm{x}) \sigma_N(\bm{x}) \mathrm{d}\nu(\bm{x}) - \beta_j} \\
    &\leq \|a_j\|_{L^2(\Omega,\nu)} \|\sigma_{\mathcal{S}} - \sigma_N\|_{L^2(\Omega,\nu)} + \Delta(\sigma_N) \\
    &\leq A_{\max} \cdot 2K_0 \sqrt{m} \epsilon_N + 2\epsilon_N \\
    &= 2(A_{\max} K_0 \sqrt{m} + 1) \epsilon_N.
  \end{align*}
  Thus:
  $$
  \int_{\Omega} \alpha_j(\bm{x}) \rho_\lambda(\bm{x}) \mathrm{d}\nu(\bm{x}) \leq (1-\lambda) [\beta_j + 2(A_{\max} K_0 \sqrt{m} + 1) \epsilon_N] + \lambda (\beta_j - \gamma).
  $$
  Set $\lambda_1 = \min\left(1, \frac{2(A_{\max} K_0 \sqrt{m} + 1) \epsilon_N}{\gamma}\right)$. Then:
  $$
  \int_{\Omega} \alpha_j(\bm{x}) \rho_\lambda(\bm{x}) \mathrm{d}\nu(\bm{x}) \leq \beta_j.
  $$

  \textit{Ensure non-negativity:}
  Since $\rho^* \geq \delta > 0$ $\nu$-a.e., we control the negative part of $\sigma_{\mathcal{S}}$. By (\ref{eq:sigma_distance}):
  $$
  \|(\sigma_{\mathcal{S}})^{-}\|_{L^1(\Omega,\nu)} \leq \|w^{-}\|_{L^1(\Omega,\nu)} \leq \|w\|_{L^1(\Omega,\nu)} \leq \sqrt{\nu(\Omega)} \|w\|_{L^2(\Omega,\nu)} \leq 2\sqrt{\nu(\Omega)} K_0 \sqrt{m} \epsilon_N.
  $$
  Set $\lambda_2 = \min\left(1, \frac{2\sqrt{\nu(\Omega)} K_0 \sqrt{m} \epsilon_N}{\delta}\right)$. Then $\rho_\lambda \geq 0$.

  Take $\lambda = \max(\lambda_1, \lambda_2)$. Since $\epsilon_N \to 0$, for large $N$, $\lambda \leq C \epsilon_N$ where:
  $$
  C = \max\left( \frac{2(A_{\max} K_0 \sqrt{m} + 1)}{\gamma}, \frac{2\sqrt{\nu(\Omega)} K_0 \sqrt{m}}{\delta} \right).
  $$

  \subparagraph*{Step 3: Distance bound.}
  Take $\rho = \rho_\lambda \in \mathcal{P}$. Then:
  $$
  \|\sigma_N - \rho\|_{L^2(\Omega,\nu)} \leq \|\sigma_N - \sigma_{\mathcal{S}}\|_{L^2(\Omega,\nu)} + \|\sigma_{\mathcal{S}} - \rho\|_{L^2(\Omega,\nu)} \leq 2K_0 \sqrt{m} \epsilon_N + \lambda \|\sigma_{\mathcal{S}} - \rho^*\|_{L^2(\Omega,\nu)}.
  $$
  From previous construction, we know
  \begin{align*}
    \|\sigma_{\mathcal{S}} - \rho^*\|_{L^2(\Omega,\nu)} & \leq \|\sigma_{\mathcal{S}}\|_{L^2(\Omega,\nu)} + \|\rho^*\|_{L^2(\Omega,\nu)} \\
    & \leq \|\sigma_N\|_{L^2(\Omega,\nu)} + \|w\|_{L^2(\Omega,\nu)} + M \\
    & \leq M + 2K_0 \sqrt{m} \epsilon_N + M \leq 2M + 2K_0 \sqrt{m}.
  \end{align*}
  Thus:
  $$
  \lambda \|\sigma_{\mathcal{S}} - \rho^*\|_{L^2(\Omega,\nu)} \leq C \epsilon_N (2M + 2K_0 \sqrt{m}) = 2C(M + K_0 \sqrt{m}) \epsilon_N.
  $$
  Combining:
  $$
  \|\sigma_N - \rho\|_{L^2(\Omega,\nu)} \leq 2K_0 \sqrt{m} \epsilon_N + 2C(M + K_0 \sqrt{m}) \epsilon_N \leq L_H \epsilon_N,
  $$
  where $L_H = 2K_0 \sqrt{m} + 2C(M + K_0 \sqrt{m})$. Taking supremum:
  \begin{equation}
    \sup_{\sigma_N \in \cPeps{N}} \inf_{\rho \in \mathcal{P}} \norm{\sigma_N - \rho}_{L^2(\Omega,\nu)} \leq L_H \epsilon_N.
    \label{eq:hausdorff-1}
  \end{equation}

  \paragraph*{Part 2: Lower bound ($\sup_{\rho \in \mathcal{P}} \operatorname{dist}(\rho, \cPeps{N}) = 0$).}
  Fix $\rho \in \mathcal{P}$. Set $\sigma_N = \rho$. Then:
  \begin{itemize}
    \item $\int_{\Omega} \sigma_N(\bm{x}) \mathrm{d}\nu(\bm{x}) = 1$.
    \item For $i = 2,\dots,m$: $\abs{\int \alpha_i(\bm{x}) \sigma_N(\bm{x}) \mathrm{d}\nu(\bm{x}) - \widehat{\beta}^{(N)}_i} = \abs{\beta_i - \widehat{\beta}^{(N)}_i} \leq \epsilon_N$.
    \item For $j \in [m'] \backslash [m]$: $\int \alpha_j(\bm{x}) \sigma_N(\bm{x}) \mathrm{d}\nu(\bm{x}) \leq \beta_j \leq \widehat{\beta}^{(N)}_j + \epsilon_N$.
    \item $\sigma_N(\bm{x}) \geq 0$.
  \end{itemize}
  Thus $\sigma_N \in \cPeps{N}$, and:
  $$
  \inf_{\sigma_N \in \cPeps{N}} \norm{\rho - \sigma_N}_{L^2(\Omega,\nu)} = 0.
  $$
  Taking supremum:
  \begin{equation}\label{eq:hausdorff-2}
    \sup_{\rho \in \mathcal{P}} \inf_{\sigma_N \in \cPeps{N}} \norm{\rho - \sigma_N}_{L^2(\Omega,\nu)} = 0.
  \end{equation}

  Combining \eqref{eq:hausdorff-1} and \eqref{eq:hausdorff-2}, we have:
  $$
  d_H(\cPeps{N}, \mathcal{P}) \leq L_H \epsilon_N.
  $$
\end{proof}

\subsection{Proof of Theorem \ref{thm: convergence of empirical bounds}}

\begin{proof}
  Since $V$ is continuous on the compact set $\mathcal{K}$, it attains its minimum and maximum on each $\mathcal{P}_{\epsilon_N}^{(N)}$ and on $\mathcal{P}$.
  We prove the convergence of the maxima; the argument for the minima is analogous.

  \paragraph{Upper bound: $\limsup V_{\max}^{(N)} \le V_{\max}$.}
  Fix $\varepsilon>0$.  By uniform continuity of $V$ on $\mathcal{K}$, there exists $\delta>0$ such that
  $$
  \|\rho - \rho'\|_{L^2(\Omega,\nu)} < \delta
  \quad\Longrightarrow\quad
  |V(\rho) - V(\rho')| < \varepsilon.
  $$
  Since $d_H(\mathcal{P}_{\epsilon_N}^{(N)},\mathcal{P})\to0$, for all sufficiently large $N$, every maximizer $\rho_N\in\mathcal{P}_{\epsilon_N}^{(N)}$ (so $V(\rho_N)=V_{\max}^{(N)}$) admits some $\rho\in\mathcal{P}$ with $\|\rho_N - \rho\|_{L^2(\Omega,\nu)}<\delta$.  Hence
  $$
  V_{\max}^{(N)}
  = V(\rho_N)
  \le V(\rho) + \varepsilon
  \le V_{\max} + \varepsilon,
  $$
  and taking $\limsup_{n\to\infty}$ gives
  $\limsup V_{\max}^{(N)} \le V_{\max} + \varepsilon$.  Since $\varepsilon$ is arbitrary,
  $\limsup V_{\max}^{(N)} \le V_{\max}.$

  \paragraph{Lower bound: $\liminf V_{\max}^{(N)} \ge V_{\max}$.}
  Let $\rho^*\in\mathcal{P}$ satisfy $V(\rho^*)=V_{\max}$.  By Hausdorff convergence, there exist points $\rho_N\in\mathcal{P}_{\epsilon_N}^{(N)}$ with $\rho_N\to \rho^*$.  Continuity of $V$ then implies $V(\rho_N)\to V(\rho^*)$.  Therefore for large $N$,
  $$
  V_{\max}^{(N)}
  \ge V(\rho_N)
  > V_{\max} - \varepsilon,
  $$
  so $\liminf V_{\max}^{(N)} \ge V_{\max} - \varepsilon$.  Letting $\varepsilon\to0$ yields
  $\liminf V_{\max}^{(N)} \ge V_{\max}.$

  Combining the two bounds gives $\lim_{N\to\infty}V_{\max}^{(N)}=V_{\max}$.  An identical argument, using a minimizer of $V$ on $\mathcal{P}$, shows
  $\lim_{N\to\infty}V_{\min}^{(N)}=V_{\min}$.
\end{proof}

\subsection{Proof of Theorem \ref{thm: difference of empirical bounds and true bounds}}

\begin{proof}
  Since $\mathcal{K}$ is compact and $V$ is Lipschitz continuous (hence continuous) on $\mathcal{K}$, it attains its minimum and maximum on both $\mathcal{P}$ and $\mathcal{P}_{\epsilon_N}^{(N)}$. We prove the bound for the maxima; the argument for the minima is analogous.

  \paragraph{Upper bound: $V_{\max}^{(N)} \leq V_{\max} + L_V L_H \epsilon_N$.}
  By the Hausdorff distance condition $d_H(\mathcal{P}_{\epsilon_N}^{(N)}, \mathcal{P}) \leq L_H \epsilon_N$, for any $\rho_N \in \mathcal{P}_{\epsilon_N}^{(N)}$, there exists $\rho \in \mathcal{P}$ such that
  $$
  \|\rho_N - \rho\|_{L^2(\Omega,\nu)} \leq L_H \epsilon_N.
  $$
  Let $\rho_{\max}^{(N)} \in \mathcal{P}_{\epsilon_N}^{(N)}$ be a maximizer satisfying $V(\rho_{\max}^{(N)}) = V_{\max}^{(N)}$. Then there exists $\rho' \in \mathcal{P}$ with
  $$
  \|\rho_{\max}^{(N)} - \rho'\|_{L^2(\Omega,\nu)} \leq L_H \epsilon_N.
  $$
  By Lipschitz continuity of $V$:
  $$
  |V(\rho_{\max}^{(N)}) - V(\rho')| \leq L_V \|\rho_{\max}^{(N)} - \rho'\|_{L^2(\Omega,\nu)} \leq L_V L_H \epsilon_N.
  $$
  Thus,
  $$
  V_{\max}^{(N)} = V(\rho_{\max}^{(N)}) \leq V(\rho') + L_V L_H \epsilon_N \leq V_{\max} + L_V L_H \epsilon_N,
  $$
  since $V(\rho') \leq V_{\max}$.

  \paragraph{Lower bound: $V_{\max}^{(N)} \geq V_{\max} - L_V L_H \epsilon_N$.}
  Let $\rho_{\max} \in \mathcal{P}$ be a maximizer satisfying $V(\rho_{\max}) = V_{\max}$. By Hausdorff distance, there exists $\rho_N^* \in \mathcal{P}_{\epsilon_N}^{(N)}$ such that
  $$
  \|\rho_{\max} - \rho_N^*\|_{L^2(\Omega,\nu)} \leq L_H \epsilon_N.
  $$
  By Lipschitz continuity:
  $$
  |V(\rho_N^*) - V(\rho_{\max})| \leq L_V \|\rho_N^* - \rho_{\max}\|_{L^2(\Omega,\nu)} \leq L_V L_H \epsilon_N.
  $$
  Thus,
  $$
  V(\rho_N^*) \geq V_{\max} - L_V L_H \epsilon_N.
  $$
  Since $V_{\max}^{(N)}$ is the maximum over $\mathcal{P}_{\epsilon_N}^{(N)}$:
  $$
  V_{\max}^{(N)} \geq V(\rho_N^*) \geq V_{\max} - L_V L_H \epsilon_N.
  $$

  Combining both bounds:
  $$
  |V_{\max}^{(N)} - V_{\max}| \leq L_V L_H \epsilon_N.
  $$
  An identical argument, using a minimizer of $V$ on $\mathcal{P}$, shows
  $|V_{\min}^{(N)} - V_{\min}| \leq L_V L_H \epsilon_N$.
\end{proof}

\section{Proofs for Section \ref{sec: transfer MAB valid bounds}}

\subsection{Proof of Theorem \ref{thm: number of pulling in MAB}}

\begin{proof}
  We discuss the three cases of $h(a)$ in the statement of the theorem.
  \begin{itemize}
    \item \textbf{Case 1}: If $h(a) < \max_{i\in\mathcal{A}} l(i)$, then arm $a$ is eliminated up-front, and thus $\mean{n_a(T)} = 0$.
    \item \textbf{Case 2}:  $\max_{i\in\mathcal{A}} l(i) \leq h(a) < \mu^*$. Recall that $a^* =  \argmax_{a\in\mathcal{A}} \mean{Y \mid \mathrm{do}(A=a)}$ denote the optimal action. We define the following event
      $$
      \mathcal{E}(t) = \left\{  \widehat{\mu}_{a} \in \left[  {\mu}_{a} - \frac{2\sigma_a^2\log (2t/\delta)}{n_{a}(t)} ,  {\mu}_{a} +\frac{2\sigma_a^2\log (2t/\delta)}{n_{a}(t)}  \right]  , \forall a\in\mathcal{A}   \right\},
      $$
      then the Bernstein's inequality yields
      $$
      \Prob{ \overline{\mathcal{E}(t) }  } \leq \sum_{a\in\mathcal{A}} \exp\biggl(  - n_{a}(t) \times \frac{2 \sigma_a^2\log (2t/\delta) }{2\sigma_a^2 n_{a}(t)}   \biggr) \leq \frac{|\mathcal{A}|\delta}{t}.
      $$
      By the design of the algorithm, the event $\{A_t=a\}$ implies that
      $$h(a)  \geq  \widehat{U}_{a}(t) > \widehat{U}_{a^*}(t).$$
      However, if $\mathcal{E}(t)$ holds, then $\mu^* >  h(a)  \geq  \widehat{U}_{a}(t) $ and $\widehat{U}_{a^*}(t) \geq \mu^*$, which leads to a contradiction.
      Therefore, if $\mathcal{E}(t)$ holds, then $A_t \neq a$, hence
      \begin{align*}
        \mean{n_a(T)} = \sum_{t=1}^{T} \Prob{ A_t = a } &= \sum_{t=1}^{T} \Prob{ A_t = a \mid \mathcal{E}(t)} \Prob{ \mathcal{E}(t) } + \mathbb{P} \Bigl( A_t = a \,\Bigm|\, \overline{\mathcal{E}(t)} \Bigr) \Prob{ \overline{\mathcal{E}(t)}}  \\
        & \leq \sum_{t=1}^{T}   \Prob{ \overline{\mathcal{E}(t) }}  \leq \sum_{t=1}^{T} \frac{|\mathcal{A}|\delta}{t} \leq  |\mathcal{A}| .
      \end{align*}
    \item \textbf{Case 3}: Fix a suboptimal arm $a\neq a^*$ with $h(a)\ge\mu^*$, and let
      $$
      \mathcal{E}'(t) = \biggl\{\forall a\in\mathcal A:\;
        \bigl|\widehat\mu_a(t)-\mu_a\bigr|\le
      \sqrt{\frac{2\sigma_a^2\log (2t/\delta)}{n_a(t)}}\biggr\}.
      $$

      By Bernstein's inequality and a union bound,
      $$
      \mathbb{P}\bigl(\overline{\mathcal{E}'(t)}\bigr)
      \le
      \sum_{a\in\mathcal{A}} \exp\biggl(  - n_{a}(t) \times \frac{2 \sigma_a^2\log (2t/\delta) }{2\sigma_a^2 n_{a}(t)}   \biggr) \leq \frac{|\mathcal{A}|\delta}{t},
      $$
      so $\sum_{t=1}^T\mathbb{P}(\overline{\mathcal{E}'(t)})\leq |\mathcal{A}|$.

      Condition on the event $\mathcal{E}'(t) $,
      if $n_{a}(t) \geq 8\sigma_a^2\log T / \Delta_{a}^2$,
      then
      $$
      \widehat{U}_{a}(t) \leq U_{a}(t) = \mu_{a} + \sqrt{ \frac{2\sigma_a^2\log (2t/\delta) }{ n_{a}(t) }  } \leq \mu_{a} + \frac{1}{2}\Delta_{a} = \mu^* \leq \widehat{U}_{a^*}(t),
      $$
      so the algorithm will not choose the action $a$ at the round $t$.
      Finally, write
      \begin{align*}
        \mathbb{E}[n_a(T)] =\sum_{t=1}^T\mathbb{P}(a_t=a)
        & \le \sum_{t=1}^T\mathbb{P}\bigl(\overline{\mathcal{E}'(t)}\bigr) + \sum_{t=1}^T\mathbb{P}\bigl(\mathcal{E}'(t),\,a_t=a\bigr)  \le |\mathcal{A}| + \frac{8\sigma_a^2\log T}{\Delta_{a}^2}.
      \end{align*}
  \end{itemize}
  This completes the proof.
\end{proof}

\subsection{Proof of Theorem \ref{thm:instance_LB_UB_MAB_with_causal_bounds}}

\begin{proof}{Proof of upper bound in \cref{thm:instance_LB_UB_MAB_with_causal_bounds}.}
  Note that
  $$
  \mathbb{E}[\mathrm{Reg}(T)]
  = \sum_{a:\,\Delta_a>0} \mathbb{E}[n_a(T)]\,\Delta_a
  = \sum_{\substack{a\in\mathcal{A}^*\\\Delta_a>0}} \mathbb{E}[n_a(T)]\,\Delta_a,
  $$
  where Theorem~\ref{thm: number of pulling in MAB} gives $\mathbb{E}[n_a(T)]=0$ for $a\notin\mathcal{A}^*$.
  Moreover, that same theorem implies
  $$
  \mathbb{E}[n_a(T)] =
  \begin{cases}
    |\mathcal{A}|, & h(a)<\mu^*,\\
    \displaystyle\frac{8\sigma_a^2 \log T}{\Delta_a^2}, & h(a)\ge\mu^*.
  \end{cases}
  $$
  Hence
  $$
  \mathbb{E}[\mathrm{Reg}(T)]
  = \sum_{\substack{a\in\mathcal{A}^*\\h(a)<\mu^*}} |\mathcal{A}|\,\Delta_a
  + \sum_{\substack{a\in\mathcal{A}^*\\h(a)\ge\mu^*}} \frac{8\sigma_a^2\log T}{\Delta_a^2}\,\Delta_a
  = \mathcal{O}\Bigl(\sum_{a\in\widetilde{\mathcal{A}}^*:\,\Delta_a>0}\frac{\log T}{\Delta_a}\Bigr),
  $$
  as claimed. $\hfill \square$
\end{proof}

\begin{proof}{Proof of lower bound in \cref{thm:instance_LB_UB_MAB_with_causal_bounds}.}
  Fix any suboptimal arm $a\in\widetilde{\mathcal{A}}^*$ with gap $\Delta_a=\mu^*-\mu_a>0$.  Define two bandit instances $P$ and $Q$ by
  $$
  \mu_i^P=\mu_i^Q=\mu_i\quad(i\neq a),\qquad
  \mu_a^P=\mu_a,\quad \mu_a^Q=\mu^*.
  $$
  Let $n_a(T)$ be the number of pulls of $a$ up to time $T$, and set $E=\{n_a(T)\le T/2\}$.  Under $P$, each pull of $a$ incurs regret $\Delta_a$, so
  $$
  \mathrm{Reg}_P(T) \ge \Delta_a\,n_a(T).
  $$
  Under $Q$, each non-pull of $a$ incurs $\Delta_a$, hence
  $$
  \mathrm{Reg}_Q(T) \ge \Delta_a\,(T-n_a(T))
  \ge \frac{T\Delta_a}{2}\,\ind{E^c}.
  $$
  Adding gives
  $$
  \mathrm{Reg}_P(T)+\mathrm{Reg}_Q(T)
  \ge
  \frac{T\Delta_a}{2}\bigl(\ind{E}+\ind{E^c}\bigr)
  =\frac{T\Delta_a}{2}.
  $$
  Taking expectations and invoking the Bretagnolle-Huber inequality yields
  $$
  \mathbb{E}_P[\mathrm{Reg}_P(T)]+\mathbb{E}_Q[\mathrm{Reg}_Q(T)]
  \ge
  \frac{T\Delta_a}{2}\bigl(P(E)+Q(E^c)\bigr)
  \ge
  \frac{T\Delta_a}{4}
  \exp\bigl(-\mathrm{KL}(\mathbb{P}_P\|\mathbb{P}_Q)\bigr).
  $$
  From the assumed consistency condition, $\mathbb{E}_Q[\mathrm{Reg}_Q(T)]\le cT^p$ for some $c>0$ and $0<p<1$.
  It follows that
  $$
  \mathbb{E}_P[\mathrm{Reg}_P(T)]
  \ge
  \frac{T\Delta_a}{4}
  \exp\bigl(-\mathrm{KL}(\mathbb{P}_P\|\mathbb{P}_Q)\bigr).
  $$
  By the chain rule for KL divergences,
  $$
  \mathrm{KL}(\mathbb{P}_P\|\mathbb{P}_Q)
  =\mathbb{E}_P\bigl[n_a(T)\bigr]\;\mathrm{KL}\bigl(P_a\|Q_a\bigr),
  $$
  and for Bernoulli arms one has $\mathrm{KL}(P_a\|Q_a)=\Theta(\Delta_a^2)$.  Rearranging gives
  $$
  \mathbb{E}_P[n_a(T)]
  \ge
  \Omega\Bigl(\frac{\ln T}{\Delta_a^2}\Bigr),
  $$
  so
  $$
  \mathbb{E}_P[\mathrm{Reg}_P(T)]
  \ge
  \Delta_a\,\mathbb{E}_P[n_a(T)]
  \ge
  \Omega\Bigl(\frac{\ln T}{\Delta_a}\Bigr).
  $$
  Summing over all $a\in\widetilde{\mathcal{A}}^*$ completes the proof. $\hfill \square$
\end{proof}

\subsection{Proof of Theorem \ref{thm:minimax_UB_MAB_with_causal_bounds}}

\begin{proof}
  Since all actions outside $\mathcal{A}^*$ cannot be optimal,
  the classical UCB analysis yields the first term.
  For certain action $a$, its confidence width is at most $w$.
  Hence, summing over all rounds yields the second term.
\end{proof}

\subsection{Proof of Theorem \ref{thm:minimax_LB_MAB_with_causal_bounds}}
\begin{proof}
  For notaton brevity, denote $M = |\mathcal{A}^*| $.
  Since all arms in $\mathcal{A}^*$ can be optimal,
  we only construct worst-case instances where arms in $\mathcal{A}^*$ have means in $[l(a), h(a)]$
  and arms outside $\mathcal{A}^*$ are clear suboptimal by setting $\mu_a = l(a)$.
  We further assume that $\mu_0>\frac{1}{2}$.
  If $\mu_0\leq \frac{1}{2}$, we can replace the following $\mu_0$ with $\frac{1}{2}$.

  \textbf{Case 1 (Weak Prior): $w \geq \kappa^{-1}\sqrt{M/T}$}

  Set $\Delta = \frac{1}{2}\sqrt{\frac{M}{T}} $. Construct two instances for arms in $\mathcal{A}^*$:
  \begin{itemize}
    \item \textbf{Instance $P$:}
      \begin{align*}
        &\text{Arm } 1: \mu_1 = \mu_0 + \Delta \\
        &\text{Arms } a = 2,\dots,M: \mu_a = \mu_0
      \end{align*}
    \item \textbf{Instance $Q$:}
      \begin{align*}
        &\text{Arm } 1: \mu_1 = \mu_0 + \Delta \\
        &\text{Arm } 2: \mu_{2} = \mu_0 + 2\Delta \\
        &\text{Other arms } a \notin \{1, 2\}: \mu_a = \mu_0
      \end{align*}
  \end{itemize}
  \textbf{Prior Compliance:} For $a \in \mathcal{A}^*$, $\mu_a \in [\mu_0, \mu_0 + \kappa w]$.
  Since $2\Delta \leq \kappa w$ (as $ \kappa w \geq \sqrt{M/T}$), all $\mu_a$ satisfies
  $$
  \mu_a \in [\mu_0, \mu_0 + \kappa w] \subset [l(a), h(a)].
  $$

  By pigeonhole principle, $\exists a'$ with $\mathbb{E}_{P}[N_{a'}(T)] \leq T/(M-1)$. Define $E = \{ N_T(1) < T/2 \}$. For Bernoulli rewards:
  $$
  KL(P_{a'} \| Q_{a'}) = KL\left(\mathrm{Bern}\left(\mu_0\right) \,\|\, \mathrm{Bern}\left(\mu_0 + 2\Delta\right)\right) \leq \frac{(2\Delta)^2}{\mu_0 \cdot \mu_0} = \frac{4}{\mu_0^2}\Delta^2.
  $$
  The chain rule gives:
  $$
  KL(P_\pi \| Q_\pi) \leq \mathbb{E}_{P}[N_{a'}(T)] \cdot \frac{4}{\mu_0^2}\Delta^2 \leq \frac{T}{M-1} \cdot \mu_0^2\Delta^2 = \frac{4T}{\mu_0^2(M-1)} \cdot \frac{M}{4T} = \frac{M}{\mu_0^2(M-1)} \leq 2\mu_0^{-2}.
  $$
  By Bretagnolle-Huber inequality, we have
  $$
  \mathbb{P}_{P}(E) + \mathbb{P}_{P}(E^c) \geq \tfrac{1}{2} e^{-2\mu_0^{-2}} = C_1.
  $$
  Regret decomposition yields
  \begin{align*}
    \mathbb{E}_{P}[\mathrm{Reg}(T)] &\geq \mathbb{P}_{P}(E) \cdot \Delta \cdot (T/2) \\
    \mathbb{E}_{P}[\mathrm{Reg}(T)] &\geq \mathbb{P}_{P}(E^c) \cdot \Delta \cdot (T/2)
  \end{align*}
  since under the instance $P$, the event $E$ implies the suboptimal pull is larger than $T/2$ with gap $\geq \Delta$,
  and under $Q$, $E^c$ implies arm 1 (gap $\Delta$) is pulled $\geq T/2$ times.
  We combine these inequalities:
  $$
  \mathbb{E}_{P}[\mathrm{Reg}(T)] + \mathbb{E}_{P}[\mathrm{Reg}(T)] \geq \tfrac{\Delta T}{2} C_1 = \tfrac{C_1}{4} \sqrt{MT}.
  $$
  Thus, $\max\left\{ \mathbb{E}_{P}[\mathrm{Reg}(T)], \mathbb{E}_{P}[\mathrm{Reg}(T)] \right\} \geq \Omega\left( \sqrt{MT} \right)$.

  \textbf{Case 2 (Strong Prior): $w < \kappa^{-1}\sqrt{M/T}$}

  Set $\Delta = \kappa w$. Construct:
  \begin{itemize}
    \item \textbf{Instance $P$:}
      \begin{align*}
        &\text{Arm } 1: \mu_1 = \mu_0 + \Delta \\
        &\text{Arms } a = 2,\dots,M: \mu_a = \mu_0
      \end{align*}
    \item \textbf{Instance $Q$:}
      \begin{align*}
        &\text{Arm } 1: \mu_1 = \mu_0 \\
        &\text{Arm } 2: \mu_{2} = \mu_0 + \Delta \\
        &\text{Other arms } a \in \mathcal{A}^* \setminus \{1, 2\}: \mu_a = \mu_0
      \end{align*}
  \end{itemize}
  \textbf{Prior Compliance:} Similar to Case 1, $\mu_a \in [\mu_0, \mu_0 + \Delta]$ for $a \in \mathcal{A}^*$ implies the satisfication.

  Choose $a'$ with $\mathbb{E}_{P}[N_{a'}(T)] \leq T/(M-1)$. Define $E = \{ N_T(1) \geq T/2 \}$. For Bernoulli rewards:
  \begin{align*}
    KL(P_1 \| Q_1) &= KL\left(\mathrm{Bern}\left(\mu_0 + \Delta\right) \|\, \mathrm{Bern}\left(\mu_0\right)\right) \leq \frac{4\Delta^2}{\mu_0 \cdot \mu_0} = \frac{2\kappa^2}{\mu_0^2}w^2 \\
    KL(P_{a'} \|\, Q_{a'}) &= KL\left(\mathrm{Bern}\left(\mu_0\right) \| \mathrm{Bern}\left(\mu_0 + \Delta\right)\right) \leq \frac{2\kappa^2}{\mu_0^2}w^2.
  \end{align*}
  The chain rule gives:
  $$
  KL(P_\pi \| Q_\pi) \leq \mathbb{E}_{P}\left[N_T(1) \cdot \frac{2\kappa^2}{\mu_0^2}w^2 + N_{a'}(T) \cdot \frac{2\kappa^2}{\mu_0^2}w^2\right] \leq \frac{2\kappa^2}{\mu_0^2} w^2 \left( T + \frac{T}{M-1} \right) \leq \frac{2M^2}{\mu_0^2(M-1)}.
  $$
  Bretagnolle-Huber yields:
  $$
  \mathbb{P}_{P}(E) + \mathbb{P}_{P}(E^c) \geq \tfrac{1}{2} e^{-\frac{2M^2}{\mu_0^2(M-1)}} = C_2.
  $$
  Similar reasoning as Case 1 gives:
  \begin{align*}
    \text{Under } P: \quad & \text{If } E^c \text{ (i.e., } N_T(1) < T/2\text{)}, \mathrm{Reg}(T)  \geq \Delta \cdot (T/2) \\
    \text{Under } Q: \quad & \text{If } E \text{ (i.e., } N_T(1) \geq T/2\text{)},  \mathrm{Reg}(T) \geq \Delta \cdot (T/2)
  \end{align*}
  since arm 1 has gap $\Delta$ in both cases. Therefore, we have
  $$
  \mathbb{E}_{P}[\mathrm{Reg}(T)] + \mathbb{E}_{P}[\mathrm{Reg}(T)] \geq \tfrac{\kappa C_2 w T}{2}.
  $$
  Hence, $\max\left\{ \mathbb{E}_{P}[\mathrm{Reg}(T)], \mathbb{E}_{P}[\mathrm{Reg}(T)] \right\} \geq \Omega\left( wT \right)$.
\end{proof}

\section{Proofs for Section \ref{sec: transfer noisy}}\label{app: proofs for experimental data}



\subsection{Proof of Lemma \ref{lem: warm-start UCB event}}

\begin{proof}
  Given the number of pulls $n_a(t)$, the rewards for the arm $a$ is conditionally independent.
  Now, note that
  $$
  \mathbb{E}[\widehat{\mu}^{\epsilon}_a(t)] = \frac{n_a(t) \mu_a +  \epsilon^{-2}_a(\delta) \widehat{h}(a)}{n_a(t) +  \epsilon^{-2}_a(\delta)}.
  $$
  Note that ${\sigma}_a^2 = \max_{\mu \in [ l(a), h(a)]}  \mu(1-\mu) $ is the true maximum variance and $\prob{\widehat{\sigma}_a^2\geq {\sigma}_a^2}\geq 1-\delta$.
  Applying Bernstein's inequality, we have
  \begin{align*}
    \bigg|\widehat{\mu}^{\epsilon}_a(t) - \frac{n_a(t) \mu_a +  \epsilon^{-2}_a(\delta) \widehat{h}(a)}{n_a(t) +  \epsilon^{-2}_a(\delta)}\bigg|
    & \le \frac{n_{a}(t)}{n_{a}(t)+ \epsilon^{-2}_{a}(\delta)}\sqrt{\frac{ 2{\sigma}_a^2 \log\bigl(2/\delta\bigr)}{ n_{a}(t) }}\\
    & \le \frac{n_{a}(t)}{n_{a}(t)+ \epsilon^{-2}_{a}(\delta)}\sqrt{\frac{ 2\widehat{\sigma}_a^2 \log\bigl(2/\delta\bigr)}{ n_{a}(t) }}
  \end{align*}
  with probability at least $1 - 2\delta$.

  Now we apply the union bound for all $n_a(t) \le t$, we have
  $$
  \mathbb{P}\left(
    \bigg|\widehat{\mu}^{\epsilon}_a(t) - \frac{n_a(t) \mu_a +  \epsilon^{-2}_a(\delta) \widehat{h}(a)}{n_a(t) +  \epsilon^{-2}_a(\delta)}\bigg|
    \le  \frac{n_{a}(t)}{n_{a}(t)+ \epsilon^{-2}_{a}(\delta)}\sqrt{\frac{ 2\widehat{\sigma}_a^2 \log\bigl(2t/\delta\bigr)}{ n_{a}(t) }}
  \right)
  \geq 1 - 2\delta.
  $$
  Since $\mu_a \leq h(a)$, the inequalities above implies
  \begin{align*}
    & |\widehat{\mu}^{\epsilon}_a(t) - \mu_a   | \\
    & \leq \left|\widehat{\mu}^{\epsilon}_a(t) - \frac{ n_{a}(t) \mu_a + \epsilon^{-2}_a(\delta) \widehat{h}(a) }{n_{a}(t) + \epsilon^{-2}_a(\delta) }  \right| +  \left|\frac{  \epsilon^{-2}_a(\delta) ( \widehat{h}(a) - h(a) ) }{n_{a}(t) + \epsilon^{-2}_a(\delta) }\right| + \left|\frac{  \epsilon^{-2}_a(\delta) h(a)+ \mu_a n_{a}(t)}{n_{a}(t) + \epsilon^{-2}_a(\delta) }- \mu_a \right| \\
    & \leq  \frac{n_{a}(t)}{n_{a}(t)+ \epsilon^{-2}_{a}(\delta)}\sqrt{\frac{ 2\widehat{\sigma}_a^2 \log\bigl(2t/\delta\bigr)}{ n_{a}(t) }} + \frac{ \epsilon^{-1}_a(\delta) }{n_{a}(t)+ \epsilon^{-2}_{a}(\delta)} + \frac{  \epsilon^{-2}_a(\delta) (h(a)-\mu_a) }{n_{a}(t) + \epsilon^{-2}_a(\delta) } \\
    & \leq \sqrt{\frac{ 2\widehat{\sigma}_a^2\log\bigl(2t/\delta\bigr)  + 1  }{ n_{a}(t)+ \epsilon^{-2}_{a}(\delta) }}+ \frac{  \epsilon^{-2}_a(\delta) (h(a)-\mu_a) }{n_{a}(t) + \epsilon^{-2}_a(\delta) }
  \end{align*}
  with probability at least $1 - 2\delta$.

  Consequently, we have the probablity bound
  \begin{align*}
    & \mathbb{P}\left( \bigl|\widehat{\mu}_{a}(t)-\mu_a\bigr|
      \le
      \sqrt{\frac{ 2\widehat{\sigma}_a^2 \log\bigl(2t/\delta\bigr)}{ n_{a}(t)}}, \quad
      |\widehat{\mu}^{\epsilon}_a(t) - \mu_a   |
      \le
      \sqrt{\frac{ 2\widehat{\sigma}_a^2\log\bigl(2t/\delta\bigr)  + 1  }{ n_{a}(t)+ \epsilon^{-2}_{a}(\delta) }}+ \frac{  \epsilon^{-2}_a(\delta) (h(a)-\mu_a) }{n_{a}(t) + \epsilon^{-2}_a(\delta) }
    \right) \\
    &\geq 1 - 2| \widehat{\mathcal{A}}|\delta.
  \end{align*}
\end{proof}

\subsection{Proof of Theorem \ref{thm: regret upper bound of MAB with noisy causal bounds}}

The proof of \cref{thm: regret upper bound of MAB with noisy causal bounds} relies on the following lemma, which provides an upper bound on the number of suboptimal pulls.
\begin{lemma}
  \label{lem: number of suboptimal pulls in noisy MAB}
  Let $a$ be a sub-optimal arm.
  Conditioning on the event $\mathcal{E}$, if the number of pulls
  \begin{equation}\label{eq:LB suboptimal pulls}
    n_a(t)> \frac{8 L_t - H_a}{\Delta_a^2}, \quad  L_t \triangleq 2\widehat{\sigma}_a^2 \log\bigl(2 t/\delta\bigr),
  \end{equation}
  then we have $a_t\neq a$.
\end{lemma}

\begin{proof}{Proof of \cref{thm: regret upper bound of MAB with noisy causal bounds}.}
  Recall that $\widehat{\mathcal{E}}$ denote the ``good'' event that $\widehat{\mathcal{A}}$ in \eqref{eq: active action set} retains the best arm.
  Set $\delta = T^{-1}$ in \cref{asp: causal bounds with noisy estimates} and \cref{lem: warm-start UCB event} and apply \cref{lem: number of suboptimal pulls in noisy MAB}, we have
  \begin{align*}
    \mathbb{E}[\mathrm{Reg}(T)]
    & = \sum_{a: \Delta_a > 0} \mathbb{E}\left[ \sum_{t=1}^{T} \ind{a_t = a} \cdot \left( \mu^* - \mu_a \right) \right]  \\
    & \leq \sum_{a \in \widehat{\mathcal{A}}: \Delta_a > 0} \mathbb{E}\left[ \sum_{t=1}^{T} \ind{a_t = a} \cdot \Delta_a \mid \mathcal{E}\cap \widehat{\mathcal{E}} \right] + \max_a \Delta_a T (\mathbb{P}(\overline{\mathcal{E}}) +\mathbb{P}(\overline{\widehat{\mathcal{E}}}) ) \\
    & = \sum_{a \in \widehat{\mathcal{A}}: \Delta_a > 0} \mathbb{E}\left[ n_a(T) \Delta_a   \mid \mathcal{E}\cap \widehat{\mathcal{E}}  \right] + 4 |\mathcal{A}| \max_a \Delta_a\\
    & \leq \sum_{a \in \widehat{\mathcal{A}}: \Delta_a > 0} \frac{(8 L_T - H_a)_+}{\Delta_a} + 4 |\mathcal{A}| \max_a \Delta_a, \\
    & = \mathcal{O}\left(\sum_{a \in \widehat{\mathcal{A}}: \Delta_a > 0} \frac{\left(16 \widehat{\sigma}_a^2\log\bigl( 2| T^2\bigr) - H_a\right)_+}{\Delta_a}\right).
  \end{align*}
  $\hfill \square$
\end{proof}

\begin{proof}{Proof of \cref{lem: number of suboptimal pulls in noisy MAB}.}
  We discuss two cases based on the magnitude of $H_a$.
  We condition throughout on the ``good'' event $\mathcal{E}$ from Lemma~\ref{lem: warm-start UCB event}.

  \paragraph{Case 1: $ H_a \le 8L_t$.}
  Then the lower bound \eqref{eq:LB suboptimal pulls} reduces to
  $
  n_a(t) > 8L_t/\Delta_a^{2}.
  $
  On $\mathcal{E}$,
  $$
  U_a(t)
  =\widehat\mu_a(t)+\sqrt{\frac{L_t}{n_a(t)}}
  \le \mu_a+2\sqrt{\frac{L_t}{n_a(t)}}
  < \mu_a+\frac{\Delta_a}{2}<\mu^{*}.
  $$
  For the optimal arm $a^{*}$ the same event gives
  $U_{a^{*}}(t)\ge\mu^{*}$ and $U^{\epsilon}_{a^{*}}(t)\ge\mu^{*}$.
  Hence
  $
  \min\{U_a(t),U^{\epsilon}_a(t)\}\le U_a(t)<\mu^{*}<
  \min\{U_{a^{*}}(t),U^{\epsilon}_{a^{*}}(t)\},
  $
  so $a_t\neq a$.

  \paragraph{Case 2:} $H_{a} > 8 L_t$.  In this case, the definition of $H_a$ implies
  $$
  n_{a}(t) + \epsilon_a^{-2} \ge \epsilon_a^{-2}  >  \frac{8 L_t}{(\mu^{*}-h(a))^{2}}.
  $$
  On the event $\mathcal{E}$, we have
  \begin{align*}
    \widehat\mu^{\epsilon}_{a}(t)
    = \frac{ n_{a}(t) \widehat\mu_{a}(t) + \epsilon_a^{-2} \widehat h(a) }{ n_{a}(t)+\epsilon_a^{-2} }
    & \le  \sqrt{\frac{ L_t  }{ n_{a}(t)+ \epsilon^{-2}_{a}(\delta) }}+ \frac{  \epsilon^{-2}_a(\delta) h(a)  + n_{a}(t)\mu_a }{n_{a}(t) + \epsilon^{-2}_a(\delta) }.
  \end{align*}
  Since $n_{a}(t) + \epsilon_a^{-2} > 8 L_t/(\mu^{*}-h(a))^{2}$, we get
  $$
  \sqrt{\frac{ L_t}{n_{a}(t)+\epsilon_a^{-2} }}
  < \frac{\mu^{*}-h(a)}{2}.
  $$
  Finally, recall from the definition \eqref{eq:warm-start UCB}, we have
  \begin{align*}
    U^{\epsilon}_{a}(t)
    & = \widehat{\mu}^{\epsilon}_a(t) + \sqrt{ \frac{L_t}{n_{a}(t) + \epsilon^{-2}_a(\delta)} } \\
    & <  \frac{  \epsilon^{-2}_a(\delta) h(a)  + n_{a}(t)\mu_a }{n_{a}(t) + \epsilon^{-2}_a(\delta) } + 2 \sqrt{ \frac{L_t}{n_{a}(t) + \epsilon^{-2}_a(\delta)} } \\
    & < \frac{  \epsilon^{-2}_a(\delta) h(a)  + n_{a}(t)\mu_a }{n_{a}(t) + \epsilon^{-2}_a(\delta) } + \mu^{*}-h(a) \\
    & = \mu^{*} - \frac{ n_{a}(t) }{ n_{a}(t)+\epsilon_a^{-2} } (h(a) - \mu_a) \leq \mu^*,
  \end{align*}
  where we used the assumption that $h(a) \geq \mu_a$. Hence
  $$
  \min\{ U_{a}(t), U^{\epsilon}_{a}(t)\}  \le  U^{\epsilon}_{a}(t)  < \mu^{*}
  < \min\{ U_{a^{*}}(t), U^{\epsilon}_{a^{*}}(t)\},
  $$
  which shows $a_{t}\neq a$.

  In both cases the stated threshold on $n_a(t)$ prevents arm $a$ from being selected, completing the proof.
  $\hfill \square$
\end{proof}

\subsection{Proof of Theorem \ref{thm: worst-case regret bound of MAB with noisy causal bounds}}

We also establish the worst-case regret bound for the \cref{alg: MAB with noisy causal bounds}.

\begin{theorem}\label{thm: worst-case regret bound of MAB with noisy causal bounds}
  With probability at least $1 - \delta$, $\mathrm{Reg}(T)$ is upper bounded by
  $$
  \mathcal{O}\left(
    \min\Biggl\{
      \max_{a\in \widehat{\mathcal{A}}} \widehat{\sigma}_a \sqrt{ |\widehat{\mathcal{A}}|T \log\left(4|\widehat{\mathcal{A}}|T/ \delta\right)} ,
      \biggl[
        T \sqrt{\frac{ \log(4|\widehat{\mathcal{A}}|T/\delta)}{\tau^*}} +  \sum_{a\in \widehat{\mathcal{A}}} \frac{h(a)-l(a)}{\epsilon^{2}_{a}(\delta) }  \log\Bigl(\kappa_a(T) \Bigr)
      \biggr]
    \Biggr\}
  \right),
  $$
  where $
  \kappa_a(T) \triangleq \frac{(h(a)-l(a))\left( T + \sum_{a'\in \widehat{\mathcal{A}}}  \epsilon^{-2}_{a'}(\delta) \right)}{\sum_{a'\in \widehat{\mathcal{A}}} (h(a')-l(a'))  \epsilon^{-2}_{a'}(\delta)} =\mathcal{O}(T)$
  and $\tau^*$ is the solution to $\sum_{a\in \widehat{\mathcal{A}}} (\widehat{\sigma}_a^2 \tau -  \epsilon^{-2}_{a}(\delta))_+  = T.$
\end{theorem}

\begin{proof}{Proof of \cref{thm: worst-case regret bound of MAB with noisy causal bounds}.}
  We analyze the regret upper bound given the event $\mathcal{E}$ and $\widehat{\mathcal{E}}$, so that the optimal arm is in the best arm candidate set $\widehat{\mathcal{A}}$ and the UCB holds.
  Let $a^*$ denote the arm with the highest expected reward, i.e., $a^* = \argmax_{a \in \mathcal{A}} \mathbb{E}[Y \mid \mathrm{do}(A = a)]$.
  Definte the confidence radii
  $$
  \mathrm{rad}_{a}(t) = \sqrt{\frac{ 2\widehat{\sigma}_a^2 \log(2 t/\delta)}{ n_{a}(t)}} \quad \text{and} \quad
  \mathrm{rad}_{a}^{\epsilon}(t) = \sqrt{ \frac{ 2\widehat{\sigma}_a^2\log(2 t/\delta)}{ n_{a}(t) + \epsilon^{-2}_a(\delta) } }+ \frac{  \epsilon^{-2}_a(\delta) (h(a)-\mu_a) }{n_{a}(t) + \epsilon^{-2}_a(\delta) }.
  $$
  Hence, under events $\mathcal{E}$ and $\widehat{\mathcal{E}}$, we have
  \begin{align*}
    \mathrm{Reg}(T)
    & =  \sum_{t=1}^{T} \bigl( \mathbb{E}[Y \mid \mathrm{do}(A = a^*)] - \mathbb{E}[Y \mid \mathrm{do}(A = a_t)] \bigr) \\
    & \leq \sum_{t=1}^{T} \left[ \min\{U_{a^*}(t), U_{a^*}^{\epsilon}(t)\} - \mathbb{E}[Y \mid \mathrm{do}(A = a_t)] \right] \\
    & \leq \sum_{t=1}^{T}\left[ \min\{U_{a_t}(t), U_{a_t}^{\epsilon}(t)\} - \mathbb{E}[Y \mid \mathrm{do}(A = a_t)]\right] \\
    & \leq \sum_{t=1}^{T} 2\left[\min\left\{ \mathrm{rad}_{a_t}(t), \mathrm{rad}_{a_t}^{\epsilon}(t) \right\} \right],
  \end{align*}
  where the first and last inequality follows from the validity of the UCBs, and the second inequality uses the fact that $a_t$ is chosen to maximize the minimum of the two UCBs and that event $\widehat{\mathcal{E}}$ holds.

  The conventional analysis for the classical UCB $\mathrm{rad}_{a_t}(t)$ shows that
  \begin{equation}\label{eq:conventional UCB regret bound}
    \sum_{t=1}^{T} \mathrm{rad}_{a_t}(t) = \mathcal{O}\left( \max_{a\in \widehat{\mathcal{A}}} \widehat{\sigma}_a \sqrt{  |\widehat{\mathcal{A}}| T \log( T / \delta)} \right).
  \end{equation}
  Hence, we only need to analyze the upper bound related to the warm-start UCB radius $\mathrm{rad}_{a_t}^{\epsilon}(t)$.

  By the definition of $\mathrm{rad}_{a_t}^{\epsilon}(t)$,
  \begin{align*}
    \sum_{t=1}^{T} \mathrm{rad}_{a_t}^{\epsilon}(t)
    & = \sum_{t=1}^{T} \left( \sqrt{\frac{ 2\widehat{\sigma}_a^2 \log(2t/\delta)}{ n_i(t) + \epsilon_i^{-2} }}+ \frac{  \epsilon^{-2}_a(\delta) (h(a)-\mu_a) }{n_{a} + \epsilon^{-2}_a(\delta) } \right) \\
    & \leq  \sum_{a \in \widehat{\mathcal{A}}} \sum_{n_a=1}^{n_a(T)} \left( \sqrt{ \frac{2\widehat{\sigma}_a^2 \log(2T/\delta)}{ n_a +  \epsilon^{-2}_a(\delta) }}+ \frac{  \epsilon^{-2}_a(\delta) (h(a)-\mu_a) }{n_{a} + \epsilon^{-2}_a(\delta) } \right) \\
    & \leq  \sum_{a \in \widehat{\mathcal{A}}} \sum_{n_a=1}^{n_a(T)} \left( \sqrt{\frac{ 2\widehat{\sigma}_a^2\log(2T/\delta)}{ n_a +  \epsilon^{-2}_a(\delta) }}+ \frac{  \epsilon^{-2}_a(\delta) (h(a)-l(a)) }{n_{a} + \epsilon^{-2}_a(\delta) } \right) ,
  \end{align*}
  where by the design of our algorithm $\sum_{a \in \widehat{\mathcal{A}}} n_a(T) = T$.
  Note that the last term no longer depends on specific problem instance, but only on the causal bounds $h(a)$ and $l(a)$, which are assumed to be known.

  In \cref{prop: optimization for variance term} and \cref{prop: optimization for bias term}, we establish worse case (with respect to all possible allocation of samples) upper bound on the two terms in the above summation.
  In particular, with $K = |\widehat{\mathcal{A}}|$, $b_i=\epsilon_i^{-2}(\delta)$ and $c_i=h(i)-l(i)$, we have
  \begin{align}
    & \sum_{a \in \widehat{\mathcal{A}}} \sum_{n_a=1}^{n_a(T)} \left( \sqrt{\frac{ 2\widehat{\sigma}_a^2\log(2T/\delta)}{ n_a +  \epsilon^{-2}_a(\delta) }}+ \frac{  \epsilon^{-2}_a(\delta) (h(a)-\mu_a) }{n_{a} + \epsilon^{-2}_a(\delta) } \right) \notag \\
    \leq & T  \sqrt{\frac{ \log(2T/\delta)}{\tau^*}}+  \sum_{a \in \widehat{\mathcal{A}}} (h(a)-l(a)) \epsilon_a^{-2}(\delta) \log\left( \frac{(h(a)-l(a)) \left( T + \sum_{j \in \widehat{\mathcal{A}}} \epsilon_j^{-2}(\delta) \right)}{\sum_{j \in \widehat{\mathcal{A}}} (h(j)-l(j)) \epsilon_j^{-2}(\delta)} \right). \label{eq:warm-start UCB regret bound}
  \end{align}
  Combining \eqref{eq:conventional UCB regret bound} and \eqref{eq:warm-start UCB regret bound}, we obtain the desired regret bound.
\end{proof}

The rest of this section is devoted to proving the two propositions used in the proof of \cref{thm: worst-case regret bound of MAB with noisy causal bounds}.

\begin{proposition}
  \label{prop: optimization for bias term}
  Given the positive constants $b_a$ and $c_a$,
  consider the following optimization problem:
  $$
  \max_{\bm{n}}\,\,\sum_{i=1}^K \sum_{t=1}^{n_i}   \frac{c_i }{ t + b_i}    \quad \text{subject to} \quad \sum_{i=1}^K n_i = T,\quad n_i \in \mathbb{N}.
  $$
  Then its optimal value is upper bounded by      $
  \sum_{i=1}^K b_ic_i \log\left( \frac{c_i \left( T + \sum_{j=1}^K b_j \right)}{\sum_{j=1}^K b_jc_j} \right).
  $

\end{proposition}
\begin{proof}
  To derive an upper bound for this optimization problem,
  we relax the non-negative integer variables $n_i$ to real variables and analyze the problem using the Lagrangian multiplier method.

  For each item $i$, since $f(t) = \frac{1}{t + b_i}$ is monotonically decreasing, we use the inequality relationship between summation and integration:
  $$
  \sum_{t=1}^{n_i} \frac{1}{t + b_i} \leq \int_{0}^{n_i} \frac{1}{t + b_i} dt = \log\left( \frac{n_i + b_i}{b_i} \right).
  $$
  Thus, the original objective function has an upper bound:
  \begin{equation}\label{eq: upper bound for objective function}
    \sum_{i=1}^K b_ic_i \log\left( \frac{n_i + b_i}{b_i} \right).
  \end{equation}

  Relax $n_i$ to (possibly negative)\footnote{We allow negative values to obtain a simpler and yet slightly looser upper bound.} real numbers and construct the Lagrangian:
  $$
  \mu_0 = \sum_{i=1}^K b_ic_i \log\left( \frac{n_i + b_i}{b_i} \right) - \lambda \left( \sum_{i=1}^K n_i - T \right).
  $$
  Taking the derivative with respect to $n_i$ and setting it to zero:
  $$
  \frac{\partial \mu_0}{\partial n_i} = \frac{b_ic_i}{n_i + b_i} - \lambda = 0 \quad \Rightarrow \quad n_i + b_i = \frac{b_ic_i}{\lambda}.
  $$
  Solving for $n_i$:
  \begin{equation}\label{eq: optimal n}
    n_i = \frac{b_ic_i}{\lambda} - b_i,
  \end{equation}
  which yields a solution
  \begin{equation}\label{eq: optimal lambda}
    \lambda = \frac{\sum_{i=1}^K b_ic_i}{T + \sum_{i=1}^K b_i}.
  \end{equation}

  Substitute \eqref{eq: optimal lambda} and \eqref{eq: optimal n} back into \eqref{eq: upper bound for objective function}, we obtain:
  $$
  \sum_{i=1}^K b_ic_i \log\left( \frac{c_i \left( T + \sum_{j=1}^K b_j \right)}{\sum_{j=1}^K b_jc_j} \right).
  $$
  This gives us the desired upper bound.
\end{proof}

\begin{proposition}
  \label{prop: optimization for variance term}
  Given the positive constants $b_i$ and $c_i$ for $i = 1, \ldots, K$,
  consider the following optimization problem:
  $$
  \max_{ \bm{n} } \,\, \sum_{i=1}^K \sum_{t=1}^{n_i}  \sqrt{ \frac{c_i}{ t+ b_i} }   \quad \text{subject to} \quad \sum_{i=1}^K n_i = T,\quad n_i \in \mathbb{N}.
  $$

  Then its optimal value is upper bounded by  $ T / \sqrt{\tau^*} $,
  where $\tau^*$ is the solution to
  $$
  \sum_{i=1}^K (c_i\tau - b_i)_+ = T.
  $$
\end{proposition}

\begin{proof}
  Note that
  $$
  \sum_{t=1}^{n_{i}} \sqrt{ \frac{c_i}{ t+ b_i} }
  \le
  \int_{0}^{n_{i}} \sqrt{ \frac{c_i}{ t+ b_i} } dx
  =
  2\sqrt{c_i}\Bigl(\sqrt{b_i + n_{i}} - \sqrt{b_i}\Bigr).
  $$
  Hence it suffices to consider the continuous relaxation of the following optimization problem:
  $$
  \max_{\bm{n}} \,\, \sum_{i=1}^K \sqrt{c_i}\left( \sqrt{n_i + b_i} - \sqrt{b_i} \right) \quad \text{subject to} \quad \sum_{i=1}^K n_i = T,\quad n_i \geq 0.
  $$
  Clearly, the optimal value of the relaxed optimization problem provides a valid upper bound for the original problem.

  To solve this problem, we introduce Lagrange multipliers: a scalar $\lambda \in \mathbb{R}$ for the equality constraint and non-negative multipliers $\mu_i \geq 0$ for the inequality constraints. The Lagrangian is given by:
  $$
  \mathcal{\mu_0}(n, \lambda, \mu) = \sum_{i=1}^K \sqrt{c_i} \left( \sqrt{n_i + b_i} - \sqrt{b_i} \right) - \lambda \left( \sum_{i=1}^K n_i - T \right) + \sum_{i=1}^K \mu_i n_i
  $$
  The Karush-Kuhn-Tucker (KKT) conditions for optimality are:
  $$
  \begin{cases}
    \frac{\partial \mathcal{\mu_0}}{\partial n_i} = \dfrac{\sqrt{c_i}}{2\sqrt{n_i^* + b_i}} - \lambda + \mu_i = 0 & \text{(Stationarity)} \\
    \mu_i n_i^* = 0 & \text{(Complementary Slackness)} \\
    \sum_{i=1}^K n_i^* = T,\quad n_i^* \geq 0 & \text{(Primal Feasibility)} \\
    \mu_i \geq 0 & \text{(Dual Feasibility)}
  \end{cases}
  $$

  We analyze two cases to characterize the optimal solution:

  \textit{Case 1: $n_i^* > 0$}. From complementary slackness, $\mu_i = 0$. Substituting into the stationarity condition:
  $$
  \frac{\sqrt{c_i}}{2\sqrt{n_i^* + b_i}} = \lambda \quad \Rightarrow \quad n_i^* = \frac{c_i}{4\lambda^2} - b_i
  $$

  \textit{Case 2: $n_i^* = 0$}. Substituting into the stationarity condition:
  $$
  \frac{\sqrt{c_i}}{2\sqrt{b_i}} - \lambda + \mu_i = 0 \quad \text{with} \quad \mu_i \geq 0 \Rightarrow \lambda \geq \sqrt{\frac{ {c_i}}{4 {b_i}}}
  $$

  Combining both cases, the optimal solution can be written in closed form as:
  $$
  n_i^* = \left(\frac{c_i}{4\lambda^2} - b_i\right)_+,
  $$
  and $\lambda$ is chosen to satisfy the constraint:
  \begin{align}
    \label{eq: equality for optimal tau}
    \sum_{i=1}^K \left(\frac{c_i}{4\lambda^2} - b_i\right)_+ = T.
  \end{align}
  For notational brevity, define  $\tau^* = \frac{1}{4\lambda^2} $
  for such $\lambda$ satisfying \eqref{eq: equality for optimal tau}.

  The optimal value of the objective function is then bounded by:
  \begin{align*}
    \sum_{i=1}^K \sqrt{c_i}\left(\sqrt{n_i^* + b_i} - \sqrt{b_i}\right)
    & \leq \sum_{i=1}^K  \frac{ \sqrt{c_i} n_i^*  }{  \sqrt{n_i^* + b_i} + \sqrt{b_i} }    \\
    & \leq \sum_{i=1}^K  \frac{ \sqrt{c_i} n_i^*  }{  \sqrt{n_i^* + b_i}  } \quad \text{ ($n_i^* + b_i \geq  {c_i} \tau^* $) } \\
    & \leq \sum_{i=1}^K  \frac{ n_i^*  }{  \sqrt{\tau^*}  } = \frac{ T  }{  \sqrt{\tau^*}  }.
  \end{align*}

\end{proof}

\section{Proofs for Section \ref{sec: transfer learning to POCB}}

\subsection{Proof of Theorem \ref{thm: upper bound for TL with function approximation}}

The framework presented in \cite{fasterCB,instanceCB_RL} provides a method to analyze contextual bandit algorithms in the universal policy space $\Psi$.
We mainly focus on a subspace of $\Psi$ shaped by causal bounds.
We demonstrate that the action distribution $p_m$ selected in \cref{alg: TL-function approximation} possesses desirable properties that contribute to achieving low regrets.

For each epoch $m$ and any round $t$ in epoch $m$, for any possible realization of $\gamma_t$, $\widehat{f}_m$,
we define the universal policy space of $\Psi$:
$$
\Psi = \prod_{w\in\mathcal{W}} \mathcal{A}^*(w).
$$
With abuse of notations, we define
$$
\mathcal{R}(\pi) = \mathbb{E}_W[f^*(\pi(W),W)]\text{ and } \mathrm{Reg}(\pi) = \mathcal{R}(\pi_{f^*}) -\mathcal{R}(\pi).
$$
The above quantities do not depend on specific values of $W$.
The following empirical version of above quantities are defined as
$$
\widehat{\mathcal{R}}_t(\pi) =  \widehat{f}_{m(t)}(\pi(w),w) \text{ and }    \widehat{\mathrm{Reg}}_t(\pi) =  \mathbb{E}_W\left[\widehat{\mathcal{R}}_t(\pi_{\widehat{f}_{m(t)}}) - \widehat{\mathcal{R}}_t(\pi)\right],
$$
where $m(t)$ is the epoch of the round $t$.

Let $Q_m(\cdot)$ be the equivalent policy distribution for $p_m(\cdot | \cdot)$, i.e.,
$$
Q_m(\pi) = \prod_{w\in\mathcal{W}}p_m(\pi(w)|w), \forall \pi\in\Psi.
$$
The existence and uniqueness of such measure $Q_m(\cdot)$ is a corollary of Kolmogorov's extension theorem.
Note that both $\Psi$ and $Q_m(\cdot)$ are $\mathcal{H}_{\tau_{m-1}}$-measurable,
where $\mathcal{H}_t$ is the filtration up to the time $t$.
We refer to Section 3.2 of \cite{fasterCB} for more detailed intuition for $Q_m(\cdot)$ and proof of existence.
By Lemma 4 of \cite{fasterCB}, we know that for all epoch $m$ and all rounds $t$ in epoch $m$,
we can rewrite the expected regret in terms of our notations as
$$
\mean{\mathrm{Reg}(T)} = \sum_{ \pi \in \Psi} Q_m(\pi) \mathrm{Reg}(\pi).
$$
For simplicity, we define an epoch-dependent quantities
$$
\rho_1 = 1, \rho_m = \sqrt{   \frac{\eta  \tau_{m-1}}{ \log( 2\delta^{-1} |\mathcal{F}^*| \log T )  }   }, m\geq 2,
$$
so $\gamma_t = \sqrt{|\mathcal{A}^*(w_t)|}\rho_{m(t)}   $ for $m(t)\geq 2$.

\begin{lemma}[Implicit Optimization Problem]
  \label{lem: IOP}
  For all epoch $m$ and all rounds $t$ in epoch $m$, $Q_m$ is a feasible solution to the following
  implicit optimization problem:
  \begin{align}
    & \sum_{\pi\in \Psi}  Q_m(\pi) \widehat{\mathrm{Reg}}_t(\pi) \leq   \mathbb{E}_W[\sqrt{ |\mathcal{A}^*(W)| }] /\rho_m     \\
    & \mathbb{E}_W\left[ \frac{1}{p_m(\pi(W)|W)} \right]  \leq \mathbb{E}_W[\mathcal{A}^*(W)]   + \mathbb{E}_W[  \sqrt{ |\mathcal{A}^*(W)| }] \rho_m  \widehat{\mathrm{Reg}}_t(\pi), \forall \pi \in \Psi.
  \end{align}

\end{lemma}

\begin{proof}
  Let $m$ and $t$ in epoch $m$ be fixed. Denote $\mathcal{P}(\mathcal{W})$ as the context distribution.
  We have
  \begin{align*}
    \sum_{\pi\in \Psi}  Q_m(\pi)    \widehat{\mathrm{Reg}}_t(\pi)
    & = \sum_{\pi\in \Psi}  Q_m(\pi) \mathbb{E}_{w_t} \left[  (  \widehat{f}_{m}(\pi_{\widehat{f}_{m}}(w_t),w_t) -  \widehat{f}_{m}(\pi(w_t),w_t)   )  \right]   \\
    & = \mathbb{E}_{w_t \sim \mathcal{P}(\mathcal{W})} \left[ \sum_{a\in\mathcal{A}^*(w_t)}    \sum_{\pi\in \Psi}  \ind{\pi(w_t)=a}   Q_m(\pi) (  \widehat{f}_{m}(\pi_{\widehat{f}_{m}}(w_t),w_t) -  \widehat{f}_{m}(a,w_t)   ) \right]      \\
    & = \mathbb{E}_{w_t\sim \mathcal{P}(\mathcal{W})} \left[   \sum_{a\in\mathcal{A}^*(w_t)}   p_m(a|w_t) (  \widehat{f}_{m}(\pi_{\widehat{f}_{m}}(w_t),w_t) -  \widehat{f}_{m}(a,w_t)   )\right].
  \end{align*}
  The first and second equalities follows from the definitions of $\widehat{\mathrm{Reg}}_t(\pi)$ and $Q_m(\pi)$, respectively.

  Now for the context $w_t$, we have
  \begin{align*}
    &   \sum_{a\in\mathcal{A}^*(w_t)}  p_m(a|w) (  \widehat{f}_{m}(\pi_{\widehat{f}_{m}}(w_t),w_t) -  \widehat{f}_{m}(a,w_t)   )       \\
    & =    \sum_{a\in\mathcal{A}^*(w_t)- \{ \pi_{\widehat{f}_{m}}(w_t)  \}   }   \frac{  \widehat{f}_{m}(\pi_{\widehat{f}_{m}}(w_t),w_t) -  \widehat{f}_{m}(a,w_t) }{  |\mathcal{A}^*(w_t)| + \gamma_t( \widehat{f}_{m}(\pi_{\widehat{f}_{m}}(w_t),w_t) -  \widehat{f}_{m}(a,w_t)   )        }     \\
    & \leq [  |\mathcal{A}^*(w_t)| -1   ]  / \gamma_t \\
    & \leq \sqrt{ |\mathcal{A}^*(w_t)| }/\rho_m.
  \end{align*}
  We plug in the above term and apply the i.d.d. assumption on $w_t$ to conclude the proof of the first inequality.
  For the second inequality, we first observe that for any policy $\pi \in \Psi$, given any context $w \in \mathcal{W}$,
  $$
  \frac{1}{p_m(\pi(w)|w)} = | \mathcal{A}^*(w) | + \gamma_t( \widehat{f}_{m}(\pi_{\widehat{f}_{m}}(w),w) -  \widehat{f}_{m}(a,w)   ) ,
  $$
  if $a \neq \pi_{\widehat{f}_{m}}(w)$, and
  $$
  \frac{1}{p_m(\pi(w)|w)}  \leq \frac{1}{1/ | \mathcal{A}^*(w) |} =  | \mathcal{A}^*(w) | + \gamma_t( \widehat{f}_{m}(\pi_{\widehat{f}_{m}}(w),w) -  \widehat{f}_{m}(a,w)   ) ,
  $$
  if $a = \pi_{\widehat{f}_{m}}(w)$.
  The result follows immediately by taking expectation over $w$.
\end{proof}

Compared with IOP in \cite{fasterCB},
the key different part is that $\mathbb{E}_W[  |\mathcal{A}^*(W)|]$ is replaced by the cardinality $|\mathcal{A}|$ of the whole action set.
Another different part is the universal policy space $\Psi$.
We define $\Psi$ as
$
\prod_{w\in\mathcal{W}} \mathcal{A}^*(w)
$
rather than $ \prod_{w\in\mathcal{W}} \mathcal{A}$.
These two points highlight the adaptivity to contexts and show how causal bound affects the action selection.

Define the following high-probability event
$$
\Gamma = \left\{ \forall m\geq 2 , \frac{1}{\tau_{m-1}} \sum_{t=1}^{\tau_{m-1}} \mathbb{E}_{a_t,w_t}[ ( \widehat{f}_{m(t)}(a_t,w_t) -  f^*(a_t,w_t)   )^2| \mathcal{H}_{t-1}    ]  \leq   \frac{1}{\rho_m^2}       \right\}.
$$
The high-probability event and its variants have been proved in literatures \cite{CBwithOracle,fasterCB,instanceCB_RL}.
Our result is slightly different from them as the whole function space is eliminated to $\mathcal{F}^*$.
Since these results share the same form, it is straightforward to show $\Gamma$ holds with probability at least $1-\delta/2$.
This is the result of the union bound and the property of the \textbf{Least Square Oracle} that is independent of algorithm design.

Our setting do not change the proof procedure of the following lemma \cite{fasterCB},
because this lemma does not explicitly involve the number of action set.
This lemma bounds the prediction error between the true reward and the estimated reward.
\begin{lemma}
  \label{lem: prediction error}
  Assume $\Gamma$ holds. For all epochs $m>1$, all rounds $t$ in epoch $m$, and all policies $\pi \in \Psi$, then
  $$
  \left|   \widehat{\mathcal{R}}_t(\pi) - \mathcal{R}_t(\pi) \right| \leq \frac{1}{2\rho_m} \sqrt{ \max_{1\leq m'\leq m-1}    \mathbb{E}_W \Bigg[  \frac{  1   }{p_{m'}(\pi(W)|W)}  \Bigg]     }.
  $$
\end{lemma}

The third step is to show that the one-step regret $\mathrm{Reg}_t(\pi)$ is close to the one-step estimated regret $\widehat{\mathrm{Reg}}_t(\pi)$.
The following lemma states the result.
\begin{lemma}
  \label{lem: prediction error of regrets}
  Assume $\Gamma$ holds. Let $c_0=5.15$.
  For all epochs $m$ and all rounds $t$ in epoch $m$, and all policies $\pi \in \Psi$,
  \begin{align}
    &   \mathrm{Reg}(\pi)\leq 2  \widehat{\mathrm{Reg}}_t(\pi)   + c_0  \sqrt{  \mathbb{E}_W  [   \mathcal{A}^*(W)  ] } /\rho_m , \label{eq: one-step regret bound 1} \\
    &   \widehat{\mathrm{Reg}}_t(\pi)  \leq 2    \mathrm{Reg}(\pi) + c_0 \sqrt{ \mathbb{E}_W  [  \mathcal{A}^*(W)  ]  } /\rho_m. \label{eq: one-step regret bound 2}
  \end{align}
\end{lemma}

\begin{proof}
  We prove this lemma via induction on $m$.
  It is easy to check
  $$
  \mathrm{Reg}(\pi) \leq 1, \widehat{\mathrm{Reg}}_t(\pi) \leq 1,
  $$
  as $\gamma_1 = 1$ and $c_0\mathbb{E}_W \big[   \mathcal{A}^*(W) \big] \geq 1$. Hence, the base case holds.

  For the inductive step, fix some epoch $m>1$ and assume that for all epochs $m'<m$, all rounds $t'$ in epoch $m'$, and all $\pi\in\Psi$,
  the inequalities \eqref{eq: one-step regret bound 1} and \eqref{eq: one-step regret bound 2} hold.
  We first show that for all rounds $t$ in epoch $m$ and all $\pi\in\Psi$,
  $$
  \mathrm{Reg}(\pi) \leq 2  \widehat{\mathrm{Reg}}_t(\pi)   +c_0  \sqrt{  \mathbb{E}_W  [   \mathcal{A}^*(W)  ] } /\rho_m  .
  $$
  We have
  \begin{align*}
    & \mathrm{Reg}(\pi) -   \widehat{\mathrm{Reg}}_t(\pi) \\
    =  &   [\mathcal{R}(\pi_{f^*}  ) -\mathcal{R}(\pi )]  -    [ \widehat{ \mathcal{R}  }_t (\pi_{\widehat{f}_m}  ) - \widehat{  \mathcal{R}}_t(\pi )   ] \\
    \leq   &  [\mathcal{R}(\pi_{f^*}  ) -\mathcal{R}(\pi )]  -   [ \widehat{ \mathcal{R}  }_t (\pi_{f^*}  ) - \widehat{  \mathcal{R}}_t(\pi )   ] \\
    \leq  & | \mathcal{R}(\pi_{f^*}  )  - \widehat{ \mathcal{R}  }_t (\pi_{f^*}  )| +| \mathcal{R}(\pi ) -  \widehat{  \mathcal{R}}_t(\pi )| \\
    \leq  &   \frac{1}{\rho_m} \sqrt{  \max_{1\leq m'\leq m-1}   \mathbb{E}_W \Bigg[   \frac{  1   }{p_{m'}(\pi_{f^*}(W)|W)}  \Bigg]     }  + \frac{1}{\rho_m} \sqrt{  \max_{1\leq m'\leq m-1}   \mathbb{E}_W \Bigg[ \frac{  1   }{p_{m'}(\pi(W)|W)}  \Bigg]     }\\
    \leq  &   \frac{  \max\limits_{1\leq m'\leq m-1}\mathbb{E}_W \Bigg[ \frac{  1   }{p_{m'}(\pi_{f^*}(W)|W)}  \Bigg]     }{5\rho_m  \sqrt{  \mathbb{E}_W  [   \mathcal{A}^*(W)  ] }   }          + \frac{ \max\limits_{1\leq m'\leq m-1} \mathbb{E}_W \Bigg[ \frac{  1   }{p_{m'}(\pi(W)|W)}  \Bigg] }{5\rho_m \sqrt{  \mathbb{E}_W  [   \mathcal{A}^*(W)  ] } }      +   \frac{5 \sqrt{  \mathbb{E}_W  [   \mathcal{A}^*(W)  ] }}{8 \rho_m  }.
  \end{align*}
  The last inequality is by the AM-GM inequality.
  There exists an epoch $i$ such that
  $$
  \max_{1\leq m'\leq m-1} \mathbb{E}_W \Bigg[ \frac{  1   }{p_{m'}(\pi(W)|W)}  \Bigg] =   \mathbb{E}_W \Bigg[ \frac{  1   }{p_{i}(\pi(W)|W)}  \Bigg].
  $$
  From \cref{lem: IOP} we know that
  $$
  \mathbb{E}_W \Bigg[ \frac{  1   }{p_{i}(\pi(W)|W)}  \Bigg] \leq \mathbb{E}_W[\mathcal{A}^*(W)]   + \mathbb{E}_W[  \sqrt{ |\mathcal{A}^*(W)| }] \rho_i  \widehat{\mathrm{Reg}}_t(\pi),
  $$
  holds for all $ \pi \in \Psi$, for all epoch $1\leq i\leq m-1$ and for all rounds $t$ in corresponding epochs.

  Hence, for epoch $i$ and all rounds $t$ in this epoch, we have
  \begin{align*}
    &   \frac{ \max\limits_{1\leq m'\leq m-1} \mathbb{E}_W \big[ (p_{m'}(\pi(W)|W))^{-1}  \big] }{ 5\rho_m  \sqrt{  \mathbb{E}_W  [   \mathcal{A}^*(W)  ] }     } \\
    =  &  \frac{ \mathbb{E}_W \big[ (p_i(\pi_{f^*}(W)|W))^{-1}  \big]     }{5\rho_m  \sqrt{  \mathbb{E}_W  [   \mathcal{A}^*(W)  ] }   } ,  \text{  (\cref{lem: IOP})}           \\
    \leq  &  \frac{ \mathbb{E}_W[\mathcal{A}^*(W)]   + \mathbb{E}_W[  \sqrt{ |\mathcal{A}^*(W)| }] \rho_i  \widehat{\mathrm{Reg}}_t(\pi) }{5 \sqrt{ \mathbb{E}_W[\mathcal{A}^*(W)]}    \rho_m}, \text{  (inductive assumption) }     \\
    \leq  &  \frac{ \mathbb{E}_W[\mathcal{A}^*(W)]   + \mathbb{E}_W[  \sqrt{ |\mathcal{A}^*(W)| }] \rho_i  [ 2 \mathrm{Reg}(\pi) + c_0 \sqrt{ \mathbb{E}_W  [  \mathcal{A}^*(W)  ]  } /\rho_i      ]       }{5 \sqrt{ \mathbb{E}_W[\mathcal{A}^*(W)]}    \rho_m} ,  \text{  (Jensen's inequality) }   \\
    \leq  &  \frac{ \mathbb{E}_W[\mathcal{A}^*(W)]   +  \sqrt{\mathbb{E}_W[  |\mathcal{A}^*(W)|  ]} \rho_i  [ 2 \mathrm{Reg}(\pi) + c_0 \sqrt{ \mathbb{E}_W  [  \mathcal{A}^*(W)  ]  } /\rho_i      ]       }{5 \sqrt{ \mathbb{E}_W[\mathcal{A}^*(W)]}    \rho_m},  \text{  ($\rho_i\leq\rho_m$ for $i\leq m$) }    \\
    \leq & \frac{2}{5} \mathrm{Reg}(\pi) + \frac{1+c_0}{5\rho_m}  \sqrt{\mathbb{E}_W[  |\mathcal{A}^*(W)|  ]} .
  \end{align*}
  We can bound $\frac{ \max_{1\leq m'\leq m-1} \mathbb{E}_W \big[ (p_{m'}(\pi(W)|W))^{-1}  \big] }{5\rho_m \sqrt{  \mathbb{E}_W  [   \mathcal{A}^*(W)  ] } }  $ in the same way.

  Combing all above inequalities yields
  \begin{align*}
    \mathrm{Reg}(\pi) -   \widehat{\mathrm{Reg}}_t(\pi) \leq &  \frac{ 2 (1+c_0) \sqrt{\mathbb{E}_W[\mathcal{A}^*(W)]}  }{5\rho_m} + \frac{4}{5} \widehat{\mathrm{Reg}}_t(\pi) +    \frac{5 \sqrt{  \mathbb{E}_W  [   \mathcal{A}^*(W)  ] }}{8 \rho_m  } \\
    \leq & \widehat{\mathrm{Reg}}_t(\pi) + (\frac{ 2 (1+c_0)}{5} + \frac{5}{8})  \frac{ \sqrt{  \mathbb{E}_W  [   \mathcal{A}^*(W)  ] }}{ \rho_m  } \\
    \leq & \widehat{\mathrm{Reg}}_t(\pi) + c_0  \frac{ \sqrt{  \mathbb{E}_W  [   \mathcal{A}^*(W)  ] }}{ \rho_m  }.
  \end{align*}

  Similarly, we have
  \begin{align*}
    &  \widehat{\mathrm{Reg}}_t(\pi) -\mathrm{Reg}(\pi) \\
    =  &     [ \widehat{ \mathcal{R}  }_t (\pi_{\widehat{f}_m}  ) - \widehat{  \mathcal{R}}_t(\pi )   ] -  [\mathcal{R}(\pi_{f^*}  ) -\mathcal{R}(\pi )] \\
    \leq   &     [ \widehat{ \mathcal{R}  }_t (\pi_{\widehat{f}_m}  ) - \widehat{  \mathcal{R}}_t(\pi )   ] - [\mathcal{R}(\pi_{\widehat{f}_m}  ) -\mathcal{R}(\pi )]\\
    \leq  & | \mathcal{R}(\pi_{\widehat{f}_m}  ) - \widehat{ \mathcal{R}  }_t (\pi_{\widehat{f}_m}  )| +| \mathcal{R}(\pi ) -  \widehat{  \mathcal{R}}_t(\pi )|.
  \end{align*}
  We can bound the above terms in the same steps.
\end{proof}

We are now ready to prove the main theorem of this section, which provides an upper bound for the cumulative regret of the algorithm \cref{alg: TL-function approximation}.
\begin{proof}{Proof of \cref{thm: upper bound for TL with function approximation}.}
  Our regret analysis builds on the framework in \cite{fasterCB}.

  \textbf{Step 1:} proving an implicit optimization problem for $Q_m$ in \cref{lem: IOP}.

  \textbf{Step 2:} bounding the prediction error between $ \widehat{\mathcal{R}}_t(\pi)$ and $\mathcal{R}_t(\pi)$ in \cref{lem: prediction error}.
  Then we can show that the one-step regrets $\widehat{\mathrm{Reg}}_t(\pi)$ and $ {Reg}(\pi) $ are close to each other.

  \textbf{Step 3:} bounding the cumulative regret $\mathrm{Reg}(T)$.

  By Lemma 4 of \cite{fasterCB},
  $$
  \mean{\mathrm{Reg}(T)} = \sum_{t=1}^T \sum_{ \pi \in \Psi} Q_{m(t)}(\pi) \mathrm{Reg}(\pi).
  $$
  From \cref{lem: prediction error of regrets}, we know
  $$
  \mathrm{Reg}(\pi)\leq 2  \widehat{\mathrm{Reg}}_t(\pi)   + c_0  \sqrt{  \mathbb{E}_W  [   \mathcal{A}^*(W)  ] } /\rho_m
  $$
  so
  \begin{align*}
    \mean{\mathrm{Reg}(T)} = & \sum_{t=1}^T \sum_{ \pi \in \Psi} Q_{m(t)}(\pi) \mathrm{Reg}(\pi) \\
    \leq & 2\sum_{t=1}^T \sum_{ \pi \in \Psi} Q_{m(t)}(\pi) \widehat{\mathrm{Reg}}_t(\pi) + \sum_{t=1}^T c_0  \sqrt{  \mathbb{E}_W  [   \mathcal{A}^*(W)  ] } /\rho_{m(t)} \\
    \leq & (2+c_0)\sqrt{  \mathbb{E}_W  [   \mathcal{A}^*(W)  ] } \sum_{t=1}^T \frac{1}{ \rho_{m(t)} }\\
    \leq & (2+c_0)\sqrt{  \mathbb{E}_W  [   \mathcal{A}^*(W)  ] }  \sum_{m=1}^{ \left\lceil \log T\right\rceil  }  \sqrt{  \log(2\delta^{-1} |\mathcal{F}^*| \log T  ) \tau_{m-1} /\eta } \\
    \leq & (2+c_0)\sqrt{  \mathbb{E}_W  [   \mathcal{A}^*(W)  ] }  \sum_{m=1}^{ \left\lceil \log T\right\rceil  }  \sqrt{  \log(2\delta^{-1} |\mathcal{F}^*| \log T  ) \tau_{m-1} /\eta  }  \\
    \leq & (2+c_0)    \sqrt{  \mathbb{E}_W  [   \mathcal{A}^*(W)  ] \log(2\delta^{-1} |\mathcal{F}^*| \log T  ) \sum_{m=1}^{ \left\lceil \log T\right\rceil  } \tau_{m-1} /\eta  }  \\
    \leq & (2+c_0)    \sqrt{  \mathbb{E}_W  [   \mathcal{A}^*(W)  ] \log(2\delta^{-1} |\mathcal{F}^*| \log T  ) T /\eta  }.
  \end{align*}
\end{proof}

\subsection{Proof of Theorem \ref{thm: lower bound for TL with function approximation}}

\begin{proof}
  We first consider $|\mathcal{W}|<\infty$.
  Since the agent have knowledge about causal bound,
  any function in $\mathcal{F}-\mathcal{F}^*$ can not be the true reward function.
  For any given context $w$, the set that the optimal arm will be in is $\mathcal{A}^*(w)$.
  For any algorithm $\mathsf{A} $, let $\mathsf{A}_w $ be the induced algorithm of $\mathsf{A} $ when $w$ occurs.
  Namely, the agent has access to a function space $\mathcal{F}_w = \{ f(w,\cdot)| \forall f \in \mathcal{F}^*  \}$ and an action set $\mathcal{A}^*(w)$.

  From the minimax theorem 5.1 in \cite{CBwithPredictableRewards}, we know that
  there exists a contextual bandit instance such that the regret of $\mathsf{A}_w $ is at least
  $
  \sqrt{ {\mathcal{A}^*}(w)T_w \log |\mathcal{F}_w|}  =  \sqrt{ {\mathcal{A}^*}(w)T_w \log |\mathcal{F}^*|}  ,
  $
  where $T_w$ is the number of occurrence of $w$. Hence,
  \begin{align*}
    \mathrm{Reg}(T) & \geq \sum_{w\in\mathcal{W}}   \sqrt{| {\mathcal{A}^*}(w)|T_w \log |\mathcal{F}^*|}  \geq         \sqrt{  \sum_{w\in\mathcal{W}} | {\mathcal{A}^*}(w)|T_w \log |\mathcal{F}^*|}.
  \end{align*}
  and thus
  \begin{align*}
    \limsup_{T\to\infty}   \frac{\mathrm{Reg}(T)}{\sqrt{T}} & \geq   \sqrt{  \sum_{w\in\mathcal{W}} | {\mathcal{A}^*}(w)|\log |\mathcal{F}^*| \cdot \limsup_{T\to\infty}   \frac{T_w}{T}} \\
    & =     \sqrt{  \sum_{w\in\mathcal{W}} | {\mathcal{A}^*}(w)|\log |\mathcal{F}^*| \Prob{W=w} }   \\
    & =    \sqrt{  \mathbb{E}_W[ | {\mathcal{A}^*}(W)| ] \log |\mathcal{F}^*|}  .
  \end{align*}

  Now assume $|\mathcal{W}|=\infty$. Thanks to Glivenko-Cantelli theorem,
  the empirical distribution converges uniformly to the true reward distribution.
  We conclude the proof by applying the dominated convergence theorem and the Fubini's theorem,
  because ${\mathcal{A}^*}(w)$ is uniformly bounded by $|\mathcal{A}|$.
\end{proof}

\subsection{Proof of Proposition \ref{prop: condition to compute Aw}}
\begin{proof}
  Due to \cref{asp: realizability}, the function $f^*$ exists in $\mathcal{F}^*$ so $\mathcal{F}^*$ is not empty.

  We first show that $\mathcal{F}^*$ is a closed subset of $\mathcal{F}$.
  For each fixed $(a, w) \in \mathcal{A} \times \mathcal{W}$, define the set:
  $$
  C_{a,w} = \{ f \in \mathcal{F} \mid l(a, w) \leq f(a, w) \leq h(a, w) \}.
  $$

  Since the interval $[l(a, w), h(a, w)] \subset \mathbb{R}$ is closed and the evaluation map:
  $$
  \operatorname{ev}_{a,w}: \mathcal{F} \to \mathbb{R}, \quad f \mapsto f(a, w)
  $$
  is continuous (under pointwise convergence), the preimage $\operatorname{ev}_{a,w}^{-1}([l(a, w), h(a, w)]) = C_{a,w}$ is closed in $\mathcal{F}$.
  Therefore, $\mathcal{F}^* = \bigcap_{a \in \mathcal{A}} \bigcap_{w \in \mathcal{W}} C_{a,w} $ is closed.
  As a closed subset of the compact set $\mathcal{F}$, $\mathcal{F}^*$ is compact.

  Then we prove the equivalence of the two conditions.

  Assume $a \in \mathcal{A}^*(w)$, so there exists a function $f_0 \in \mathcal{F}^*$ such that
  $
  a = \argmax_{i \in \mathcal{A}} f_0(i, w).
  $
  This implies:
  $
  f_0(a, w) \geq f_0(i, w), \forall i \in \mathcal{A},
  $
  which can be rewritten as:
  $$
  f_0(a, w) \geq \max_{i \neq a} f_0(i, w).
  $$
  Since $f_0 \in \mathcal{F}^*$, the maximum over $\mathcal{F}^*$ satisfies:
  $$
  \max_{f \in \mathcal{F}^*} \left( f(a, w) - \max_{i \neq a} f(i, w) \right) \geq f_0(a, w) - \max_{i \neq a} f_0(i, w) \geq 0.
  $$
  Thus, the condition holds.

  Assume the condition holds:
  $$
  \max_{f \in \mathcal{F}^*} \left( f(a, w) - \max_{i \neq a} f(i, w) \right) \geq 0.
  $$
  By compactness of $\mathcal{F}^*$, this maximum is attained. Thus, there exists a function $f_0 \in \mathcal{F}^*$ such that:
  $$
  f_0(a, w) - \max_{i \neq a} f_0(i, w) \geq 0,
  $$
  which implies:
  $
  f_0(a, w) \geq \max_{i \neq a} f_0(i, w).
  $
  Thus, $a$ is a maximizer of $f_0(\cdot, w)$ and
  $
  a \in \mathcal{A}^*(w).
  $
\end{proof}

\section{Discussions}
\label{sec in appendix: discussion}

\subsection{Discrete sample space}
\label{subsec in appendix: discrete sample space}
We assume that $A$, $Y$, $W$, and $U$ are discrete random variables with finite supports, under which we can represent the set of compatible causal models $\mathfrak{C}$ as a convex polytope in the joint distribution space of the endogenous variables $(A,Y,W,U)$.
\begin{assumption}\label{assu: discrete}
  The variables $A$, $Y$, $W$, and $U$ each take values in finite sets of cardinalities $n_{\mathcal{A}}$, $n_{\mathcal{Y}}$, $n_{\mathcal{W}}$, and $n_{\mathcal{U}}$, respectively.
\end{assumption}
For discrete random variables, the reference measure $\nu$ is the counting measure.

We now explore the implications of this assumption on the sampling process.
Suppose that the feasible density is constrained by marginal density of $(A,Y,W)$ and $U$.
When $\Omega$ is finite, such constrains are subspace constraints, i.e.,
$$
\ind{\bm x \in  \{a_i\} \times \{ y_j\} \times \{ w_k\} \times \mathcal{U}} \text{ for } i=1,\cdots,n_{\mathcal{A}}; j=1,\cdots,n_{\mathcal{Y}}; k=1,\cdots,n_{\mathcal{W}}
$$
with coefficients $\beta_{ijk} = \mathbb{P}(A=a_i,Y=y_j,W=w_k)$.
For mariginal density of $U$, the constraints are analogous.
Without loss of generality, we can assume that the consistent condition $\sum_l \beta_l =  \sum_{i,j,k} \beta_{ijk}=1$.

Let $\bm p = \mathbb{P}(A=a_i,Y=y_j,W=w_k,U=u_l)$ denote the vector of probability mass functions.
We consider the causal effect $\mathbb{E}[Y|\mathrm{do}(a)]$.
To ensure the continuity of this quantity, we require that the vector $\bm p$ has a unform positive lower bound, i.e.,
$
p_{ijkl} \geq \kappa > 0.
$
Note that the number of variables in the feasible set is $n = n_{\mathcal{A}} n_{\mathcal{Y}} n_{\mathcal{W}} n_{\mathcal{U}}$.
To avoid empty sets, we also require that $n\kappa<1$.

Therefore, the optimization problem for the causal upper bound of $\mathbb{E}[Y|\mathrm{do}(a_i)]$ can be formulated as follows:
\begin{align*}
  \max_{\bm p} & \quad \sum_{j,k,l}\frac{ y_j p_{ijkl}\sum_{i',j'} p_{i'j'kl}}{  \sum_{j'} p_{ij'kl}} \\
  \text{s.t.} & \quad \sum_{i,j,k} p_{ijkl} = \beta_{l}, \quad l=1,\cdots,n_{\mathcal{U}} \\
  & \quad \sum_{l} p_{ijkl} = \beta_{i,j,k}, \quad i=1,\cdots,n_{\mathcal{A}}, \quad j=1,\cdots,n_{\mathcal{Y}},\quad k=1,\cdots,n_{\mathcal{W}} \\
  & \quad p_{ijkl} \geq \kappa.
\end{align*}

\subsection{Accelerating the sampling process for discrete random variables}
\label{app: problem-specific acceleration}
\subsubsection{More parallelization tricks}

Suppose $\bm d =(d_{ijkl})$ be a random direction in the null space of $\mathscr{A}$, i.e., $\mathscr{A} \bm d = 0$.
Then the range for step size $\lambda$ is determined by the inequality constraints:
$$
\lambda^{\min} = \max\limits_{d_{ijkl}>0} \dfrac{\kappa - p_{ijkl}}{d_{ijkl}}, \quad
\lambda^{\max} = \min\limits_{d_{ijkl}<0} \dfrac{\kappa - p_{ijkl}}{d_{ijkl}}
$$
Computing the feasible range for $\lambda$ requires $\mathcal{O}(n)$ work serially but can be reduced to $\mathcal{O}(\log n)$ using parallel prefix scans over the $n$ coordinates with $\mathcal{O}(n)$ processors \citep{find_max}.

\subsubsection{Dimensionality reduction.}
Since the mixing time of the Markov chain scales as $\mathcal{O}(n^2)$,
it is desirable to reduce the number of variables and constraints involved to accelerate the sampling process.

The projection matrix need only be calculated once and the main computation burden is applying the projection.
Direct projection of directions into the null space of $ \mathscr{A} $ can be further optimized by exploiting the special structure of $ \mathscr{A}$.
Define the $n\times (n-1)$ matrix $ H_n $ as:
$$
H_n = \left[
  \begin{array}{cccc}
    \frac{1}{\sqrt{2}} &  \frac{1}{\sqrt{6}} & \cdots  & \frac{1}{\sqrt{n(n-1)}} \\
    -\frac{1}{\sqrt{2}} &  \frac{1}{\sqrt{6}} &  \cdots  &  \frac{1}{\sqrt{n(n-1)}} \\
    0 &  -\frac{2}{\sqrt{6}} &  \cdots &  \frac{1}{\sqrt{n(n-1)}}\\
    \vdots &  \vdots &   &  \vdots\\
    0 &  0  & \cdots  &   \frac{-(n-1)}{\sqrt{n(n-1)}}\\
  \end{array}
\right].
$$
The column vectors of $ H_n $ form an orthogonal set of unit vectors of $\mathbb{R}^n$.

\begin{proposition}
  \label{prop: matrix A}
  Suppose the unknown variable $ p_{ijkl} $ is vectorized into a column vector following the index order $ i, j, k, l $.
  Then the coefficient matrix $ \mathscr{A}$ corresponding to the constraints satisfies the following properties:
  \begin{enumerate}
    \item The matrix $ \mathscr{A}$ can be obtained by deleting any row of the following matrix:
      $$
      \begin{pmatrix} I_{n_{\mathcal{A}} n_{\mathcal{Y}} n_{\mathcal{W}}} \otimes \bm{1}_{n_{\mathcal{U}}  }^\top \\ \bm{1}_{n_{\mathcal{A}} n_{\mathcal{Y}} n_{\mathcal{W}}}^\top \otimes I_{n_{\mathcal{U}} }
      \end{pmatrix}
      $$
    \item $dim(\ker (\mathscr{A})) = (n_{\mathcal{A}} n_{\mathcal{Y}} n_{\mathcal{W}}-1)(n_{\mathcal{U}} -1)$
    \item The column vectors of $Q$:
      $$
      Q = H_{ n_{\mathcal{A}} n_{\mathcal{Y}} n_{\mathcal{W}} } \otimes H_{ n_{\mathcal{U}}  }.
      $$
      form an orthonormal basis of $\ker (\mathscr{A})$.
      That is, for each column $\bm{q}_r$ of $Q$, we have:
      $\mathscr{A} q_r=0$ for $r=1,\cdots, (n_{\mathcal{A}} n_{\mathcal{Y}} n_{\mathcal{W}}-1)(n_{\mathcal{U}} -1)$.
    \item Let $z_r,r=1,\cdots, (n_{\mathcal{A}} n_{\mathcal{Y}} n_{\mathcal{W}}-1)(n_{\mathcal{U}} -1)$ be i.i.d. standard Gaussian variables.
      Then the random vector
      $$
      \bm{d} = \sum_{r=1}^{ (n_{\mathcal{A}} n_{\mathcal{Y}} n_{\mathcal{W}}-1)(n_{\mathcal{U}} -1) } z_r  \bm{q}_r,
      $$
      is isotropic in the null space of $\mathscr{A}$, i.e., it is rotationally invariant with covariance equal to the projection matrix onto ${\ker }(\mathscr{A})$.
  \end{enumerate}
\end{proposition}

Based on the construction of $\mathscr{A}$ in \cref{prop: matrix A},
the corresponding parameter vector $\bm{\beta}$ satisfying
$$
\beta_n =
\begin{cases}
  \beta_{i,j,k}, & \text{if } 1 \leq n \leq n_{\mathcal{A}} \times n_{\mathcal{Y}} \times n_{\mathcal{W}}, \\
  \beta_{l}, & \text{if } n > n_{\mathcal{A}} \times n_{\mathcal{Y}} \times n_{\mathcal{W}},
\end{cases}
$$
where the indices $i, j, k$ or $l$ are uniquely determined from $n$ using the following formula:
$$
i = \left\lfloor \frac{n - 1}{n_{\mathcal{Y}} \times n_{\mathcal{W}}} \right\rfloor + 1, \;
j = \left\lfloor \frac{(n - 1) \bmod (n_{\mathcal{Y}} \times n_{\mathcal{W}})}{n_{\mathcal{W}}} \right\rfloor + 1, \;
k = ((n - 1) \bmod (n_{\mathcal{Y}} \times n_{\mathcal{W}})) \bmod n_{\mathcal{W}} + 1,
$$
for $1 \leq n \leq n_{\mathcal{A}} \times n_{\mathcal{Y}} \times n_{\mathcal{W}}$,
and
$$
l = n - n_{\mathcal{A}} \times n_{\mathcal{Y}} \times n_{\mathcal{W}}.
$$
for $ n_{\mathcal{A}} \times n_{\mathcal{Y}} \times n_{\mathcal{W}} + 1\leq n\leq n_{\mathcal{A}} \times n_{\mathcal{Y}} \times n_{\mathcal{W}} + n_{\mathcal{U}}.$

\begin{proof}
  Firstly, we consider the construction of $\mathscr{A}$ corresponding the following constraints:
  \begin{enumerate}
    \item For all $ i, j, k, l $, $\sum_{l'} p_{ijkl'} = \beta_{ijk}$.
    \item For all $ l $, $\sum_{i',j',k'} p_{i'j'k'l} = \beta_l$.
  \end{enumerate}

  For each fixed $ i, j, k $, the corresponding equation is $\sum_{l'} p_{ijkl'} = \beta_{ijk}$.
  In matrix $ \mathscr{A}$, each such equation corresponds to a row where all column positions corresponding to $ p_{ijkl'} $ (i.e., fixed $ i, j, k $ and varying $ l' $) are set to 1, and the rest are 0.
  This part consists of $ n_{\mathcal{A}} \times n_{\mathcal{Y}} \times n_{\mathcal{W}} $ rows, each with $ n_{\mathcal{U}} $ entries of 1.
  For each fixed $ l $, the corresponding equation is $\sum_{i',j',k'} p_{i'j'k'l} = \beta_l$.
  In matrix $ \mathscr{A}$, each such equation corresponds to a row where all column positions corresponding to $ p_{i'j'k'l} $ (i.e., fixed $ l $ and varying $ i', j', k' $) are set to 1, and the rest are 0.
  This part consists of $ n_{\mathcal{U}} $ rows, each with $ n_{\mathcal{A}} \times n_{\mathcal{Y}} \times n_{\mathcal{W}}  $ entries of 1.

  Matrix $ \mathscr{A}$ is formed by vertically stacking two parts: the upper part handles the first constraint, and the lower part handles the second constraint.
  The upper part is an $ (n_{\mathcal{A}} n_{\mathcal{Y}} n_{\mathcal{W}}) \times (n_{\mathcal{A}}   n_{\mathcal{Y}}   n_{\mathcal{W}}   n_{\mathcal{U}}) $ matrix, where each row corresponds to fixed $ i, j, k $ and all $ l' $.
  The lower part is an $ n_{\mathcal{U}} \times  (n_{\mathcal{A}} n_{\mathcal{Y}} n_{\mathcal{W}}  n_{\mathcal{U}}) $ matrix, where each row corresponds to fixed $ l $ and all $ i', j', k' $.

  Using the Kronecker product and unit vectors for formal representation, the coefficient matrix can be expressed as:
  $$
  \begin{pmatrix}
    \bigoplus_{i,j,k} (e_i^\top \otimes e_j^\top \otimes e_k^\top \otimes \bm{1}_{ n_{\mathcal{U}}}^\top) \\
    \bigoplus_{l} (\bm{1}_{n_{\mathcal{A}} }^\top \otimes \bm{1}_{n_{\mathcal{Y}} }^\top \otimes \bm{1}_{n_{\mathcal{W}} }^\top \otimes e_l^\top)
  \end{pmatrix}
  $$
  where $ e_i $ is the $ i $-th standard unit vector; $ \bm{1} $ is the all-ones row vector; $ \otimes $ denotes the Kronecker product; $ \bigoplus $ denotes row-wise stacking.
  Hence, we can rewrite the coefficient matrix as
  $$
  \begin{pmatrix} I_{n_{\mathcal{A}} n_{\mathcal{Y}} n_{\mathcal{W}}} \otimes \bm{1}_{n_{\mathcal{U}}  }^\top \\ \bm{1}_{n_{\mathcal{A}} n_{\mathcal{Y}} n_{\mathcal{W}}}^\top \otimes I_{n_{\mathcal{U}} }
  \end{pmatrix}
  $$

  Secondly, we compute the dimension of $\ker (\mathscr{A})$.
  The variable $ p_{ijkl} $ has a total number of components $n_{\mathcal{A}} n_{\mathcal{Y}} n_{\mathcal{W}} n_{\mathcal{U}} $.
  The constraints for $U$ leads to $n_{\mathcal{A}} n_{\mathcal{Y}} n_{\mathcal{W}}  $ equations,
  and constraints for $A,Y,W$ results in $ n_{\mathcal{U}} $ equations.

  If we sum all first-type constraints, the result is
  $\sum_{i,j,k,l} p_{ijkl} = 1$.
  Similarly, summing all second-type constraints gives the same result.
  Therefore, there is 1 redundant equation among all constraints, leading to
  $ n_{\mathcal{A}} n_{\mathcal{Y}} n_{\mathcal{W}}+ n_{\mathcal{U}}-1$
  independent constraints.
  Hence, the rank of $\mathscr{A}$ is $ n_{\mathcal{A}} n_{\mathcal{Y}} n_{\mathcal{W}}+ n_{\mathcal{U}}-1$ .

  By the rank-nullity theorem:
  $$
  \dim(\ker (\mathscr{A})) =  n_{\mathcal{A}} n_{\mathcal{Y}} n_{\mathcal{W}} n_{\mathcal{U}} - \text{rank}(\mathscr{A}),
  $$
  which completes the proof.

  Thirdly, we construct the orthogonal vectors in $\ker (\mathscr{A})$.
  We directly compute
  \begin{align*}
    Q^\top Q =& ( H_{ n_{\mathcal{A}} n_{\mathcal{Y}} n_{\mathcal{W}} } \otimes H_{ n_{\mathcal{U}}  }  )^\top (H_{ n_{\mathcal{A}} n_{\mathcal{Y}} n_{\mathcal{W}} } \otimes H_{ n_{\mathcal{U}}  } ) \\
    =& ( H_{ n_{\mathcal{A}} n_{\mathcal{Y}} n_{\mathcal{W}} }^\top \otimes H_{ n_{\mathcal{U}}  }^\top  ) (H_{ n_{\mathcal{A}} n_{\mathcal{Y}} n_{\mathcal{W}} } \otimes H_{ n_{\mathcal{U}}  } ) \\
    =& (H_{ n_{\mathcal{A}} n_{\mathcal{Y}} n_{\mathcal{W}} }^\top  H_{ n_{\mathcal{A}} n_{\mathcal{Y}} n_{\mathcal{W}} }   ) \otimes ( H_{ n_{\mathcal{U}}  }^\top  H_{ n_{\mathcal{U}}  }  ) \\
    =& I_{ n_{\mathcal{A}} n_{\mathcal{Y}} n_{\mathcal{W}} -1}  \otimes   I_{ n_{\mathcal{U}} -1 }   \\
    =& I_{ (n_{\mathcal{A}} n_{\mathcal{Y}} n_{\mathcal{W}}-1)( n_{\mathcal{U}} -1) } .
  \end{align*}
  This indicates that the column vectors consist of unit orthogonal vectors.

  It is easy to check the column vector of $H_n$ is sum up to $0$.
  Since $Q$ is equal to the tensor product of such two types of matrices,
  then the column vector of $Q$ also satisfies the required property.
  In other words, $\mathscr{A}\bm{q}_r=0$.

  Fourthly, we express the feasible directions in $\ker (\mathscr{A})$ in terms of the previous basis vector $\bm{q}_r$.

  Since $\bm{d} $ is a linear combination of the basis vectors $ \{ \bm{q}_r \} $ with i.i.d. Gaussian coefficients,
  we can verify isotropy by directly computing its covariance matrix.

  Substituting $ \bm{d} = \sum_r z_r \bm{q}_r $, we get:
  $$
  \Sigma = \mathbb{E}\left[ \left( \sum_r z_r \bm{q}_r \right) \left( \sum_s z_s \bm{q}_s^T \right) \right] = \sum_{r, s} \mathbb{E}[z_r z_s] \bm{q}_r \bm{q}_s^\top = \sum_r \bm{q}_r \bm{q}_r^\top,
  $$
  as $ z_r \sim \mathcal{N}(0,1) $ are independent.
  Since $ \{ \bm{q}_r \} $ form an orthonormal basis of $ \ker (\mathscr{A}) $, we denote
  $$
  \sum_r \bm{q}_r \bm{q}_r^\top = P.
  $$
  Hence, $ P $ is the orthogonal projection matrix onto $ \ker (\mathscr{A}) $.

  Next, we need to verify rotational invariance.
  Let $ R $ be any orthogonal matrix acting on $ \ker (\mathscr{A}) $ (i.e., $ R P = P R = R $).
  Then the rotated vector is:
  $
  R \bm{d} = R \left( \sum_r z_r \bm{q}_r \right) = \sum_r z_r (R \bm{q}_r)
  $
  Since $ \{ R \bm{q}_r \} $ still form an orthonormal basis of $ \ker (\mathscr{A}) $, and $ \{ z_r \} $ are i.i.d. standard Gaussians,
  the distribution of $ R \bm{d} $ is identical to that of $ \bm{d} $.
\end{proof}

Based on the previous property, we explicitly construct an orthonormal basis $ \{ \bm{q}_r \} $ for the null space $ \ker (\mathscr{A}) $. Using this basis, a Gaussian random vector $ \bm{d} \in \ker (\mathscr{A}) $ can be expressed as
$$
\bm{d} = \sum_{r=1}^{ (n_{\mathcal{A}} n_{\mathcal{Y}} n_{\mathcal{W}} - 1)(n_{\mathcal{U}} - 1) } z_r \bm{q}_r,
$$
where each $ z_r \sim \mathcal{N}(0, 1) $ is an independent standard Gaussian random variable.

This representation significantly reduces the computational complexity of generating $ \bm{d} $: instead of performing $ (n_{\mathcal{A}} n_{\mathcal{Y}} n_{\mathcal{W}} n_{\mathcal{U}})^2 $ multiplications required for a full projection, we only need $ (n_{\mathcal{A}} n_{\mathcal{Y}} n_{\mathcal{W}} - 1)^2 (n_{\mathcal{U}} - 1)^2 $ multiplications when using the basis form.
Moreover, the resulting vector $ \bm{d} $ is isotropic in the null space and requires only $ (n_{\mathcal{A}} n_{\mathcal{Y}} n_{\mathcal{W}} - 1)(n_{\mathcal{U}} - 1) $ independent random variables.

\subsubsection{Further dimensionality reduction specific to special objectives.}

Though our method can deal with general objectives,
targeting at the special structures of optimization problems can simplify the solving and sampling process.
We consider the target
$$
V(\mathcal{\mathcal{M}}) = \mathbb{P}(Y=y \mid \mathrm{do}(A = a), W= w).
$$
Suppose we know the information $\rho(a,y)$ and $\rho(w,u)$.

From do-calculus, we have
$$
\mathbb{P}(Y=y \mid \mathrm{do}(A = a), W=w) = \sum_u \rho(y|a,w,u)\rho(u|w) = \sum_u \frac{  \rho(a,y,w,u)\rho(u|w) } {\rho(a,w,u)} .
$$
We now fix $\rho(a,w,u)$. The linear constraints for $\rho(a,y,w,u)$ are
\begin{align*}
  \sum_{w,u} \rho(a,y,w,u) = \rho(a,y), \\
  \sum_{a,y} \rho(a,y,w,u) = \rho(w,u),\\
  \sum_y \rho(a,y,w,u) = \rho(a,w,u), \\
  0\leq \rho(a,y,w,u)\leq 1.
\end{align*}
For a fix $\rho(a,w,u)$, this is a linear program for $\rho(a,y,w,u)$, which can be solved efficiently.

Therefore, we only need to deal with the sample of $\rho(a,w,u)$, which follows the constraints
$$
\sum_{w,u} \rho(a,w,u) = \rho(a), \sum_{a} \rho(a,w,u) = \rho(w,u), 0\leq  \rho(a,w,u) \leq 1.
$$

\subsection{Approximation of infinite-dimensional funtion space}
\label{sec: approximation of infinite function space}
Though our sampling method is designed for finite-dimensional function spaces,
it can be extended to general function spaces with finite-dimensional approximation.
\subsubsection{Parametric function space}
We consider an infinite-dimensional function space with basis expansions:
$$
\mathcal{P} = \left\{ \rho \in L^2(\Omega, \nu) \colon \rho \geq 0 \text{ a.e.},  \abs{\int_{\Omega} \rho(\bm{x}) e_k(\bm{x}) \mathrm{d}\nu(\bm{x}) }\leq \frac{1}{k^2} \text{ for all } k \geq 1 \right\}.
$$
We make the following assumptions on the basis functions $e_k$:
\begin{itemize}
  \item The basis functions $e_k$ are orthonormal, i.e., $\int_{\Omega} e_i(\bm{x}) e_j(\bm{x}) \mathrm{d}\nu(\bm{x}) = \delta_{ij}$.
  \item The basis functions $e_k$ are uniformly bounded, i.e., $\sup_{k} \|e_k\|_\infty \leq M$ with $1\leq M <\frac{1}{\pi^2/6-1}$.
  \item The normalization constraint is satisfied, i.e., $\int_{\Omega} e_1(\bm{x}) \rho(\bm{x}) \mathrm{d}\nu(\bm{x}) = 1$ for $e_1=1$.
\end{itemize}
Under these assumptions, the function space $\mathcal{P}$ is compact and convex.
For each element $\rho \in \mathcal{P}$, we can express it as a linear combination of the basis functions:
$$
\rho = 1 + \sum_{k=2}^{\infty} c_k e_k.
$$
Since
$$
\Biggl|\sum_{k=2}^{\infty} c_k e_k\Biggr| \leq \sum_{k=2}^{\infty} \frac{M}{k^2} = M(\frac{\pi^2}{6} - 1)<1,
$$
the function $\rho$ is lower bounded by $1-M(\frac{\pi^2}{6} - 1) >0$.

Consider the finite-dimensional approximation of $\mathcal{P}$:
$$
\mathcal{P}_n = \left\{ \rho \in \mathcal{P} \colon  \int_{\Omega} \rho(\bm{x}) e_k(\bm{x}) \mathrm{d}\nu(\bm{x}) =0 \text{ for all } k \geq n+1 \right\}.
$$
Since the nonnegativity is satisfied in $\mathcal{P}$, we do not need to deal with this constraint explicitly.

The convergence justifies the use of $\mathcal{P}_n$ for sampling, as $\mathcal{P}_n$ approximates $\mathcal{P}$ well for large $n$.
We can select the campact set $\mathcal{K} = \mathcal{P}$ as required by \cref{assu: continuous}.
One then apply \cref{alg: hit-run for density} to sample from $\mathcal{P}_n$ to solve the non-convex optimization problem in an infinite-dimensional space.

\begin{proposition}
  The Hausdorff distance between $\mathcal{P}$ and $\mathcal{P}_n$ in $L^2(\Omega,\nu)$ is given by:
  $$
  d_H(\mathcal{P}, \mathcal{P}_n) = \left( \sum_{k=n+1}^\infty \frac{1}{k^4} \right)^{1/2}.
  $$
  Moreover, this distance satisfies the asymptotic bounds:
  $$
  \frac{1}{\sqrt{3} (n+1)^{3/2}} \leq d_H(\mathcal{P}, \mathcal{P}_n) \leq \frac{1}{\sqrt{3} n^{3/2}} \quad \text{for all } n \geq 1,
  $$
  and thus $d_H(\mathcal{P}, \mathcal{P}_n) \to 0$ as $n \to \infty$ with rate $\mathcal{O}(n^{-3/2})$.

\end{proposition}

\begin{proof}
  Since $\mathcal{P}_n \subseteq \mathcal{P}$, the second term vanishes:
  $$
  \sup_{\rho_n \in \mathcal{P}_n} \inf_{\rho \in \mathcal{P}} \|\rho_n - \rho\|_{L^2(\Omega,\nu)} = 0,
  $$
  because for any $\rho_n \in \mathcal{P}_n$, we can choose $\rho = \rho_n \in \mathcal{P}$.

  For any $\rho \in \mathcal{P}$, expand it in the orthonormal basis:
  $$
  \rho = 1 + \sum_{k=2}^\infty c_k e_k, \quad \text{where} \quad c_k = \int_\Omega \rho(\bm{x}) e_k(\bm{x})  \mathrm{d}\nu(\bm{x}).
  $$
  By definition of $\mathcal{P}$, $|c_k| \leq k^{-2}$ for all $k \geq 2$. Define the truncation:
  $$
  \tilde{\rho}_n = 1 + \sum_{k=2}^n c_k e_k.
  $$
  Since $\mathcal{P}_n$ is a subset of $\mathcal{P}$ with truncated basis expansions,
  we have $\tilde{\rho}_n \in \mathcal{P}_n$.

  The approximation error is:
  $$
  \rho - \tilde{\rho}_n = \sum_{k=n+1}^\infty c_k e_k.
  $$
  By orthonormality, the $L^2(\Omega,\nu)$-norm is
  $
  \|\rho - \tilde{\rho}_n\|_{L^2(\Omega,\nu)}^2 =  \sum_{k=n+1}^\infty |c_k|^2,
  $
  As $|c_k| \leq k^{-2}$,
  $
  \|\rho - \tilde{\rho}_n\|_{L^2(\Omega,\nu)}^2 \leq \sum_{k=n+1}^\infty (k^{-2})^2 = \sum_{k=n+1}^\infty k^{-4},
  $
  and thus
  $$
  \inf_{\rho_n \in \mathcal{P}_n} \|\rho - \rho_n\|_{L^2(\Omega,\nu)} \leq \|\rho - \tilde{\rho}_n\|_{L^2(\Omega,\nu)} \leq \left( \sum_{k=n+1}^\infty k^{-4} \right)^{1/2}.
  $$
  This holds for all $\rho \in \mathcal{P}$, so:
  $$
  \sup_{\rho \in \mathcal{P}} \inf_{\rho_n \in \mathcal{P}_n} \|\rho - \rho_n\|_{L^2(\Omega,\nu)} \leq \left( \sum_{k=n+1}^\infty k^{-4} \right)^{1/2}.
  $$

  Consider the specific element $\rho^* \in \mathcal{P}$ defined by:
  $$
  \rho^* = 1 + \sum_{k=2}^\infty k^{-2} e_k.
  $$
  For any $\rho_n \in \mathcal{P}_n$, write $\rho_n = 1 + \sum_{k=2}^n d_k e_k$ with $|d_k| \leq k^{-2}$. Then:
  $$
  \rho^* - \rho_n = \sum_{k=2}^n (k^{-2} - d_k) e_k + \sum_{k=n+1}^\infty k^{-2} e_k.
  $$
  Therefore, we have
  $$
  \|\rho^* - \rho_n\|_{L^2(\Omega,\nu)}^2 = \sum_{k=2}^n |k^{-2} - d_k|^2 + \sum_{k=n+1}^\infty |k^{-2}|^2 \geq \sum_{k=n+1}^\infty k^{-4},
  $$
  since $|k^{-2} - d_k|^2 \geq 0$. Equality holds when $d_k = k^{-2}$ for $k = 2, \dots, n$, which is achievable because $\rho_n = 1 + \sum_{k=2}^n k^{-2} e_k \in \mathcal{P}_n$. Thus,
  $$
  \inf_{\rho_n \in \mathcal{P}_n} \|\rho^* - \rho_n\|_{L^2(\Omega,\nu)} = \left( \sum_{k=n+1}^\infty k^{-4} \right)^{1/2}.
  $$
  Therefore,
  $$
  \sup_{\rho \in \mathcal{P}} \inf_{\rho_n \in \mathcal{P}_n} \|\rho - \rho_n\|_{L^2(\Omega,\nu)} \geq \inf_{\rho_n \in \mathcal{P}_n} \|\rho^* - \rho_n\|_{L^2(\Omega,\nu)} = \left( \sum_{k=n+1}^\infty k^{-4} \right)^{1/2}.
  $$

  Combining both parts yields
  $$
  d_H(\mathcal{P}, \mathcal{P}_n) = \left( \sum_{k=n+1}^\infty k^{-4} \right)^{1/2}.
  $$

  To bound the series, use integrals:
  $$
  \int_{n+1}^\infty x^{-4}  dx \leq \sum_{k=n+1}^\infty k^{-4} \leq \int_n^\infty x^{-4}  dx.
  $$
  Therefore,
  $$
  \frac{1}{3(n+1)^3} \leq \sum_{k=n+1}^\infty k^{-4} \leq \frac{1}{3n^3},
  $$
  and taking square roots:
  $$
  \frac{1}{\sqrt{3} (n+1)^{3/2}} \leq d_H(\mathcal{P}, \mathcal{P}_n) \leq \frac{1}{\sqrt{3} n^{3/2}}.
  $$
  This implies $d_H(\mathcal{P}, \mathcal{P}_n) = \mathcal{O}(n^{-3/2})$ as $n \to \infty$.
\end{proof}

\subsubsection{Nonparametric function space}
Consider the nonparametric function space defined by the following constraints:
$$
\mathcal{P}= \left\{ \rho \in C(\Omega) \colon \rho \geq \kappa > 0, \int_{\Omega} \rho(\bm{x}) \mathrm{d}\nu(\bm{x}) = 1, \, \abs{\rho(\bm x) - \rho(\bm y)}  \leq L \norm{\bm x - \bm y }_{\infty}, \forall \bm x,\bm y \in \Omega  \right\}.
$$
To briefly illustrate our idea, we assume $\Omega = [0,1]$ and the reference measure is the Lebesgue measure on $[0,1]$.
The function space is rewritten as:
$$
\mathcal{P} = \left\{ \rho \in C([0,1]) : \rho \geq \kappa > 0, \int_{0}^{1} \rho(x)  dx = 1, |\rho(x) - \rho(y)| \leq L |x - y|, \forall x,y \in [0,1] \right\}.
$$
We aim to construct sets $\mathcal{P}_n$ of piecewise linear functions such that the Hausdorff distance $d_H(\mathcal{P}_n, \mathcal{P})$ under the uniform norm converges to 0 as $n \to \infty$, where:
$$
\|f - g\|_{\infty} = \sup_{x \in [0,1]} |f(x) - g(x)|, \quad
d_H(A, B) = \max\left\{ \sup_{a \in A} \inf_{b \in B} \|a - b\|_{\infty}, \sup_{b \in B} \inf_{a \in A} \|b - a\|_{\infty} \right\}.
$$

For each $n \in \mathbb{N}$, partition $[0,1]$ into $n$ equal subintervals with nodes $x_i = i/n$ for $i = 0, 1, \dots, n$. Define $\mathcal{P}_n$ as the set of continuous piecewise linear functions $g$ satisfying:
$$
\mathcal{P}_n = \left\{ \rho \in C([0,1]) : \rho \geq \kappa_n, \int_{0}^{1} \rho(x)  dx = 1, |\rho(x) - \rho(y)| \leq L |x - y|, \forall x,y \in [0,1] \right\}.
$$
The Lipschitz condition implies that the slope in each subinterval $[x_i, x_{i+1}]$ is bounded by $L$.

As required by \cref{assu: continuous}, we define a set $\mathcal{K}$ as follows:
$$
\mathcal{K} = \left\{ \rho \in C([0,1]) : \rho \geq \kappa_{n_0} > 0, \int_{0}^{1} \rho(x)  dx = 1, |\rho(x) - \rho(y)| \leq L |x - y|, \forall x,y \in [0,1] \right\}
$$
for sufficiently large and fixed $n_0$.
The compactness of $\mathcal{K}$ follows from the Arzelà-Ascoli theorem.

\begin{proposition}
  Let $\mathcal{P}$ be defined as above, and $\mathcal{P}_n$ constructed with $\kappa_n = \kappa - 2L/n$. Then:
  $$
  d_H(\mathcal{P}_n, \mathcal{P}) \leq   \frac{C}{n},
  $$
  where the constant $C > 0$ depending only on $\kappa, L$.
\end{proposition}

\begin{proof}
  We prove both directions with explicit constructions.

  \textbf{Part 1: Approximation of $\mathcal{P}$ by $\mathcal{P}_n$.} For any $\rho \in \mathcal{P}$, define $\rho_n^0$ as the piecewise linear interpolant of $\rho$ at nodes $\{x_i\}_{i=0}^n$. Since $\rho$ is $L$-Lipschitz:
  $$
  |\rho(x) - \rho_n^0(x)| \leq L \cdot \frac{1}{n}, \quad \forall x \in [0,1].
  $$
  Let $I_n = \int_0^1 \rho_n^0(x) dx$. Then:
  $$
  |I_n - 1| \leq \|\rho - \rho_n^0\|_{\infty} \leq \frac{L}{n}.
  $$
  Define the adjusted function:
  $$
  \rho_n(x) = \rho_n^0(x) + (1 - I_n).
  $$
  This satisfies $\int_0^1 \rho_n dx = 1$. For the lower bound:
  $$
  \rho_n(x) \geq \kappa - \frac{L}{n} - |1 - I_n| \geq \kappa - \frac{2L}{n} = \kappa_n.
  $$
  The Lipschitz constant is preserved since we add a constant. Thus $\rho_n \in \mathcal{P}_n$. The error is:
  $$
  \|\rho - \rho_n\|_{\infty} \leq \|\rho - \rho_n^0\|_{\infty} + |1 - I_n| \leq \frac{2L}{n}.
  $$

  \textbf{Part 2: Approximation of $\mathcal{P}_n$ by $\mathcal{P}$.} For any $g \in \mathcal{P}_n$, define:
  $
  \rho_1(x) = \max\{ g(x), \kappa \}.
  $
  Without loss of generality, we assume $g(x) < \kappa $ on some interval;
  otherwise, we can take $g \in \mathcal{P}$ directly.
  Then $\rho_1 \geq \kappa$ and $\rho_1$ is $L$-Lipschitz. Let $\delta_n = \int_0^1 (\rho_1 - g) dx > 0$, which satisfies:
  $$
  \delta_n \leq (\kappa - \kappa_n) \cdot \nu(S), \quad S = \{x : g(x) < \kappa\}.
  $$
  Since $g$ is $L$-Lipschitz and $g \geq \kappa_n$, the set $S$ has measure $\nu(S) \leq 2/n$ (because $g$ must rise from below $\kappa$ to $\kappa$ over an interval of length at least $(\kappa - \kappa_n)/L = 2/n$). Thus:
  $$
  \delta_n \leq (\kappa - \kappa_n) \cdot \frac{2}{n} = \frac{4L}{n^2}.
  $$
  Define $\rho_2(x) = \rho_1(x)/(1 + \delta_n)$. This satisfies $\int_0^1 \rho_2 dx = 1$ and:
  $$
  \rho_2(x) \geq \frac{\kappa}{1 + \delta_n} \geq \kappa \left(1 - \frac{4L}{n^2}\right),
  $$
  with Lipschitz constant at most $L$.
  Finally, define $f(x) = \rho_2(x) + \kappa - \min_{y \in [0,1]} \rho_2(y) \geq \rho_2(x)$.
  Moreover $f \geq \kappa$ and its Lipschitz constant is no more than $ L$. Let $I_f = \int_0^1 f dx \geq 1$ and define:
  $$
  \rho(x) = \frac{f(x)}{I_f}.
  $$
  Then $\rho \in \mathcal{P}$. The error accumulates as:
  \begin{align*}
    \|g - \rho\|_{\infty} &\leq \|g - \rho_1\|_{\infty} + \|\rho_1 - \rho_2\|_{\infty} + \|\rho_2 - f\|_{\infty} + \|f - \rho\|_{\infty} \\
    &\leq \frac{2L}{n} + \delta_n \|\rho_1\|_{\infty} + \left(\kappa - \min \rho_2\right) + |I_f - 1| \|f\|_{\infty}.
  \end{align*}
  Bounding each term by $\mathcal{O}(1/n)$, we obtain $\|g - \rho\|_{\infty} \leq C/n$ for some $C > 0$.

  Combining both parts yields the result.
\end{proof}

As a conclusion, we formulate the optimization problem for sampling from the function space $\mathcal{P}_n$ as follows:
$$
\begin{array}{ll}
  \mathop{\min/\max}\limits_{\rho_0,\rho_1,\dots,\rho_n} & V(\rho_0,\rho_1,\dots,\rho_n) \\[10pt]
  \text{s.t.} &
  \frac{1}{2n}(\rho_0 + 2\sum\limits_{i=1}^{n-1}\rho_i + \rho_n) = 1 \\[10pt]
  & |\rho_{i+1} - \rho_i| \leq \dfrac{L}{n}, \quad \forall i = 0,1,\dots,n-1 \\[10pt]
  & \rho_i \geq \kappa_n, \quad \forall i = 0,1,\dots,n.
\end{array}
$$

\begin{remark}
  For higher-dimensional domains like the unit cube,
  the core approximation idea remains similar but requires triangulation.
  Piecewise linear functions are then defined by linearly interpolating values at grid vertices within each triangle,
  and the Lipschitz condition is maintained by bounding function slopes on triangles.
  The approximation quality improves as the grid refines, and this approach preserves the convergence result observed in one dimension.
\end{remark}

\section{More Numerical Experimental Results}
\subsection{Causal Bound Computation}\label{app: benchmark causal bound computation}

\subsubsection*{Numerical Setup}
We present the observational distribution $\mathbb{P}(A, Y, W)$ for the POCB dataset in \cref{tab: observational distribution}.
The variable $U$ is binary with parameter 0.1.
\begin{table}[hbtp]
  \centering
  \small
  \begin{tabular}{c|ccccccccc}
    \toprule
    $(A,Y,W)$      & $(0,0,0)$ & $(0,0,1)$ & $(0,1,0)$ & $(0,1,1)$ & $(1,0,0)$ & $(1,0,1)$ & $(1,1,0)$ & $(1,1,1)$ \\
    \midrule
    $\rho({a,y,w})$ & 0.2328 & 0.1784 & 0.1351 & 0.1467 & 0.0304 & 0.1183 & 0.0149 & 0.1433 \\
    \bottomrule
  \end{tabular}
  \caption{Observational distribution $\mathbb{P}(A, Y, W)$ for POCB.}\label{tab: observational distribution}
\end{table}

\subsubsection*{Benchmark Sampling Methods}
\label{sec in appendix: Benchmark Sampling Method}
To show the efficiency of \cref{alg: hit-run for density},
we brefly describe several benchmark sampling methods and provide additional numerical results to support our claims.
Corresponding notations have been introduced in \cref{subsec in appendix: discrete sample space}.
To make fair comparisons, we use the feasible region defined by the constraints in the following optimization problem:
\begin{align}
  \label{eq: nonlinear programming for causal bounds}
  \begin{split}
    \max / \min \  & \sum_{j,k,l}   \frac{y_j p_{ijkl} \sum_{i',j'} p_{i'j'kl} }{  \sum_{j'} p_{ij'kl} }, \\
    \text{s.t.} \  & \sum_{l} p_{ijkl} = \beta_{ijk}, \quad \forall i, j, k, \\
    & \sum_{i,j,k} p_{ijkl} = \beta_{l}, \quad \forall l, \\
    & 0 \leq p_{ijkl} \leq 1, \quad \forall i, j, k, l.
  \end{split}
\end{align}
Note that previous works have considered similar optimization problems but often yield non-tight causal bounds.
For instance, the solutions from the nonlinear optimization problem in \citep{CEbound} may not correspond to any valid causal model, resulting in non-tight bounds.
We observe that focusing solely on constraints for the specific value $a$ of interest generally leads to looser bounds.
Therefore, to obtain tight bounds, it is essential to incorporate constraints for \textit{all} $a'\in\mathcal{A}$,
rather than just the value $a$ relevant to the intervention $\mathrm{do}(A=a)$.

The main idea is to randomly draw samples from the feasible region of the optimization problem \eqref{eq: nonlinear programming for causal bounds}, which is a polytope defined by $n_{\mathcal{A}} n_{\mathcal{Y}} n_{\mathcal{W}} + n_{\mathcal{U}} - 1$ linearly independent constraints. Given that there are $n_{\mathcal{A}} n_{\mathcal{Y}} n_{\mathcal{W}}n_{\mathcal{U}} $ unknown variables, this setup requires determining the values of $n_{\mathcal{A}} n_{\mathcal{Y}} n_{\mathcal{W}}n_{\mathcal{U}} - n_{\mathcal{A}} n_{\mathcal{Y}} n_{\mathcal{W}} - n_{\mathcal{U}} + 1$ unknowns.

A naive approach would be to sample each $p_{ijkl}$ independently from a uniform distribution supported on $[0,1]$, rejecting any samples that do not meet the constraints.
This approach can be highly sample-inefficient.

To improve efficiency, we can incorporate inequality constraints to narrow the sampling range and increase the likelihood of generating valid samples.
Given that we are essentially considering all possible joint distributions with fixed marginals, \citet{CEbound} used the Fr\'{e}chet inequalities to reduce the search space for $p_{ijkl}$, as follows:
\begin{equation}\label{eq: frechet inequalities}
  \max\left\{0, \beta_{ijk}+\beta_l-1\right\} \leq p_{ijkl} \leq \min\left\{ \beta_{ijk},\beta_l \right\}.
\end{equation}
However, the solutions generated by \citet{CEbound}, which involve sampling each variable from the reduced interval in \eqref{eq: frechet inequalities}, may not satisfy all the constraints in \eqref{eq: nonlinear programming for causal bounds}, leading to a lack of tightness.

One may further improve sample efficiency by solving the following \textit{linear programming problem} to find tight bounds on each $p_{ijkl}$:
\begin{align}
  \label{eq: linear programming to find support}
  \begin{split}
    \max / \min \  & p_{ijkl}, \\
    \text{s.t.} \  & \text{linear constrains in \eqref{eq: nonlinear programming for causal bounds}.}
  \end{split}
\end{align}
While these bounds are tight for each \textit{individual} $p_{ijkl}$, the Cartesian product of these bounds may not be tight for the entire vector of $p_{ijkl}$ values.

To support this claim, we report the proportion of valid samples (i.e., joint distributions that satisfy the constraints) obtained using different sample spaces in \cref{tb: valid sample proportion with different sample spaces} for the example discussed in \cref{sec: numerical}.
We observed that even with the bounds derived from the individual LPs in \eqref{eq: linear programming to find support}, only $0.3\%$ of the samples were valid, leading to a significant loss in sample efficiency.
\begin{table}[hbtp]
  \centering
  \small
  \begin{tabular}{c|c}
    \toprule
    sample space for $p_{ijkl}$           &  proportion of valid samples \\
    \midrule
    $[0,1]$ &       $\approx 0$         \\
    support given by \eqref{eq: frechet inequalities} &    $<10^{-4}$             \\
    support given by \eqref{eq: linear programming to find support} &        $0.3\%$         \\
    \midrule
    \cref{alg: MC-causal model} & $100\%$           \\
    \bottomrule
  \end{tabular}
  \caption{Proportion of valid samples obtained with different sample spaces for the example in \cref{sec: numerical}.}
  \label{tb: valid sample proportion with different sample spaces}
\end{table}
Additionally, this sampling method offers both an intuitive explanation and numerical evidence for why the approach by \citet{CEbound} may fail to yield tight bounds, and how our approach improves upon it.

We further introduce a sampling algorithm based on sequential linear programming to generate valid samples from the feasible region.
Let $S$ denote a set of free variables for the linear equations in \eqref{eq: nonlinear programming for causal bounds}.
The procedure begins by selecting such a set $S$ with cardinality $n_{\mathcal{A}} n_{\mathcal{Y}} n_{\mathcal{W}}n_{\mathcal{U}} - n_{\mathcal{A}} n_{\mathcal{Y}} n_{\mathcal{W}} - n_{\mathcal{U}} + 1$.
We then iteratively sample each variable $p_{n_r}$ with $p_{n_r}\in S$.
For the first variable $p_{n_1}$,
we solve \eqref{eq: linear programming to find support} to determine its support interval $[l_{n_1},h_{n_1}]$,
and then sample a value $\widehat{p}_{n_1}$ from a user-specified distribution truncated to $[l_{n_1},h_{n_1}]$.
At iteration $r$, with the values of $p_{n_1},\cdots,p_{n_{r-1}}$ already sampled, we add constraints to ensure that each of these variables is fixed to its sampled value. Specifically, we find the support $[l_{n_r},h_{n_r}]$ for $p_{n_r}$ by solving
\begin{align}
  \label{eq: sequential linear programming to find support}
  \begin{split}
    \max / \min \  & p_{n_r}, \\
    \text{s.t.} \  & \text{linear constrains in \eqref{eq: nonlinear programming for causal bounds},} \\
    & p_{n_s} = \widehat{p}_{n_s}, \quad \forall s=1,2,\cdots,r-1,
  \end{split}
\end{align}
We then sample $\widehat{p}_{n_r}$ is from the support $[l_{n_r},h_{n_r}]$.
After completing all $\nu(S)$ steps, the remaining $p_{ijkl} \notin S$, can be uniquely determined by solving the equality constraints of \eqref{eq: nonlinear programming for causal bounds}.

Each sample represents a possible joint distribution consistent with the observed marginals.
By sequentially solving linear programs (LPs), this algorithm ensures that each sample respects the imposed constraints, thereby avoiding invalid distributions and significantly improving sample efficiency; see Table \ref{tb: valid sample proportion with different sample spaces} for a comparison with existing methods.

\begin{algorithm}[hbtp]
  \renewcommand{\algorithmicrequire}{\textbf{Input:}}
  \renewcommand{\algorithmicensure}{\textbf{Output:}}
  \caption{Monte-Carlo sampling for compatible causal models using sequential LP}
  \label{alg: MC-causal model}
  \begin{algorithmic}[1]
    \Require Observational distribution $F(a,y,w)$ and $F(u)$ and sampling distribution ${F}_s$
    \State Select a set of free variables $S$ with cardinality $n_{\mathcal{A}} n_{\mathcal{Y}} n_{\mathcal{W}}n_{\mathcal{U}} - n_{\mathcal{A}} n_{\mathcal{Y}} n_{\mathcal{W}} - n_{\mathcal{U}} + 1$
    \State Compute each $\beta_{ijk}$ and $\beta_l$
    \State Sequentially solve LP \eqref{eq: sequential linear programming to find support} to find the support $[l_{ijkl},h_{ijkl}]$ of $p_{ijkl}$ for each $p_{ijkl}\in S$
    \State Sample a value $\widehat{p}_{ijkl}$ from $F_s$ truncated to $[l_{ijkl},h_{ijkl}]$ for each $p_{ijkl}\in S$
    \State Solve the remaining $\widehat{p}_{ijkl}$ using the equality constraints in \eqref{eq: nonlinear programming for causal bounds} for all $p_{ijkl}\notin S$
    \Ensure Joint distribution of the endogenous variables represented by $\widehat{\bm p} \triangleq \{\widehat{p}_{ijkl}: \forall i,j,k,l\}$
  \end{algorithmic}
\end{algorithm}

\begin{theorem}[\citet{CEbound}]
  \label{thm:pearl'bounds}
  Given a causal diagram $G$ and a distribution compatible with $G$, let $W \cup U$ be a set of variables satisfying the back-door criterion in $G$ relative to an ordered pair $(X, Y)$, where $W \cup U$ is partially observable, i.e., only probabilities $\mathbb{P}(X, Y, W)$ and $\mathbb{P}(U)$ are given. The causal effects of $X$ on $Y$ are then bounded as follows:
  \begin{equation}
    \mathrm{LB} \leq \mathbb{P}(Y=y|\mathrm{do}(x)) \leq \mathrm{UB}
  \end{equation}
  where $\mathrm{LB}$ is the solution to the non-linear optimization problem in Equation \ref{eq:lb} and $\mathrm{UB}$ is the solution to the non-linear optimization problem in Equation \ref{eq:ub}.
  \begin{equation}\label{eq:lb}
    \mathrm{LB} = \min\sum_{w,u} \frac{a_{w,u}b_{w,u}}{c_{w,u}},
  \end{equation}
  \begin{equation}\label{eq:ub}
    \mathrm{UB} = \max \sum_{w,u} \frac{a_{w,u}b_{w,u}}{c_{w,u}},
  \end{equation}
  where,
  \begin{align*}
    \sum_{u} a_{w,u} &= \rho(x, y, w), \\
    \sum_{u} b_{w,u} &= \rho(w), \\
    \sum_{u} c_{w,u} &= \rho(x, w) \quad \text{for all } w \in W;
  \end{align*}
  and for all $w \in W$ and $u \in U$,
  \begin{align*}
    b_{w,u} &\geq c_{w,u} \geq a_{w,u}, \\
    \max\{0, \rho(x, y, w) + \rho(u) - 1\} &\leq a_{w,u} \leq \min\{\rho(x, y, w), \rho(u)\}, \\
    \max\{0, \rho(w) + \rho(u) - 1\} &\leq b_{w,u} \leq \min\{\rho(w), \rho(u)\}, \\
    \max\{0, \rho(x, w) + \rho(u) - 1\} &\leq c_{w,u} \leq \min\{\rho(x, w), \rho(u)\}.
  \end{align*}
\end{theorem}

\subsubsection*{Numerical setups for \cref{fig: sampling general}}
\label{subsubsec in appendix: numerical setups for sampling general}

We set $n_{\mathcal{A}}=n_{\mathcal{Y}}=n_{\mathcal{W}}=n_{\mathcal{U}}=2$,
and randomly generate a feasible set of parameters $\beta_{i,j,k}$ and $\beta_l$ for $i=1,2$, $j=1,2$, $k=1,2$, and $l=1,2$.
To solve nonlinear optimization problems, we use SciPy's \texttt{optimize.minimize},
seeding it with multiple starting points drawn from Algorithm \ref{alg: hit-run for density}.
This simple initialization helps the optimizer escape poor local optima.

\subsubsection*{Performance Comparison of solving two optimization problems}

The known parameters for the optimization problem were randomly generated to ensure generalized evaluation.
To initialize the optimization process,
we first sampled 2,000 feasible points uniformly from the solution domain.
From this pool, subsets of 50, 100, 200, and 500 points were randomly selected as starting points for optimization oracles.
The computational overhead of the sampling phase was negligible (contributing $<0.5\%$ to total runtime),
confirming that the initialization method does not materially impact time complexity.
The dominant computational cost is attributed to the optimization algorithms themselves.

Our method outperforms \cite{CEbound} in causal bound optimization,
delivering tighter bounds more efficiently.
Our method achieves order-of-magnitude speedups (e.g., $6.4\times$ faster at 500 points: 12.3s vs. 78.8s)
and our solutions yield strictly narrower and more informative bounds.

\begin{table}[htbp]
  \centering
  \caption{Optimization performance comparison}
  \label{tab:solver_performance}
  \small
  \begin{tabular}{c *{2}{c c c}}
    \toprule
    \multirow{2}{*}{Starting Points} &
    \multicolumn{3}{c}{\citet{CEbound}} &
    \multicolumn{3}{c}{Ours} \\
    \cmidrule(lr){2-4} \cmidrule(l){5-7}
    & Time (s) & Minimum & Maximum &
    Time (s) & Minimum & Maximum \\
    \midrule
    50    & 9.7 & 0.110 & 0.454 & 0.8 & 0.165 & 0.346 \\
    100   & 30.9 & 0.110 & 0.454 & 2.3 & 0.165 & 0.346 \\
    200   & 39.4 & 0.110 & 0.454 & 4.3 & 0.165 & 0.380 \\
    500   & 78.8 & 0.110 & 0.467 & 12.3 & 0.136 & 0.381 \\
    \bottomrule
  \end{tabular}
\end{table}

\subsection{Transfer Learning for Multi-Armed Bandits.}

\begin{table}[htbp]
  \centering
  \caption{Detailed numerical results for \cref{alg: MAB with noisy causal bounds} (Arm 3)}
  \label{tab: Detailed numerical results for Arm 3}
  \small
  \begin{tabular}{cccccc}
    \toprule
    Estimation Error & \multicolumn{2}{c}{Final Regret} & \multicolumn{2}{c}{Selection Count} \\
    \cmidrule(lr){2 - 3} \cmidrule(lr){4 - 5}
    ${\epsilon}_3(\delta)$ & Mean & SD & Mean & SD \\
    \midrule
    0.015 & 57.248 & 3.481 & 0.0 & 0.0 \\
    0.018 & 59.708 & 4.713 & 20.7 & 24.611 \\
    0.020 & 70.132 & 4.075 & 130.16 & 21.692 \\
    0.022 & 78.686 & 5.294 & 213.98 & 25.781 \\
    0.025 & 87.146 & 4.653 & 299.58 & 27.370 \\
    0.030 & 92.872 & 3.428 & 358.62 & 25.509 \\
    \bottomrule
  \end{tabular}
\end{table}

\begin{table}[htbp]
  \centering
  \caption{Detailed numerical results for \cref{alg: MAB with noisy causal bounds} (Arm 4)}
  \label{tab: Detailed numerical results for Arm 4}
  \small
  \begin{tabular}{cccccc}
    \toprule
    Estimation Error & \multicolumn{2}{c}{Final Regret} & \multicolumn{2}{c}{Selection Count} \\
    \cmidrule(lr){2 - 3} \cmidrule(lr){4 - 5}
    ${\epsilon}_4(\delta)$ & Mean & SD & Mean & SD \\
    \midrule
    0.015 & 93.846 & 4.356 & 374.66 & 30.335 \\
    0.018 & 93.984 & 4.780 & 368.70 & 29.131 \\
    0.02 & 93.346 & 4.335 & 366.72 & 32.515 \\
    0.022 & 92.244 & 4.385 & 368.38 & 33.180 \\
    0.025 & 94.218 & 4.618 & 371.60 & 26.967 \\
    0.03 & 94.823 & 4.726 & 373.46 & 31.629 \\
    \bottomrule
  \end{tabular}
\end{table}

\subsection{Negative Transfer in Multi-Armed Bandits.}

\definecolor{barcolor}{RGB}{52,152,219}
\definecolor{linecolor}{RGB}{231,76,60}
\definecolor{errorcolor}{RGB}{44,62,80}
\begin{figure}[htbp]
  \centering
  \begin{subfigure}{0.46\textwidth}
    \centering
    \begin{tikzpicture}
      \begin{axis}[
          width=0.9\textwidth,
          height=0.62\textwidth,
          axis y line*=left,
          xlabel={Number of Offline Samples for Arm 4},
          ylabel={Selection Count},
          ymin=0, ymax=2200,
          ytick={0,500,1000,1500,2000},
          yticklabel style={font=\footnotesize},
          ylabel style={font=\small},
          xmin=-200, xmax=3500,
          xtick={100,1000,1500,2000,2500,3000},
          xticklabel style={font=\footnotesize},
          xlabel style={font=\small},
          legend style={at={(0.8,1.3)}, font=\small},
          grid=major,
          grid style={dashed, gray!30},
          bar width=15pt,
          error bars/y dir=both,
          error bars/y explicit,
          error bars/error bar style={color=errorcolor, thick}
        ]
        \addplot+[
          ybar,
          fill=barcolor!40,
          draw=barcolor,
        ]
        coordinates {
          (100, 483.76) +- (0,29.021)
          (1000, 1063.24) +- (0,39.581)
          (1500, 1296.6) +- (0,35.955)
          (2000, 1501.86) +- (0,43.817)
          (2500, 1668.98) +- (0,55.242)
          (3000, 1865.28) +- (0,62.670)
        };
        \addlegendentry{Selection Count}

        \node at (axis cs:100, 483.76) [below, font=\tiny, xshift=-2pt] {483.8};
        \node at (axis cs:1000, 1063.24) [below, font=\tiny, xshift=-2pt] {1063.2};
        \node at (axis cs:1500, 1296.6) [below, font=\tiny, xshift=-2pt] {1296.6};
        \node at (axis cs:2000, 1501.86) [below, font=\tiny, xshift=-2pt] {1501.9};
        \node at (axis cs:2500, 1668.98) [below, font=\tiny, xshift=-2pt] {1669.0};
        \node at (axis cs:3000, 1865.28) [below, font=\tiny, xshift=-2pt] {1865.3};
      \end{axis}

      \begin{axis}[
          width=0.9\textwidth,
          height=0.62\textwidth,
          axis y line*=right,
          axis x line=none,
          ymin=0, ymax=300,
          xmin=-200, xmax=3500,
          yticklabel style={font=\footnotesize},
          ylabel style={font=\small},
          legend style={at={(0.98,0.88)}, anchor=north east, font=\small},
        ]
        \addplot+[
          color=linecolor,
          mark=*,
          mark options={fill=white, scale=1.2},
          line width=1.2pt,
          error bars/.cd,
          y dir=both,
          y explicit,
          error bar style={color=errorcolor, thick}
        ]
        coordinates {
          (100, 119.824) +- (0,5.368)
          (1000, 177.006) +- (0,4.835)
          (1500, 200.104) +- (0,5.126)
          (2000, 220.530) +- (0,5.968)
          (2500, 235.528) +- (0,7.904)
          (3000, 256.762) +- (0,7.941)
        };

        \node at (axis cs:100,119.824) [above, font=\tiny,  yshift=2pt] {119.8};
        \node at (axis cs:1000,177.006) [above, font=\tiny, yshift=2pt] {177.0};
        \node at (axis cs:1500,200.104) [above, font=\tiny, yshift=2pt] {200.1};
        \node at (axis cs:2000,220.530) [above, font=\tiny, yshift=2pt] {220.5};
        \node at (axis cs:2500,235.528) [above, font=\tiny, yshift=2pt] {235.5};
        \node at (axis cs:3000,256.762) [above, font=\tiny, yshift=2pt] {256.8};
      \end{axis}
    \end{tikzpicture}
    \caption{Impact on Arm 4}
    \label{fig:negative_impact_arm4}
  \end{subfigure}
  \hfill
  \begin{subfigure}{0.46\textwidth}
    \centering
    \begin{tikzpicture}
      \begin{axis}[
          width=0.9\textwidth,
          height=0.62\textwidth,
          axis y line*=left,
          xlabel={Number of Offline Samples for Arm 5},
          ymin=8500, ymax=9400,
          ytick={8500,8700,8900,9100,9300},
          yticklabel style={font=\footnotesize},
          ylabel style={font=\small},
          xmin=-200, xmax=3500,
          xtick={100,1000,1500,2000,2500,3000},
          xticklabel style={font=\footnotesize},
          xlabel style={font=\small},
          legend style={at={(0.98,0.98)}, anchor=north east, font=\small},
          grid=major,
          grid style={dashed, gray!30},
          bar width=15pt,
          error bars/y dir=both,
          error bars/y explicit,
          error bars/error bar style={color=errorcolor, thick}
        ]
        \addplot+[
          ybar,
          fill=barcolor!40,
          draw=barcolor,
        ]
        coordinates {
          (100, 9107.62) +- (0,41.382)
          (1000, 9018.4) +- (0,47.626)
          (1500, 8988.3) +- (0,38.118)
          (2000, 8866.48) +- (0,64.606)
          (2500, 8716.82) +- (0,83.119)
          (3000, 8622.4) +- (0,72.802)
        };

        \node at (axis cs:100, 9107.62) [above, font=\tiny, yshift=3pt] {9107.6};
        \node at (axis cs:1000, 9018.4) [below, font=\tiny, yshift=-3pt] {9018.4};
        \node at (axis cs:1500, 8988.3) [below, font=\tiny, yshift=-2pt] {8988.3};
        \node at (axis cs:2000, 8866.48) [above, font=\tiny, yshift=3pt] {8866.5};
        \node at (axis cs:2500, 8716.82) [above, font=\tiny, yshift=4pt] {8716.8};
        \node at (axis cs:3000, 8622.4) [above, font=\tiny, yshift=4pt] {8622.4};
      \end{axis}

      \begin{axis}[
          width=0.9\textwidth,
          height=0.62\textwidth,
          axis y line*=right,
          axis x line=none,
          yticklabel style={font=\footnotesize},
          ylabel={Cumulative Regret},
          xmin=-200, xmax=3500,
          ymin=0, ymax=200,
          ylabel style={font=\small},
          legend style={at={(0.8,1.3)}, font=\small},
        ]
        \addplot+[
          color=linecolor,
          mark=*,
          mark options={fill=white, scale=1.2},
          line width=1.2pt,
          error bars/.cd,
          y dir=both,
          y explicit,
          error bar style={color=errorcolor, thick}
        ]
        coordinates {
          (100, 113.658) +- (0,4.254)
          (1000, 123.662) +- (0,5.105)
          (1500, 127.368) +- (0,4.126)
          (2000, 139.906) +- (0,6.755)
          (2500, 155.344) +- (0,8.504)
          (3000, 165.412) +- (0,7.413)
        };
        \addlegendentry{Cumulative Regret}

        \node at (axis cs:100,113.658) [below, font=\tiny, yshift=-2pt] {113.7};
        \node at (axis cs:1000,123.662) [above, font=\tiny, yshift=2pt] {123.7};
        \node at (axis cs:1500,127.368) [above, font=\tiny, yshift=2pt] {127.4};
        \node at (axis cs:2000,139.906) [above, font=\tiny, yshift=4pt] {139.9};
        \node at (axis cs:2500,155.344) [above, font=\tiny, yshift=4pt] {155.3};
        \node at (axis cs:3000,165.412) [above, font=\tiny, yshift=4pt] {165.4};
      \end{axis}
    \end{tikzpicture}
    \caption{Impact on Arm 5}
    \label{fig:negative_impact_arm5}
  \end{subfigure}

  \caption{Impact of offline dataset size on transfer learning algorithm performance.
    For both subgraphs, bars represent the average selection count of the test arm (left axis), while the line shows
  the mean final regret (right axis). Error bars indicate $\pm1$ standard deviation.}
  \label{fig:negative_transfer}
\end{figure}
To simulate naive knowledge transfer, the UCB-variant algorithm is warm-started using potentially incorrect prior reward estimates derived from a source environment.
Specifically, we initialize the prior estimates for the six arms as (0.5, 0.6, 0.7, 0.78, 0.85, 0.75).
Crucially, these priors introduce bias, most notably causing misidentification of the optimal arm during online learning.
This is demonstrated by assuming the true optimal arm in the target environment has a below-average expected reward of 0.75 (Arm 5),
while the priors incorrectly suggest Arm 4 (0.85) is optimal.
The true mean rewards during online learning correspond to the target environment configuration in \cref{tab:arm configuration in MAB}.
This discrepancy between the warm-start priors and the target environment's reality models the negative transfer effect inherent in naive knowledge transfer.

To investigate offline data volume impacts,
we vary Arm 4's/Arm 5's sample size from 100 to 3000 while maintaining other arms at 30 samples.
After executing $T = 10^4$ rounds over 50 trials (results in \cref{fig:negative_impact_arm4}),
we observe that increasing offline samples for either Arm 4 or Arm 5 (see \cref{fig:negative_impact_arm5}) degrades performance below standard UCB (\cref{fig:negative_transfer}).
Crucially, larger offline samples increase final regret—conclusive evidence of negative transfer.
This manifests in shifting arm selection:
The suboptimal Arm 4's selection count increases with offline data volume,
while optimal Arm 5's decreases, demonstrating how biased priors mislead exploration.

\subsection{Limiting Behavior}

We set a fixed estimation error $\epsilon_a(\delta)$ for each causal bound.
In this experiment, this fixed value was configured from $0.1$ to $0.01$.


\begin{figure}[htbp]
  \centering
  \begin{tikzpicture}
    \begin{axis}[
        width = 12cm,
        height = 8cm,
        xlabel = {Estimation Error ($\epsilon_a(\delta)$)},
        ylabel = {Mean Final Regret},
        grid = major,
        grid style = {dashed, gray!30},
        legend style={
          at={(0.38,0.8)},
          anchor=north east,
          font=\small,
          legend columns=1,
          cells={anchor=west},
        },
        xmode = log,
        log basis x = 10,
        xmin = 0.0005, xmax = 1.36,
        ymin = 0, ymax = 120,
        xtick = {0.001,0.01,0.1,1},
        extra x ticks = {0.02,0.04,0.06},
        extra x tick labels = {0.02,0.04},
        error bars/y dir = both,
        error bars/y explicit,
        error bars/error bar style = {blue, thick},
        title = {Limiting Behavior of \cref{alg: MAB with noisy causal bounds}},
        title style = {font=\bfseries}
      ]
      \addplot[name path=top1, draw=none, forget plot] coordinates {(0.0005,118.9) (1.36,118.9)};
      \addplot[name path=bottom1, draw=none, forget plot] coordinates {(0.0005,109.9) (1.36,109.9)};
      \addplot[green!20, fill opacity=0.4] fill between[of=top1 and bottom1];
      \addlegendentry{Plain UCB region}

      \addplot[green, thick, no marks] coordinates {(0.0005,114.4) (1.36,114.4)};
      \addlegendentry{Plain UCB}

      \node[green, anchor=west] at (axis cs: 0.002, 110) {\small $114.4 \pm 4.5$};
      \addplot[name path=top2, draw=none, forget plot] coordinates {(0.0005,32.9) (1.36,32.9)};
      \addplot[name path=bottom2, draw=none, forget plot] coordinates {(0.0005,27.3) (1.36,27.3)};
      \addplot[red!20, fill opacity=0.4] fill between[of=top2 and bottom2];
      \addlegendentry{\cref{alg: TL-MAB} region}

      \addplot[red, thick, no marks] coordinates {(0.0005,30.1) (1.36,30.1)};
      \addlegendentry{\cref{alg: TL-MAB}}

      \node[red, anchor=west] at (axis cs: 0.002, 34) {\small $30.1 \pm 2.8$};
      \addplot+ [
        blue,
        mark = square*,
        mark options = {solid, fill = blue!20},
        thick,
      ]
      coordinates {
        (0.001, 36.322) +- (0, 2.566)
        (0.01,  37.034) +- (0, 3.232)
        (0.02,  43.596) +- (0, 4.466)
        (0.04,  74.096) +- (0, 4.216)
        (0.06,  80.908) +- (0, 4.658)
        (0.08,  88.720) +- (0, 4.667)
        (0.1,   95.062) +- (0, 4.465)
        (1.0,   113.71) +- (0, 4.399)
      };
      \addlegendentry{\cref{alg: MAB with noisy causal bounds}}

      \node[font=\small] at (axis cs: 0.001,36.322) [above]  {36.3};
      \node[font=\small] at (axis cs: 0.01, 37.034) [above]  {37.0};
      \node[font=\small] at (axis cs: 0.02, 43.596) [above]  {43.6};
      \node[font=\small] at (axis cs: 0.04, 74.096) [above]  {74.1};
      \node[font=\small] at (axis cs: 0.06, 80.908) [above]  {80.9};
      \node[font=\small] at (axis cs: 0.08, 88.720) [above]  {88.7};
      \node[font=\small] at (axis cs: 0.1,  95.062) [above]  {95.1};
      \node[font=\small] at (axis cs: 1,  113.71) [below]  {113.7};
    \end{axis}
  \end{tikzpicture}
\end{figure}

\begin{itemize}
  \item \textbf{Superior Performance with Precise Causal Information:} Algorithm \ref{alg: MAB with noisy causal bounds} demonstrates remarkable efficiency when provided with accurate causal bounds ($\epsilon_a(\delta) \leq 0.01$).
    In this regime, it achieves near-optimal regret ($\sim$37), coming remarkably close to the performance ceiling set by Algorithm \ref{alg: TL-MAB} ($30.1 \pm 2.8$) which requires \textit{perfect} causal knowledge.
    This represents a $\sim67\%$ reduction in regret compared to the plain UCB baseline ($114.4 \pm 4.5$),
    showcasing its ability to effectively leverage reliable causal structures for substantial performance gains.
    The algorithm maintains this significant advantage across the low-error regime ($\epsilon_a(\delta) \leq 0.1$), with performance optimization directly linked to causal estimation quality.

  \item \textbf{Robustness and Asymptotic Guarantees:} Crucially, even when causal bounds become unreliable ($\epsilon_a(\delta) = 1.0$), the regret ($113.7 \pm 4.399$) remains statistically indistinguishable from the non-causal baseline ($114.4 \pm 4.5$). This demonstrates that \textit{our algorithm provably avoids negative transfer}, ensuring performance never deteriorates below the plain UCB benchmark. The results confirm the theoretical limiting behavior: as $\epsilon_a(\delta) \to 0^+$, regret approaches the optimal TL-MAB level, while as $\epsilon_a(\delta) \to \infty$, it converges to the plain UCB baseline. This graceful degradation guarantees robust performance, leveraging causal information when accurate while maintaining baseline-level efficiency when estimates are uninformative.
\end{itemize}

\subsection{Transfer Learning in Contextual Bandits.}

\subsubsection*{Numerical Setup in \cref{sec: numerical}}
In this numerical setup, we define the feature vectors $\phi(a, w)$ in \cref{tb:feature_vectors},
the lower and upper bounds $l(a, w)$ and $h(a, w)$ in \cref{tb:causal bounds in numerical experiment},
as well as two candidate sets for each context $w$ in \cref{tb: Comparison of Candidate Sets}.

\begin{table}[htbp]
  \centering
  \caption{Feature Vector Definitions $\phi(a,w)$}
  \label{tb:feature_vectors}
  \small
  \begin{tabular}{lccccc}
    \toprule
    Context & $a_1$ & $a_2$ & $a_3$ & $a_4$ & $a_5$ \\
    \midrule
    $w_1$ & $[1.0, 0.0]$ & $[0.0, 1.0]$ & $[1.0, 1.0]$ & $[0.5, 0.5]$ & $[2.0, 0.0]$ \\
    $w_2$ & $[1.0, 0.0]$ & $[0.0, 1.0]$ & $[0.5, 0.5]$ & $[1.0, 1.0]$ & $[0.0, 1.0]$ \\
    $w_3$ & $[0.8, 0.0]$ & $[0.0, 0.8]$ & $[0.0, 0.0]$ & $[0.0, 0.0]$ & $[0.0, 0.0]$ \\
    $w_4$ & $[1.2, 0.0]$ & $[0.0, 1.2]$ & $[0.0, 0.0]$ & $[0.0, 0.0]$ & $[0.0, 0.0]$ \\
    $w_5$ & $[1.0, 0.0]$ & $[0.0, 1.0]$ & $[1.0, 1.0]$ & $[0.5, 0.5]$ & $[0.0, 0.0]$ \\
    $w_6$ & $[1.0, 0.0]$ & $[0.0, 1.0]$ & $[0.5, 0.5]$ & $[1.0, 1.0]$ & $[1.0, 0.5]$ \\
    $w_7$ & $[1.0, 0.0]$ & $[0.0, 1.0]$ & $[0.5, 0.5]$ & $[0.5, 0.5]$ & $[2.0, 0.0]$ \\
    $w_8$ & $[1.5, 0.0]$ & $[0.0, 1.0]$ & $[0.7, 0.7]$ & $[0.1, 0.1]$ & $[0.1, 0.1]$ \\
    $w_9$ & $[1.0, 0.0]$ & $[0.0, 1.0]$ & $[0.1, 0.1]$ & $[0.0, 1.0]$ & $[0.1, 0.1]$ \\
    $w_{10}$ & $[0.5, 0.0]$ & $[0.0, 0.5]$ & $[0.5, 0.5]$ & $[0.5, 0.5]$ & $[0.5, 0.5]$ \\
    $w_{11}$ & $[1.0, 0.0]$ & $[0.0, 1.0]$ & $[0.0, 0.0]$ & $[1.0, 1.0]$ & $[0.1, 2.0]$ \\
    \bottomrule
  \end{tabular}%
\end{table}

\begin{table}[htbp]
  \centering
  \caption{Bound Definitions $[l(a,w), h(a,w)]$ for Expected Rewards}
  \label{tb:causal bounds in numerical experiment}
  \small
  \begin{tabular}{lccccc}
    \toprule
    Context & $a_1$ & $a_2$ & $a_3$ & $a_4$ & $a_5$ \\
    \midrule
    $w_1$ & $[0.5, 0.95]$ & $[0.5, 0.95]$ & $[0.95, 1.9]$ & $[0.0, 0.85]$ & $[1.7, 1.9]$ \\
    $w_2$ & $[0.5, 0.95]$ & $[0.5, 0.95]$ & $[0.0, 0.85]$ & $[0.95, 1.9]$ & $[0.0, 0.94]$ \\
    $w_3$ & $[0.6, 1.05]$ & $[0.0, 0.85]$ & $[0.0, 0.5]$ & $[0.0, 0.5]$ & $[0.0, 0.7]$ \\
    $w_4$ & $[0.0, 1.1]$ & $[0.8, 1.05]$ & $[0.0, 0.01]$ & $[0.0, 0.01]$ & $[0.0, 0.01]$ \\
    $w_5$ & $[0.0, 0.9]$ & $[0.0, 0.9]$ & $[0.95, 1.9]$ & $[0.0, 0.9]$ & $[0.0, 0.01]$ \\
    $w_6$ & $[0.5, 0.95]$ & $[0.5, 0.95]$ & $[0.0, 0.85]$ & $[0.95, 1.9]$ & $[0.95, 1.9]$ \\
    $w_7$ & $[0.5, 0.95]$ & $[0.5, 0.95]$ & $[0.0, 0.85]$ & $[0.0, 0.85]$ & $[1.4, 1.9]$ \\
    $w_8$ & $[0.8, 1.35]$ & $[0.8, 0.95]$ & $[0.7, 1.9]$ & $[0.0, 0.2]$ & $[0.0, 0.2]$ \\
    $w_9$ & $[0.5, 0.95]$ & $[0.8, 0.95]$ & $[0.0, 0.4]$ & $[0.8, 0.95]$ & $[0.0, 0.4]$ \\
    $w_{10}$ & $[0.0, 1.0]$ & $[0.0, 1.0]$ & $[0.0, 1.0]$ & $[0.0, 1.0]$ & $[0.0, 1.0]$ \\
    $w_{11}$ & $[0.5, 0.95]$ & $[0.5, 0.95]$ & $[0.0, 1.0]$ & $[0.0, 1.9]$ & $[0.0, 1.9]$ \\
    \bottomrule
  \end{tabular}%
\end{table}

\begin{table}[htbp]
  \centering
  \caption{Comparison of Two Candidate Sets}
  \label{tb: Comparison of Candidate Sets}
  \begin{tabular}{@{}lcc@{}}
    \toprule
    \textbf{Context} &
    \multicolumn{1}{c}{ $\{ a\in\mathcal{A} \mid  h(a,w)\geq \max_i l(i,w)\} $} &
    \multicolumn{1}{c}{ $\mathcal{A}^*(w)$} \\
    \midrule
    $w_1$ & $\{a_3, a_5\}$ & $\{a_3, a_5\}$ \\
    $w_2$ & $\{a_1, a_2, a_4\}$ & $\{a_4\}$ \\
    $w_3$ & $\{a_1, a_2, a_5\}$ & $\{a_1, a_2\}$ \\
    $w_4$ & $\{a_1, a_2\}$ & $\{a_1, a_2\}$ \\
    $w_5$ & $\{a_3\}$ & $\{a_3\}$ \\
    $w_6$ & $\{a_1, a_2, a_4, a_5\}$ & $\{a_4\}$ \\
    $w_7$ & $\{a_5\}$ & $\{a_5\}$ \\
    $w_8$ & $\{a_1, a_2, a_3\}$ & $\{a_1, a_3\}$ \\
    $w_9$ & $\{a_1, a_2, a_4\}$ & $\{a_1, a_2, a_4\}$ \\
    $w_{10}$ & $\{a_1, a_2, a_3, a_4, a_5\}$ & $\{a_1, a_2, a_3, a_4, a_5\}$ \\
    $w_{11}$ & $\{a_1, a_2, a_3, a_4, a_5\}$ & $\{a_4, a_5\}$ \\
    \bottomrule
  \end{tabular}
\end{table}

\subsubsection{Finite Function Space}
We generate function space $\mathcal{F} = {(w - w_0)^\top (a - a_0)}$ of size 50 by sampling parameters $w_0$ and $a_0$ in $\mathbb{R}^d$ from $\mathcal{N}(0, 0.1)$, where $d = 10$. A true reward function $f^*$ is randomly selected from the first 5 functions in $\mathcal{F}$. The reward is then generated as
$$
Y= f^*(W,A) + \mathcal{N}(0,0.1),
$$
where the context $W$ is drawn i.i.d. from a standard normal distribution, and $A$ is the selected action. The action set $\mathcal{A}$ is initialized uniformly at random from $[-1, 1]^d$ with a size of 10. Each experiment is repeated 50 times to smooth the regret curves.

We compare the performance of our algorithm with FALCON \citep{fasterCB}, a well-known implementation of IGW. The numerical results in \cref{fig: comparison with classical and causally enhanced algorithms in function approximation settings} demonstrate that our algorithm significantly outperforms FALCON, even without explicitly removing infeasible functions.
In the experiments, the average size of the action subset $\mathcal{A}(w)$ is $3.254$, highlighting the substantial performance gains achieved by reducing the size of the action space.
Additionally, our algorithm excels with homogeneous functions, which often attain their maximum values at the same points.
In such scenarios, adaptively eliminating suboptimal actions proves to be a highly effective strategy for minimizing regrets.

\begin{figure*}[hbtp]
  \centering
  \resizebox{.45\textwidth}{!}{
    \begin{tikzpicture}
      \begin{axis}[
          height = 0.3\textwidth,
          width = 0.5\textwidth,
          xlabel = time $t$,
          ylabel = regret,
          ymin=0,
          ymax=16000,
          xmin=0,
          xmax=130000,
          xtick pos = left,
          ytick pos = left,
          legend style={at={(0.64,0.81)}, anchor=west, nodes={scale=0.8, transform shape}}
        ]

        \addplot [color=red, line width=1.3pt ] table [x index=0,y index=1, col sep = comma] {data/IGW_mean_fixed_repeat_50_K_10.csv};
        \addplot [color=green, line width=1.3pt  ] table [x index=0,y index=1, col sep = comma] {data/IGW_mean_adaptive_repeat_50_K_10.csv};
        \addplot [color=red, dashed, line width=1.3pt ] table [x index=0,y index=1, col sep = comma] {data/IGW+std_fixed_repeat_50_K_10.csv};
        \addplot [color=red, dashed, line width=1.3pt ] table [x index=0,y index=1, col sep = comma] {data/IGW-std_fixed_repeat_50_K_10.csv};
        \addplot [name path=ada_top, color=green, dashed , line width=1.3pt] table [x index=0,y index=1, col sep = comma] {data/IGW+std_adaptive_repeat_50_K_10.csv};
        \addplot [name path=ada_down, color=green, dashed , line width=1.3pt] table [x index=0,y index=1, col sep = comma] {data/IGW-std_adaptive_repeat_50_K_10.csv};

        \legend{
          FALCON,
          ours
        }
      \end{axis}
    \end{tikzpicture}
  }
  \caption{Comparison of classical and causally enhanced algorithms in function approximation settings. The solid curves represent the average cumulative regret over time for each algorithm. The top and bottom dashed curves correspond to one standard deviation added to and subtracted from the mean cumulative regret.}
  \label{fig: comparison with classical and causally enhanced algorithms in function approximation settings}
\end{figure*}

\section{Related Materials}
\label{sec in appendix: related material}

\subsection{Causal Inference}
\begin{definition}(Back-Door Criterion)
  Given an ordered pair of variables $(X, Y )$ in a directed acyclic graph $\mathcal{G}$, a set of
  variables $\bm{Z}$ satisfies the back-door criterion relative to $(X, Y)$, if no node in $\bm{Z}$ is a descendant of $X$, and $\bm{Z}$ blocks
  every path between $X$ and $Y$ that contains an arrow into $X$.
\end{definition}

\begin{definition}{d-separation}
  In a causal diagram $\mathcal{G}$, a path $\mathcal{P}$ is blocked by a set of nodes $\bm{Z}$ if and only if
  \begin{enumerate}
    \item $\mathcal{P}$ contains a chain of nodes $A\leftarrow B \leftarrow C$ or a fork $A\rightarrow B \leftarrow C$ such that the middle node $B$ is in $\bm{Z}$ (i.e., $B$
      is conditioned on), or
    \item $\mathcal{P}$ contains a collider $A\leftarrow B \rightarrow C$ such that the collision node $B$ is not in $\bm{Z}$, and no descendant of $B$ is in $\bm{Z}$.
  \end{enumerate}
  If $\bm{Z}$ blocks every path between two nodes $X$ and $Y$ , then $X$ and $Y$ are d-separated conditional on $\bm{Z}$, and thus are
  independent conditional on $\bm{Z}$.
\end{definition}

If $X$ is a variable in a causal model, its corresponding intervention variable $I_X$ is an exogenous variable with one arrow pointing into $X$.
The range of $I_X$ is the same as the range of $X$, with one additional value we can call ``off''.
When $I_X$ is off, the value of $X$ is determined by its other parents in the causal model.
When $I_X$ takes any other value, $X$ takes the same value as $I_X$, regardless of the value of $X$'s other parents.
If $X$ is a set of variables, then $I_X$ will be the set of corresponding intervention variables.
We introduce the following do-calculus rules proposed in \cite{causaloverview}.

\textbf{Rule 1 (Insertion/deletion of observations)}
$$
\Prob{\bm{Y}\mid \mathrm{do}(\bm{x}),\bm{Z},\bm{W}  } = \Prob{\bm{Y}\mid \mathrm{do}(\bm{x}),\bm{W}}
$$
if $\bm{Y}$ and $I_{\bm{Z}}$ are d-separated by $\bm{x} \cup \bm{W}$ in $\mathcal{G}^*$,
the graph obtained from $\mathcal{G}$ by removing all arrows pointing into variables in $\bm{x}$.

\textbf{Rule 2 (Action/observation exchange)}
$$
\Prob{\bm{Y}\mid \mathrm{do}(\bm{x}), \mathrm{do}(\bm{Z}),\bm{W}  } = \Prob{\bm{Y}\mid \mathrm{do}(\bm{x}),\bm{Z},\bm{W}}
$$
if $\bm{Y}$ and $I_{\bm{Z}}$ are d-separated by $\bm{x} \cup \bm{Z}\cup \bm{W}$ in $\mathcal{G}^{\dagger}$,
the graph obtained from $\mathcal{G}$ by removing all arrows pointing into variables in $\bm{x}$ and all arrows pointing out of variables in $\bm{z}$.

\textbf{Rule 3 (Insertion/deletion of actions)}
$$
\Prob{\bm{Y}\mid \mathrm{do}(\bm{x}), \mathrm{do}(\bm{Z}),\bm{W}  } = \Prob{\bm{Y}\mid \mathrm{do}(\bm{x}),\bm{W}}
$$
if $\bm{Y}$ and $I_{\bm{Z}}$ are d-separated by $\bm{x} \cup \bm{W}$ in $\mathcal{G}^*$,
the graph obtained from $\mathcal{G}$ by removing all arrows pointing into variables in $\bm{x}$.

\subsection{Hausdorff Convergence}

Since we mainly consider the normed linear space, we focus on the norm $\norm{\cdot}$ instead of general distance measures.
The Hausdorff distance $d_{H}$ for two sets $\Omega_1$ and $\Omega_2$ is defined as
$$
d_H(\Omega_1,\Omega_2)
=  \max\Bigl\{
  \sup_{a\in \Omega_1}\inf_{b\in \Omega_2}\|a-b\|,
  \sup_{b\in \Omega_2}\inf_{a\in \Omega_1}\|a-b\|
\Bigr\}.
$$

\begin{theorem}(\citealt{VAbook})
  \label{thm:hausdorff-convergence}
  Let $M$ be a compact metric space, and let $\{\Omega_k\}_{k=1}^\infty$ and $\Omega_\infty$ be nonempty compact subsets of $M$.  Then
  $$
  \lim_{k\to\infty}d_H(\Omega_k,\Omega_\infty) = 0
  \iff
  \limsup_{k\to\infty}\Omega_k
  =\liminf_{k\to\infty}\Omega_k
  =\Omega_\infty,
  $$
  where
  $$
  \limsup_{k\to\infty}\Omega_k
  = \bigl\{x\in M:\exists k_j\to\infty, x_{k_j}\in\Omega_{k_j}, x_{k_j}\to x\bigr\},
  $$
  and
  $$
  \liminf_{k\to\infty}\Omega_k
  = \bigl\{x\in M:\forall k\ge N, \exists x_k\in\Omega_k, x_k\to x\text{ as }k\to\infty\bigr\}.
  $$
\end{theorem}

\begin{corollary}
  Let $\{\Omega_k\}$ be a sequence of nonempty compact sets in a compact space $M$, and let $\Omega_\infty$ be a nonempty compact subset of $M$.  Then $\Omega_k\to\Omega_\infty$ in Hausdorff distance if and only if:
  \begin{enumerate}
    \item For every $x\in\Omega_\infty$, there exists a sequence $x_k\in\Omega_k$ such that $x_k\to x$.
    \item Whenever $x_k\in\Omega_k$ is any sequence with $x_k\to x$, we have $x\in\Omega_\infty$.
  \end{enumerate}
\end{corollary}

\end{document}